\documentclass[journal]{IEEEtran}
\usepackage{stmaryrd}
\usepackage{bm}
\usepackage{amssymb}
\usepackage{booktabs}
\usepackage{graphicx}
\usepackage{float}
\usepackage{graphicx}
\usepackage{stfloats}
\usepackage{amsmath}
\usepackage{epstopdf}
\usepackage{algorithm,algorithmic}
\usepackage{multirow}
\usepackage{bbm}
\usepackage{framed}
\usepackage[colorlinks,linkcolor=red,anchorcolor=green,citecolor=green]{hyperref}
\usepackage[numbers,sort&compress]{natbib}
\usepackage{amsthm}

\newcommand{\av}{\boldsymbol{a}}

\newcommand{\bv}{\boldsymbol{b}}

\newcommand{\xv}{\boldsymbol{x}}

\newcommand{\gammav}{{\boldsymbol{\gamma}}}

\newcommand{\sigmav}[0]{{\boldsymbol{\sigma}} }

\newcommand{\ie}{{\em i.e.}}
\newcommand{\eg}{{\em e.g.}}

\newcommand{\bihan}[1]{{\color{black}#1}}            % Bihan

\newtheorem{lemma}{Lemma}

\newtheorem{theorem}{Theorem}%[section]

\ifCLASSINFOpdf

\else

\fi

\begin{document}

\title{From Rank Estimation to Rank Approximation: Rank Residual Constraint for Image Restoration}

\author{Zhiyuan~Zha,~\IEEEmembership{Member,~IEEE}, Xin~Yuan,~\IEEEmembership{Senior Member,~IEEE,} Bihan~Wen,~\IEEEmembership{Member,~IEEE}, Jiantao~Zhou,~\IEEEmembership{Member,~IEEE}, Jiachao~Zhang,~\IEEEmembership{Member,~IEEE} and Ce~Zhu,~\IEEEmembership{Fellow,~IEEE}

\IEEEcompsocitemizethanks{
\IEEEcompsocthanksitem This work was supported by the NSFC (61571102), the applied research programs of science and technology., Sichuan Province (No. 2018JY0035), the Ministry of Education, Republic of Singapore, under the Start-up Grant and the Macau Science and Technology Development Fund, Macau SAR (File no. SKL-IOTSC-2018-2020, 077/2018/A2, 022/2017/A1).
\IEEEcompsocthanksitem Z. Zha and C. Zhu are with the School of Information and Communication Engineering,
University of Electronic Science and Technology of China, Chengdu, 611731, China.  E-mail: zhazhiyuan.mmd@gmail.com, eczhu@uestc.edu.cn. \emph{Corresponding Author: Zhiyuan Zha.}
\IEEEcompsocthanksitem X. Yuan is with Nokia Bell Labs, 600 Mountain Avenue, Murray Hill, NJ, 07974, USA. E-mail: xyuan@bell-labs.com.
\IEEEcompsocthanksitem B. Wen is with School of Electrical \& Electronic Engineering, Nanyang Technological University, Singapore 639798. E-mail: bihan.wen@ntu.edu.sg.
\IEEEcompsocthanksitem J. Zhou is with State Key Laboratory of Internet of Things for Smart City and Department of Computer and Information Science,  University of Macau, Macau 999078, China. E-mail: jtzhou@umac.mo.
\IEEEcompsocthanksitem J. Zhang is with Artificial Intelligence Institute of Industrial Technology,  Nanjing Institute of Technology, Nanjing 211167, China. E-mail: zhangjc07@foxmail.com. %The first two authors contributed to this work equally.
%\IEEEcompsocthanksitem The first two authors contributed to this work equally.
}
}

\markboth{IEEE Transaction on Image Processing,~2020}%
{Shell \MakeLowercase{\textit{et al.}}: Bare Demo of IEEEtran.cls for IEEE Journals}

\maketitle

\begin{abstract}
In this paper, we propose a novel approach to the rank minimization problem, termed rank residual constraint (RRC) model. Different from existing low-rank based approaches, such as the well-known nuclear norm minimization (NNM) and the weighted nuclear norm minimization (WNNM), which \emph{estimate} the underlying low-rank matrix directly from the corrupted observations, we progressively \emph{approximate} the underlying low-rank matrix via \emph{minimizing the rank residual}.  Through integrating the image nonlocal self-similarity (NSS) prior with the proposed RRC model, we apply it to image restoration tasks, including image denoising and image compression artifacts reduction. Towards this end, we first obtain a good reference of the original image groups by using the image NSS prior, and then the rank residual of the image groups between this reference and the degraded image is minimized to achieve a better estimate to the desired image. In this manner, both the reference and the estimated image are updated gradually and jointly in each iteration. Based on the group-based sparse representation model, we further provide an analytical investigation on the feasibility of the proposed RRC model. Experimental results demonstrate that the proposed RRC method outperforms many state-of-the-art schemes in both the objective and perceptual quality.
\end{abstract}


\IEEEpeerreviewmaketitle

\section{Introduction}
\IEEEPARstart{L}{ow}-rank matrix estimation has attracted increasing attention due to its wide applications. In particular, it has been successfully applied in various machine learning and computer vision tasks \cite{1,2,3,4,5,6,7,8,9,10,11,12,13,14,15,16,17,18,19,20,21,22,23,24}. For instance, the Netflix customer data matrix is treated as a low-rank one for the reason that the customers' choices are largely dependent on a few common factors \cite{1,2}. The foreground and background in a video are also modeled as being sparse and low-rank, respectively \cite{3,4,5}. As the matrix formed by nonlocal similar patches in a natural image is of low-rank, various low-rank models for image completion problems have been proposed, such as collaborative filtering \cite{6}, image alignment \cite{4}, image/video denoising \cite{5,7,16}, shadow removal \cite{8} and reconstruction of occluded/corrupted face images \cite{1,9,24}.

One typical low-rank matrix estimation method is the low-rank matrix factorization (LRMF) \cite{9,10,11,12,13,14}, which factorizes an observed matrix $\textbf{\emph{Y}}$ into a product of two matrices that can be used to reconstruct the desired matrix with certain fidelity terms. A series of LRMF methods have been developed, such as the classical singular value decomposition (SVD) under $\ell_2$-norm \cite{10}, robust LRMF methods under $\ell_1$-norm \cite{11,12} and other probabilistic methods \cite{13,14}.

Another parallel research is the rank minimization \cite{15,16,17,18,19,20,21,22,23,24}, with the nuclear norm minimization (NNM) \cite{15} being the most representative one. The nuclear norm of a matrix $\textbf{\emph{X}}$, denoted by $\left\|\textbf{\emph{X}}\right\|_*$, is the summation of its singular values, \ie, $\left\|\textbf{\emph{X}}\right\|_*=\sum\nolimits_i{\boldsymbol\sigma_i}(\textbf{\emph{X}})$, with ${\boldsymbol\sigma_i}(\textbf{\emph{X}})$ representing the $i^{th}$ singular value of $\textbf{\emph{X}}$. NNM aims to recover the underlying low-rank matrix $\textbf{\emph{X}}$ from its degraded observation matrix $\textbf{\emph{Y}}$, while minimizing $\left\|\textbf{\emph{X}}\right\|_*$. However, NNM tends to over-shrink the rank components, and therefore limits its capability and flexibility. %In general, the rank minimization is an NP-hard problem.
To enforce the rank minimization efficiently, inspired by the success of $\ell_p$ ($0<p<1$) sparse optimization, Schatten $p$-norm is proposed \cite{20,21,22}, which is defined as the $\ell_p$-norm ($0<p<1$) of the singular values. Compared with the traditional NNM, Schatten $p$-norm minimization (SNM) achieves a more accurate recovery of the matrix rank with only one requirement of a \emph{weakly restricted isometry property} \cite{22}. Nonetheless, similar to NNM, most SNM-based methods also shrink all singular values equally, which may be infeasible in executing many practical problems \cite{23}. To further improve the flexibility of NNM, most recently, Gu $\emph{et al}.$ \cite{16,5} proposed a weighted nuclear norm minimization (WNNM) model, which heuristically sets the weights being inverse to the singular values. Compared with NNM, WNNM assigns different weights to different singular values such that the matrix rank estimation becomes more rigid. Similar case also exists in the truncated nuclear norm minimization \cite{18} and the partial sum minimization \cite{19}.

\begin{figure*}[!t]
	\vspace{-2mm}
	\centering
	{\includegraphics[width= 1\textwidth]{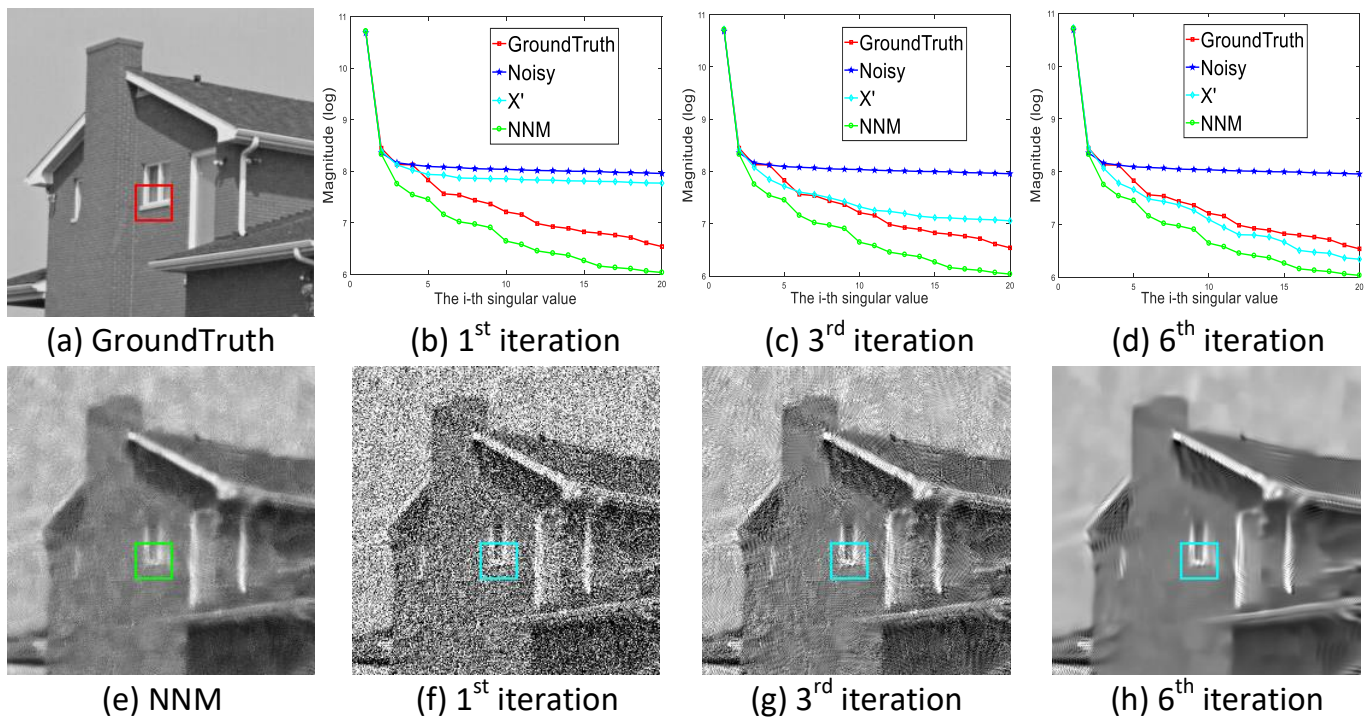}}
\vspace{-6mm}
	\caption{ Illustration of the proposed image denoising method via the rank residual constraint (RRC). The image {\em House} is corrupted by zero-mean Gaussian noise with standard deviation $\sigmav_n$ = 100. (b-d) The singular values of the image group (with reference in the cyan box at the bottom row) from the ground truth image (red), noisy image (blue), image recovered by NNM (green) and the proposed reference matrix $\textbf{\emph{X}}'$ (cyan) at the $1^{st}$, $3^{rd}$ and $6^{th}$ iterations of our algorithm. (f-h) Reconstructed images at the $1^{st}$, $3^{rd}$ and $6^{th}$ iterations using the proposed RRC model. It can be seen that the singular values of the reference matrix $\textbf{\emph{X}}'$ progressively approach the ground truth and the reconstructed image is getting close to the original image.}
	\label{fig:1}
	\vspace{-2mm}
\end{figure*}

However, one common practice in the aforementioned low-rank models is only to estimate the low-rank matrix directly from the corrupted observations, which may lead to a defective result in real applications, such as image inverse problems. In this paper, we propose a novel approach, dubbed rank residual constraint (RRC) model, for the rank minimization problem. Different from existing low-rank based methods, such as the well-known NNM and WNNM, we progressively \emph{approximate}  the underlying low-rank matrix via \emph{minimizing the rank residual}. By integrating the image nonlocal self-similarity (NSS) prior \cite{5,16} with the proposed RRC model, we apply it to image restoration tasks, including image denoising and image compression artifacts reduction. In a nutshell, given a corrupted image ${\textbf {\emph y}}$, in each iteration, we construct a reference low-rank matrix ${\textbf {\emph X}}'$ (for each image  group) by exploiting the image NSS prior, and {\em approximate} our recovered matrix $\hat{\textbf{\emph{X}}}$ to  this reference matrix ${\textbf {\emph X}}'$  via the proposed RRC model. %It is noted that, the reference matrix ${\textbf {\emph X}}'$  and the recovered matrix $\hat{\textbf{\emph{X}}}$  are improved gradually and jointly in each iteration.
Fig.~\ref{fig:1} depicts that the reconstructed image from our proposed algorithm can progressively \emph{approximate} the ground truth, by taking the widely used image {\em House} as an example, which is corrupted by zero-mean Gaussian noise with standard deviation $\sigmav_n$ = 100. It can be observed that the singular values of the reference matrix ${\textbf {\emph X}}'$ approaches the singular values  of the ground truth progressively and so does the recovered image (Fig.~\ref{fig:1} (f-h)).

It is worth noting that the significant difference between the proposed RRC and the existing low-rank based methods (\eg, NNM and WNNM) is that we analyze the rank minimization problem from a different perspective.  To be concrete, traditional low-rank based methods estimated the low-rank matrix directly from the corrupted observations. By contrast, in our RRC model, we analyze the rank minimization problem from the point of approximation theory \cite{71}, namely,  \emph{minimizing the rank residual}; the singular values of the recovered matrix  progressively approaches the singular values  of the reference matrix.   Note that the reference matrix and the recovered matrix in our RRC model are both updated gradually and jointly in each iteration. Moreover, we provide an analytical investigation on the feasibility of the proposed RRC model from the perspective of group-based sparse representation \cite{28,29,72,27}.

The rest of this paper is organized as follows. Section~\ref{sec:2} describes the RRC model based on the rank minimization scenario. Section~\ref{sec:3} presents how to use the proposed RRC model for image denoising. Section~\ref{sec:4} develops the algorithm for image compression artifacts reduction exploiting the proposed RRC model. Section~\ref{sec:5} provides an analysis investigation for the proposed RRC model in terms of group-based sparse representation. Extensive results for image restoration are presented in Section~\ref{sec:6} and Section~\ref{sec:7} concludes the paper.

\section {Rank Minimization via Rank Residual Constraint}
 \label{sec:2}
In this section, we first analyze the weakness of the traditional NNM model and then propose the RRC model to improve the performance of rank estimation.

\subsection {Nuclear Norm Minimization}	
 \label{sec:2.1}
According to \cite{6,15,17}, NNM is the tightest convex relaxation of the original rank minimization problem. Given a data matrix $\textbf{\emph{Y}}\in\mathbb{R}^{d\times m}$, the goal of NNM is to find a low-rank matrix $\textbf{\emph{X}}\in\mathbb{R}^{d\times m}$ of rank $r \ll min ({d, m})$, satisfying the following objective function,
\vspace{-1mm}
\begin{equation}
\hat{\textbf{\emph{X}}}=\arg\min_{\textbf{\emph{X}}}\left(\frac{1}{2}\left\|\textbf{\emph{Y}}-\textbf{\emph{X}}\right\|_F^2 +\lambda \left\|\textbf{\emph{X}}\right\|_*\right),
\label{eq:1}
\end{equation} %equal 1
where $\left\|\cdot\right\|_F$ denotes the Frobenius norm and $\lambda>0$ is a regularization parameter. Cand{\`e}s $\emph{et al}.$ \cite{17} proved that the low-rank matrix can be perfectly recovered from the degraded/corrupted data matrix with a high probability by solving an NNM problem.  Despite the theoretical guarantee of the singular value thresholding (SVT) algorithm \cite{6}, it has been observed that the recovery performance of such a convex relaxation will degrade in the presence of noise, and the solution will seriously deviate from the original solution of rank minimization problem \cite{23}. More specifically, NNM tends to over-shrink the rank of the matrix. Taking the widely used image $\emph{Lena}$ in Fig.~\ref{fig:2}(a) as an example, we add Gaussian noise with standard deviation ${\boldsymbol\sigma}_{n}$ = 100 to the clean image and perform NNM to recover a denoised image in Fig.~\ref{fig:2}(c). We randomly extract a patch from the noisy image in Fig.~\ref{fig:2}(b) and search 60 similar patches to generate a group (please refer to Section~\ref{sec:3} for details of constructing the group). These patches (after vectorization) in this group are then stacked into a data matrix. Since all the patches in this group have similar structures, the constructed data matrix is of low-rank. Based on this, we plot the singular values of the group in the noisy image, NNM recovered image, and the original image in Fig.~\ref{fig:2}(d). As can be seen, the solution of NNM (green line) is severely deviated (over-shrunk) from the ground truth (red line).

\begin{figure}[!t]
%\vspace{-2mm}
\centering
{\includegraphics[width=0.48\textwidth]{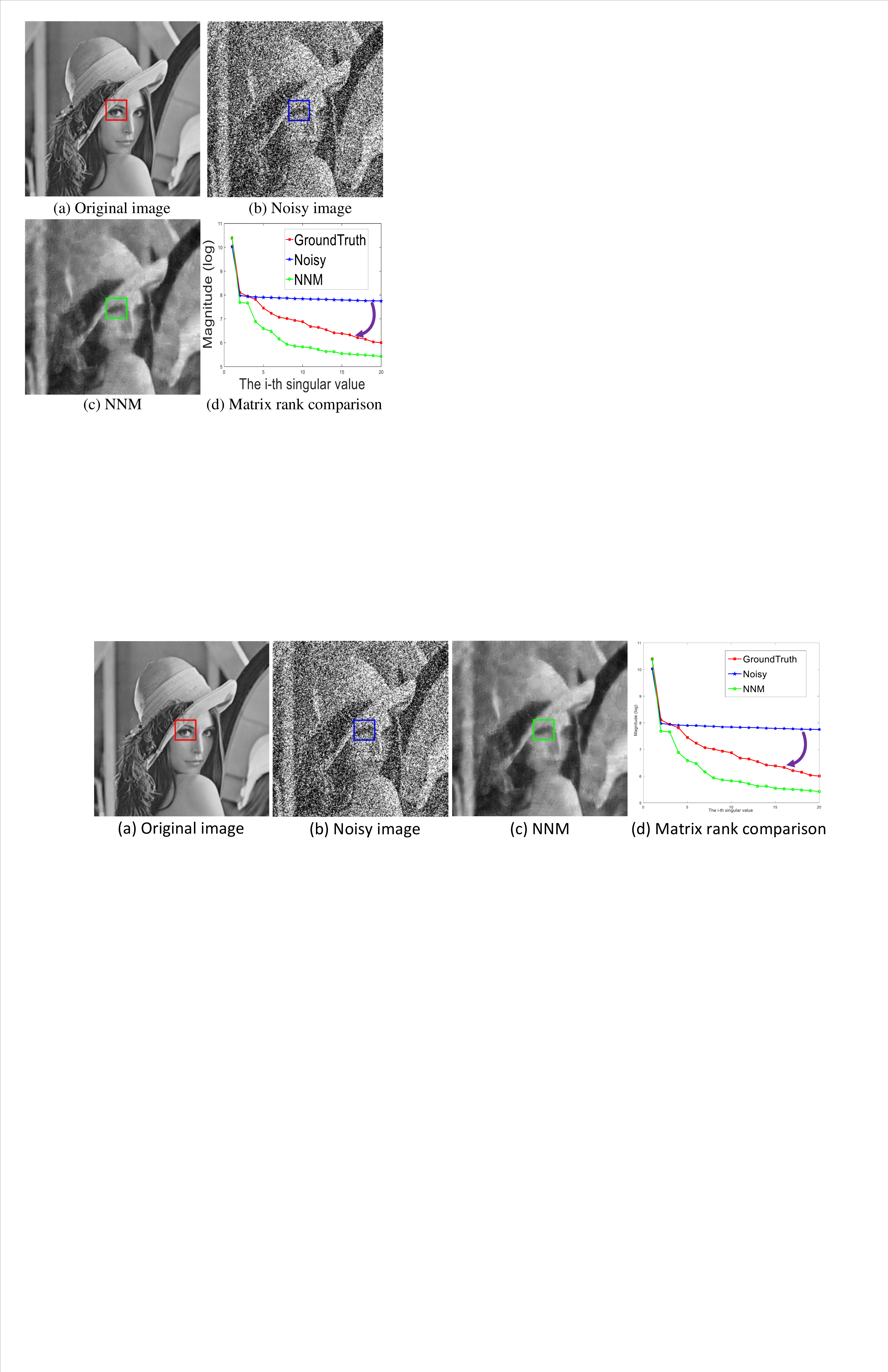}}
\vspace{-2mm}
\caption{ Analyzing the matrix rank by image denoising. }
\label{fig:2}
\vspace{-4mm}
\end{figure}

%\vspace{-4mm}
\subsection{Rank Residual Constraint}
 \label{sec:2.2}
As demonstrated in Fig.~\ref{fig:2}, due to the influence of noise, it is difficult to estimate the matrix rank precisely using NNM.
More specifically, in Fig.~\ref{fig:2}(d), the singular values of the observed matrix are seriously deviated from the singular values of the original matrix. However, in low-rank matrix estimation, we expect that the singular values of the recovered matrix {${\textbf{\emph{X}}}$} and the singular values of the original matrix $\textbf{\emph{X}}^\star$ are as close as possible. Explicitly, we define the {\em rank residual} by
\begin{equation}
\boldsymbol\gamma^\star \stackrel{\rm def}{=} \boldsymbol\sigma-\boldsymbol\psi^\star,
\label{eq:2}
\end{equation} %equal 2
where $\boldsymbol\sigma$ and $\boldsymbol\psi^\star$ are the singular values of  {${\textbf{\emph{X}}}$} and $\textbf{\emph{X}}^\star$, respectively. It can be seen that the rank estimation of the  matrix $\textbf{\emph{X}}$ largely depends on the level of this rank residual $\boldsymbol\gamma^\star$.

However, in real applications, the original matrix $\textbf{\emph{X}}^\star$ is unavailable, and thus we desire a good {\em estimate} of it, denoted by $\textbf{\emph{X}}'$. Via introducing this $\textbf{\emph{X}}'$ and defining $\boldsymbol\gamma \stackrel{\rm def}{=}\boldsymbol\sigma-\boldsymbol\psi$ with $\boldsymbol\psi$ being the singular values of $\textbf{\emph{X}}'$, we propose the RRC model below,
\begin{equation}
\hat{\textbf{\emph{X}}}=\arg\min_{\textbf{\emph{X}}}\left(\frac{1}{2}\left\|\textbf{\emph{Y}}-\textbf{\emph{X}}\right\|_F^2 +\lambda\left\|\gammav\right\|_{S_p}\right),
	\label{eq:3}
\end{equation} %equal 3
where $S_p$ denotes some type of norm. We will describe how to estimate $\textbf{\emph{X}}'$ and solve Eq.~\eqref{eq:3} below. Specifically, in the following sections, we apply the proposed RRC model to image restoration tasks, including image denoising and image compression artifacts reduction.

\section{Image Denoising via Rank Residual Constraint}
\label{sec:3}
In this section, we firstly employ the proposed RRC model in image denoising. It is well-known that image denoising \cite{26,76,27} is not only an important problem in image processing, but also an ideal test bench to measure different statistical image models.  Mathematically, image denoising aims to recover the latent clean image $\textbf{\emph{x}}$ from its noisy observation $\textbf{\emph{y}}$, which can be modeled as
\begin{equation}
\textbf{\emph{y}}= \textbf{\emph{x}}+\textbf{\emph{n}},
\label{eq:4}
\end{equation} %equal 4
where $\textbf{\emph{n}}$ is usually assumed to be zero-mean Gaussian with standard deviation $\boldsymbol\sigma_n$. Owing to the ill-posed nature of image denoising, it is critical to exploit the prior knowledge that characterizes the statistical features of the image.

The well-known NSS prior \cite{25}, which investigates the repetitiveness of textures and structures of natural images within nonlocal regions, implies that many similar patches can be searched  given a reference patch. To be concrete, a noisy (vectorized) image $\textbf{\emph{y}}\in\mathbb{R}^{N}$ is divided into $n$ overlapping patches of size $\sqrt{d}\times \sqrt{d}$, and each patch is denoted by a vector ${\textbf{\emph{y}}}_i\in\mathbb{R}^d, i=1, 2, \dots, n$. For the $i^{th}$ patch $\textbf{\emph{y}}_i$,  its $m$ similar patches are selected from a surrounding (searching) window with $L \times L$ pixels to form a set ${\textbf{\emph{S}}}_i$. Note that the K-Nearest Neighbour (KNN) algorithm \cite{70} is used to search similar patches. Following this, these patches in ${\textbf{\emph{S}}}_i$ are stacked into a matrix ${\textbf{\emph{Y}}}_i\in\mathbb{R}^{d\times m}$, \ie, ${\textbf{\emph{Y}}}_i=\{{\textbf{\emph{y}}}_{i,1}, {\textbf{\emph{y}}}_{i,2},\dots,{\textbf{\emph{y}}}_{i,m}\}.$ This matrix ${\textbf{\emph{Y}}}_i$ consisting of patches with similar structures is thus called a group \cite{28,29}, where $\{{\textbf{\emph{y}}}_{i,j}\}_{j=1}^m$ denotes the $j^{th}$ patch in the $i^{th}$ group. Then we have ${\textbf{\emph{Y}}}_i = \textbf{\emph{X}}_i +{\textbf{\emph{N}}}_i$, where $\textbf{\emph{X}}_i$ and ${\textbf{\emph{N}}}_i$ are the corresponding group matrices of the original image and noise, respectively. Since all patches in each data matrix have similar structures, the constructed data matrix ${\textbf{\emph{Y}}}_i$ is of low-rank. By adopting the proposed RRC model in Eq.~\eqref{eq:3}, the low-rank matrix ${\textbf{\emph{X}}}_i$ can be estimated by solving the following optimization problem,
\begin{equation}
\hat{\textbf{\emph{X}}}_i=\arg\min_{\textbf{\emph{X}}_i}\left(\frac{1}{2}\left\|\textbf{\emph{Y}}_i-\textbf{\emph{X}}_i\right\|_F^2 +\lambda\left\|\gammav_i\right\|_{S_p}\right),
\label{eq:5}
\end{equation} %equal 5
where $\boldsymbol\gamma_i = \boldsymbol\sigma_i-\boldsymbol\psi_i$, with $\boldsymbol\sigma_i$ and $\boldsymbol\psi_i$ representing the singular values of  $\textbf{\emph{X}}_i$ and $\textbf{\emph{X}}_i'$, respectively. $\textbf{\emph{X}}_i'\in{\mathbb R}^{d\times m}$ is a good estimate of the original image group $\textbf{\emph{X}}_i^\star$. In order to achieve a high performance for image denoising, we hope  that the rank residual $\boldsymbol\gamma_i$ of each group is small enough.

\begin{figure}[!t]
	%\vspace{-2mm}
	\centering
	{\includegraphics[width=.48\textwidth]{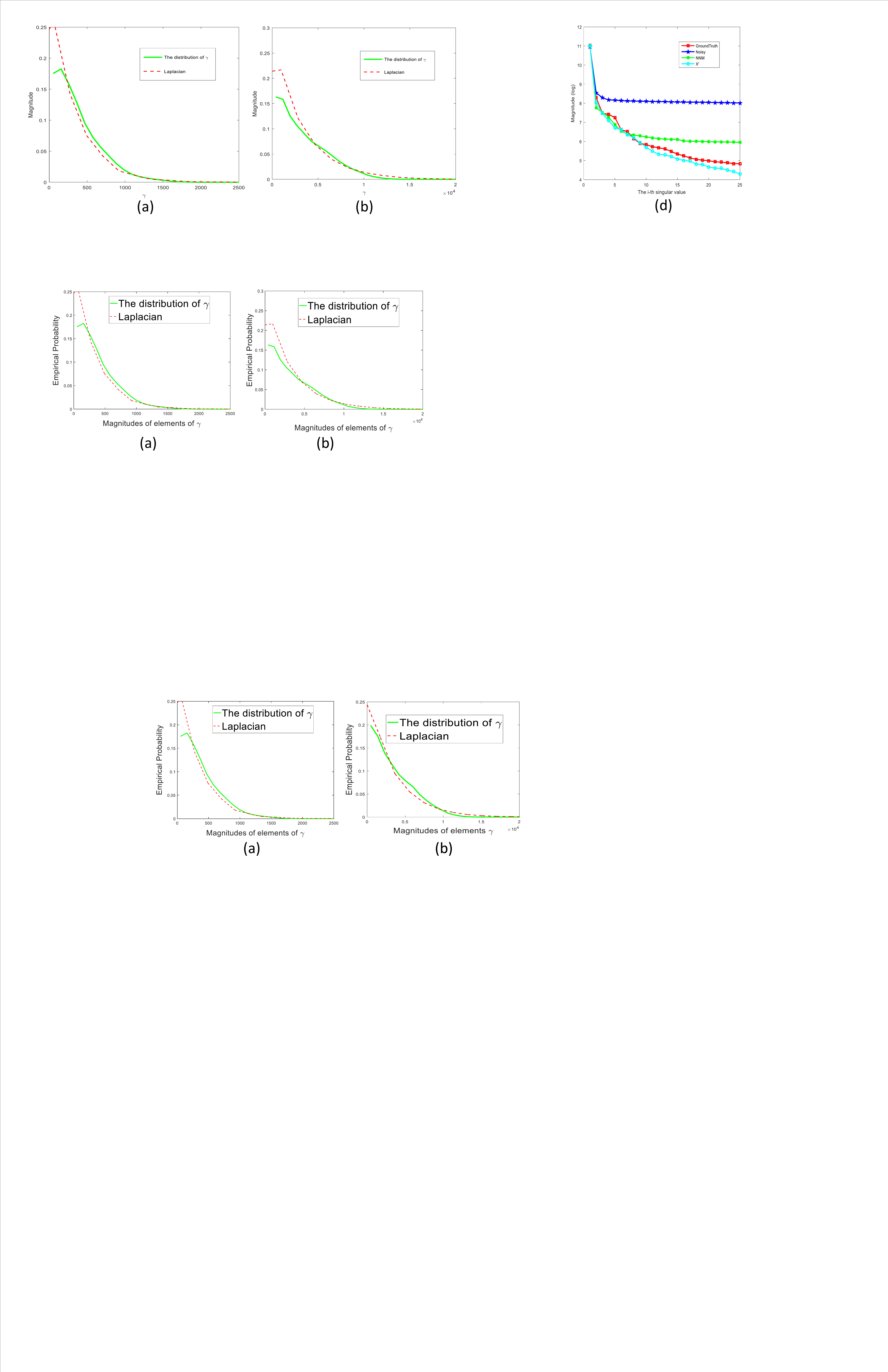}}
	%	\end{minipage}
	\vspace{-2mm}
	\caption{ The distributions of the rank residual $\boldsymbol\gamma$ for image $\emph{Fence}$ with $\boldsymbol\sigma_n$ = 20 in (a) and image $\emph{Parrot}$ with $\boldsymbol\sigma_n$ = 50 in (b).}
	\label{fig:3}
	\vspace{-2mm}
\end{figure}

\subsection{Determine $S_p$}
\label{3.1}
Let us come back to Eq.~\eqref{eq:5}. Obviously, one important issue of our RRC based image denoising is the determination of $S_p$. Hereby, we perform some experiments to investigate the statistical property of $\boldsymbol\gamma$, where  $\boldsymbol\gamma$ denotes the set of $\boldsymbol\gamma_i = \boldsymbol\sigma_i-\boldsymbol\psi_i$. We use the original image ${\textbf{\emph{x}}}$ to construct ${\textbf{\emph{X}}}'$.
In these experiments, two typical images $\emph{Fence}$ and $\emph{Parrot}$ are corrupted by Gaussian noise with standard deviations $\boldsymbol\sigma_n$ = 20 and $\boldsymbol\sigma_n$ = 50 respectively, to generate the noisy observation $\textbf{\emph{y}}$. Fig.~\ref{fig:3} shows the fitting results of empirical distributions of the rank residual $\boldsymbol\gamma$ on these two images. It can be observed that both the empirical distributions can be well approximated by a Laplacian distribution, which is usually modeled by an $\ell_1$-norm. Therefore,
Eq.~\eqref{eq:5} can now be rewritten as
\begin{equation}
\hat{\textbf{\emph{X}}}_i=\arg\min_{\textbf{\emph{X}}_i}\left(\frac{1}{2}\left\|\textbf{\emph{Y}}_i-\textbf{\emph{X}}_i\right\|_F^2 +\lambda\left\|\gammav_i\right\|_{1}\right).
\label{eq:6}
\end{equation} %equal 6

\subsection{Estimate $\textbf{\emph{X}}'$}
\label{3.2}
In Eq.~\eqref{eq:5}, after determining $S_p$, we also need to estimate $\textbf{\emph{X}}_i'$, as the original image is  unavailable in real applications.  There are a variety of methods to estimate $\textbf{\emph{X}}_i'$, which depends on the prior knowledge of the original image ${\textbf{\emph{x}}}$. For example, if there exist many example images that are similar enough to the original image $\textbf{\emph{x}}$, similar patches could be searched to construct the matrix $\textbf{\emph{X}}_i'$ from the example image set \cite{30,31}. However, under many practical situations, the example image set is inaccessible. In this paper, inspired by the fact that natural images often contain repetitive structures \cite{25,26,28,29,72,27}, we search nonlocal similar patches in the degraded image to the given patch and use the image NSS prior to estimate $\textbf{\emph{X}}_i'$. Specifically, a good estimation of each reference patch ${\textbf{\emph{x}}}_{i,j}'$  in $\textbf{\emph{X}}_i'$ can be computed by the weighted average of  the patches $\{\hat{\textbf{\emph{x}}}_{i, k}\}$ in $\hat{\textbf{\emph{X}}}_i$,  associated with each group including $m$ nonlocal similar patches in each iteration. Note that the initialization of $\hat{\textbf{\emph{X}}}_i$ is the corresponding noisy group ${\textbf{\emph{Y}}}_i$. Then we have,
\begin{equation}
{\textbf{\emph{x}}}_{i,j}' =\sum\nolimits_{k=1}^{m-j+1} \textbf{\emph{w}}_{i,k}\hat{\textbf{\emph{x}}}_{i,k},
\label{eq:7}
\end{equation} %equal 7
where ${\textbf{\emph{x}}}_{i,j}'$ and $\hat{\textbf{\emph{x}}}_{i,k}$ represent the $j$-th and $k$-th patch of $\textbf{\emph{X}}_i'$ and $\hat{\textbf{\emph{X}}}_i$, respectively. $m$ is the total number of similar patches and $\textbf{\emph{w}}_{i,k}$ is the weight, which is inversely proportional to the distance between patches $\hat{\textbf{\emph{x}}}_{i}$ and $\hat{\textbf{\emph{x}}}_{i,k}$, \ie, ${\textbf{\emph{w}}_{i,k} = \frac{1}{W}{\rm exp}(-\|\hat{\textbf{\emph{x}}}_{i}- \hat{\textbf{\emph{x}}}_{i,k}\|_2^2/h)}$, where $h$ is a predefined constant and $W$ is a normalization factor. It is worth noting that Eq.~\eqref{eq:7} is  based on nonlocal means filtering \cite{32}. Again, the recovered matrix $\hat{\textbf{\emph{X}}}_i$ and the reference matrix $\textbf{\emph{X}}_i'$ are updated \emph{gradually and jointly} in each iteration.

\subsection{Iterative Shrinkage Algorithm to Solve the Proposed RRC Model for Image Denoising}
\label{3.3}
We now develop an efficient algorithm to solve the optimization in Eq.~\eqref{eq:6}. In order to do so, we first introduce the following Lemma and Theorem.

\begin{lemma}
\label{lemma:1}
The minimization problem
\begin{equation}
\hat{\xv}=\arg\min_{\xv} \left(\frac{1}{2}\left\|\av -\xv\right\|_2^2+\tau\left\|\xv -\bv\right\|_1\right),
\label{eq:8}
\end{equation} %equal 8
has a closed-form solution,
\begin{equation}
\hat{\xv}={\rm soft}(\av-\bv,\tau)+\bv.
\label{eq:9}
\end{equation} %equal 9
where ${\rm soft}(\av,\tau)= {\rm sgn}(\av)\odot{\rm max}({\rm abs}(\av)-\tau,0)$; $\odot$ denotes the element-wise (Hadamard) product, and $\av,\bv,\xv$ are vectors of the same dimension.
\end{lemma}
\begin{proof}
See \cite{33}.
\end{proof}

\begin{theorem}
\label{theorem:1}
(von Neumann) For any two matrices $\textbf{{A}}, \textbf{{B}} \in{\mathbb R}^{m \times n}$, ${\rm Tr}({\textbf{{A}}}^T\textbf{{B}}) \leq {\rm Tr}(\sigma({\textbf{{A}}})^T\sigma({\textbf{{B}}}))$, where ${ {\rm Tr}}$ calculates the trace of the ensured matrix; $\sigma({\textbf{{A}}})$ and $\sigma({\textbf{\emph{B}}})$ are the ordered singular value matrices of $\textbf{{A}}$ and $\textbf{{B}}$ with the same order, respectively.
\end{theorem}
\begin{proof}
See \cite{34}.
\end{proof}

We now provide the solution of Eq.~\eqref{eq:6} by the following Theorem.
\begin{theorem}
\label{theorem:2}
Let $\textbf{{Y}}_i= \textbf{{U}}_i\boldsymbol\Delta_i\textbf{{V}}_i^T$ be the SVD of $\textbf{{Y}}_i\in{\mathbb R}^{d\times m}$ with $\boldsymbol\Delta_i=diag(\delta_{i,1},\dots,\delta_{i,j})$, $j= min(d,m)$, $\textbf{{X}}_i'= \textbf{{R}}_i\boldsymbol\Lambda_i\textbf{{Q}}_i^T$ be the SVD of $\textbf{{X}}_i'\in{\mathbb R}^{d\times m}$ with $\boldsymbol\Lambda_i=diag(\psi_{i,1},\dots,\psi_{i,j})$. The optimal solution $\textbf{{X}}_i$ to the problem in Eq.~\eqref{eq:6} is $\textbf{{U}}_i\boldsymbol\Sigma_i\textbf{{V}}_i^T$, where $\boldsymbol\Sigma_i=diag(\sigma_{i,1},\dots,\sigma_{i,j})$ and the $k^{th}$ diagonal element $\sigma_{i,k}$ is solved by
\begin{equation}
\begin{aligned}
&\min\limits_{\sigma_{i,k}\geq0}
\left(\frac{1}{2}(\delta_{i,k}-\sigma_{i,k})^2+\lambda |\sigma_{i,k}-\psi_{i,k}| \right),~
\forall k=1,\dots, j.
\end{aligned}
\label{eq:10}
\end{equation}%equal 10
\end{theorem}

\begin{proof}
See Appendix~\ref{theorem2}.
\end{proof}

Thereby, the minimization problem in Eq.~\eqref{eq:6} can be simplified by minimizing the problem in Eq.~\eqref{eq:10}.
For fixed $\delta_{i,k}$, $\psi_{i,k}$ and $\lambda$, based on Lemma~\ref{lemma:1}, the closed-form solution of Eq.~\eqref{eq:10} is
%\vspace{-3mm}
\begin{equation}
\sigma_{i,k}= {\rm soft} (\delta_{i,k}-\psi_{i,k}, \lambda) +  \psi_{i,k},
\label{eq:11}
\end{equation}%equal 11

With the solution of $\boldsymbol\Sigma_i$ in Eq.~\eqref{eq:11} provided, the estimated group matrix ${{\textbf{\emph{X}}}_i}$ can be reconstructed by $\hat{\textbf{\emph{X}}}_i={{\textbf{\emph{U}}}_i}\boldsymbol\Sigma_i{{\textbf{\emph{V}}}_i^T}$. Then the denoised image $\hat{\textbf{\emph{x}}}$ can be reconstructed by aggregating all the group matrices $\{{\hat{\textbf{\emph{X}}}_i}\}_{i=1}^n$.

In image denoising, we would perform the below denoising procedure several iterations to achieve better results. In the $t^{th}$ iteration, the iterative regularization strategy \cite{35} is used to update $\textbf{\emph{y}}$ by
\begin{equation}
\textbf{\emph{y}}^{t}=\hat{\textbf{\emph{x}}}^{t-1}+\mu (\textbf{\emph{y}}-\textbf{\emph{y}}^{t-1}),
\label{eq:12}
\end{equation}
where $\mu$ represents the step-size. The standard deviation of the noise in the $t^{th}$ iteration is adjusted by \cite{5,29}:  ${\boldsymbol\sigma_n^{t}}=\rho~\sqrt{({\boldsymbol\sigma_n^2-\|{{\textbf{\emph{y}}}}-{\hat{{\textbf{\emph{x}}}}}^{t-1}\|_2^2})}$,
where $\rho$ is a constant.

The parameter $\lambda$ that balances the fidelity term and the regularization term can also be adaptively determined in each iteration, and inspired by \cite{36}, $\lambda$ in each iteration is set to
\begin{equation}
\lambda=\frac{{c~2\sqrt{2}~{\boldsymbol\sigma}_n^2}}{({\boldsymbol\varphi_i+\epsilon})}.
\label{eq:13}
\end{equation}
where ${\boldsymbol\varphi}_i$ denotes the estimated standard variance of $\boldsymbol\gamma_i$, and $c, \epsilon$ are the constants.

The complete description of the proposed RRC based image denoising approach to solve the problem in Eq.~\eqref{eq:6} is presented in Algorithm~\ref{algo:1}.

\begin{center}
\begin{algorithm}[!t]
\caption{The Proposed RRC for Image Denoising.}
\begin{algorithmic}[1]
\REQUIRE Noisy image $\textbf{\emph{y}}$.
\STATE  Initialize $\hat{\textbf{\emph{x}}}^{0}=\textbf{\emph{y}}$, $\textbf{\emph{y}}^{0}=\textbf{\emph{y}}$.
\STATE  Set parameters  $\boldsymbol\sigma_n$, $c$, $d$, $m$, $L$, $h$, $\rho$, $\mu$ and $\epsilon$.
\FOR{$t=1$ \TO Max-Iter }
\STATE Iterative Regularization $\textbf{\emph{y}}^{t}=\hat{\textbf{\emph{x}}}^{t-1}+\mu (\textbf{\emph{y}}-\textbf{\emph{y}}^{t-1})$.
\FOR{Each patch $\textbf{\emph{y}}_i$ in $\textbf{\emph{y}}^{t}$}	
\STATE Find similar patches to construct matrix ${\textbf{\emph{Y}}}_i$.
\STATE Perform $[\textbf{\emph{U}}_i, \boldsymbol\Delta_i, \textbf{\emph{V}}_i]= SVD ({\textbf{\emph{Y}}}_i)$.
\STATE Estimate the reference matrix ${\textbf{\emph{X}}}_i'$ by Eq.~\eqref{eq:7}.
\STATE Perform $[\textbf{\emph{R}}_i, \boldsymbol\Lambda_i, \textbf{\emph{Q}}_i]= SVD ({\textbf{\emph{X}}}_i')$.
\STATE Update $\lambda$ by Eq.~\eqref{eq:13}.%$\lambda={{c~2\sqrt{2}~{\boldsymbol\sigma_n}^2}}/({\boldsymbol\varphi_i+\epsilon})$.
\STATE Estimate  $\boldsymbol\Sigma_i$ by  Eq.~\eqref{eq:11}.
\STATE Get the estimation:   $\hat{\textbf{\emph{X}}}_i=\textbf{\emph{U}}_i\boldsymbol\Sigma_i\textbf{\emph{V}}_i^T$.
\ENDFOR
\STATE Aggregate $\hat{\textbf{\emph{X}}}_i$  to form the denoised image $\hat{\textbf{\emph{x}}}^{t}$.
\ENDFOR
\STATE $\textbf{Output:}$ The final denoised image $\hat{\textbf{\emph{x}}}$.
\end{algorithmic}
\label{algo:1}
\end{algorithm}
\end{center}
%\vspace{-4mm}

\section{Image compression artifacts reduction via Rank Residual Constraint}
 \label{sec:4}

With the rapid development of social network and mobile Internet, billions of image and video resources have been spread through miscellaneous ways on the Internet everyday. To save the limited bandwidth and storage space, lossy compression scheme (\eg, JPEG \cite{37}, WebP \cite{38} and HEVC-MSP \cite{39}) has been widely used to compress images and videos. However, these lossy compression techniques often give rise to visually annoying compression artifacts, \ie, sacrificing the image quality to satisfy the bit-budget, which severely degrade the user experience. Furthermore, the performance of many other computer vision tasks (\eg, image recognition \cite{40} and object detection \cite{41}) largely depends on the quality of  input images, and therefore it is desired  to recover visually pleasing artifact-free images  from these compressed images. It is well-known that JPEG is the most popular lossy compression method \cite{37}, and therefore in this work we focus on JPEG compression artifacts reduction.

Specifically, we apply the proposed RRC model to image compression artifacts reduction. Similar to the above procedure in Section~\ref{sec:3}, given a JPEG compressed (vectorized) image  $\textbf{\emph{x}}\in\mathbb{R}^{N^2}$,  we extract each patch  ${\textbf{\emph{x}}}_k\in\mathbb{R}^{\sqrt{d}\times\sqrt{d}}$ %\forall k=1,...,K$%
from it, and search for its similar patches to generate $K$ groups, where each group ${\textbf{\emph{X}}}_k\in{\mathbb R}^{d\times m}$, \ie,  ${\textbf{\emph{X}}}_k=\{{\textbf{\emph{x}}}_{k, 1}, {\textbf{\emph{x}}}_{k, 2},\dots,{\textbf{\emph{x}}}_{k, m}\}$. In our application of image compression artifacts reduction, the observed data ${\textbf{\emph{X}}}_k$ includes the quantization error, which can be modeled as noise. We thus model the data matrix by ${\textbf{\emph{X}}}_k = \textbf{\emph{Z}}_k +{\textbf{\emph{N}}}_k$, where $\textbf{\emph{Z}}_k$ and ${\textbf{\emph{N}}}_k$ are the corresponding group matrices of the original (desired clean) image and the noise assumed to be zero-mean Gaussian noise here \cite{42,43,44,45}. Since all patches in each data matrix have similar structures, the constructed data matrix ${\textbf{\emph{X}}}_k$ is of low-rank. By invoking the proposed RRC model in Eq.~\eqref{eq:3}, the low-rank matrix ${\textbf{\emph{Z}}}_k$ can be estimated by solving the following optimization problem,
\begin{equation}
{\hat{\textbf{\emph{Z}}}}_k=\arg\min_{\textbf{\emph{Z}}_k}\left(\frac{1}{{2\sigma^2_{\textbf{\emph{e}}}}}\left\|\textbf{\emph{X}}_k-\textbf{\emph{Z}}_k\right\|_F^2 +\lambda\left\|\gammav_k\right\|_{S_p}\right),
\label{eq:15}
\end{equation} %equal 15
where ${{\sigma^2_{\textbf{\emph{e}}}}}$ denotes the variance of additive Gaussian noise, $\boldsymbol\gamma_k=\boldsymbol\psi_k-\boldsymbol\delta_k$, with $\boldsymbol\psi_k$ and $\boldsymbol\delta_k$ representing the singular values of  $\textbf{\emph{Z}}_k$ and $\textbf{\emph{Z}}_k'$ respectively, and $\textbf{\emph{Z}}_k'\in{\mathbb R}^{d\times m}$ is a good estimate of the original image group $\textbf{\emph{Z}}_k^\star$. In order to achieve a high quality recovered image, we also expect that the rank residual $\boldsymbol\gamma_k$ of each group is as small as possible.

\begin{figure}[!t]%[htbp!]
	%\begin{minipage}[b]{1\linewidth}
	\vspace{-2mm}
	\centering
	{\includegraphics[width=0.48\textwidth]{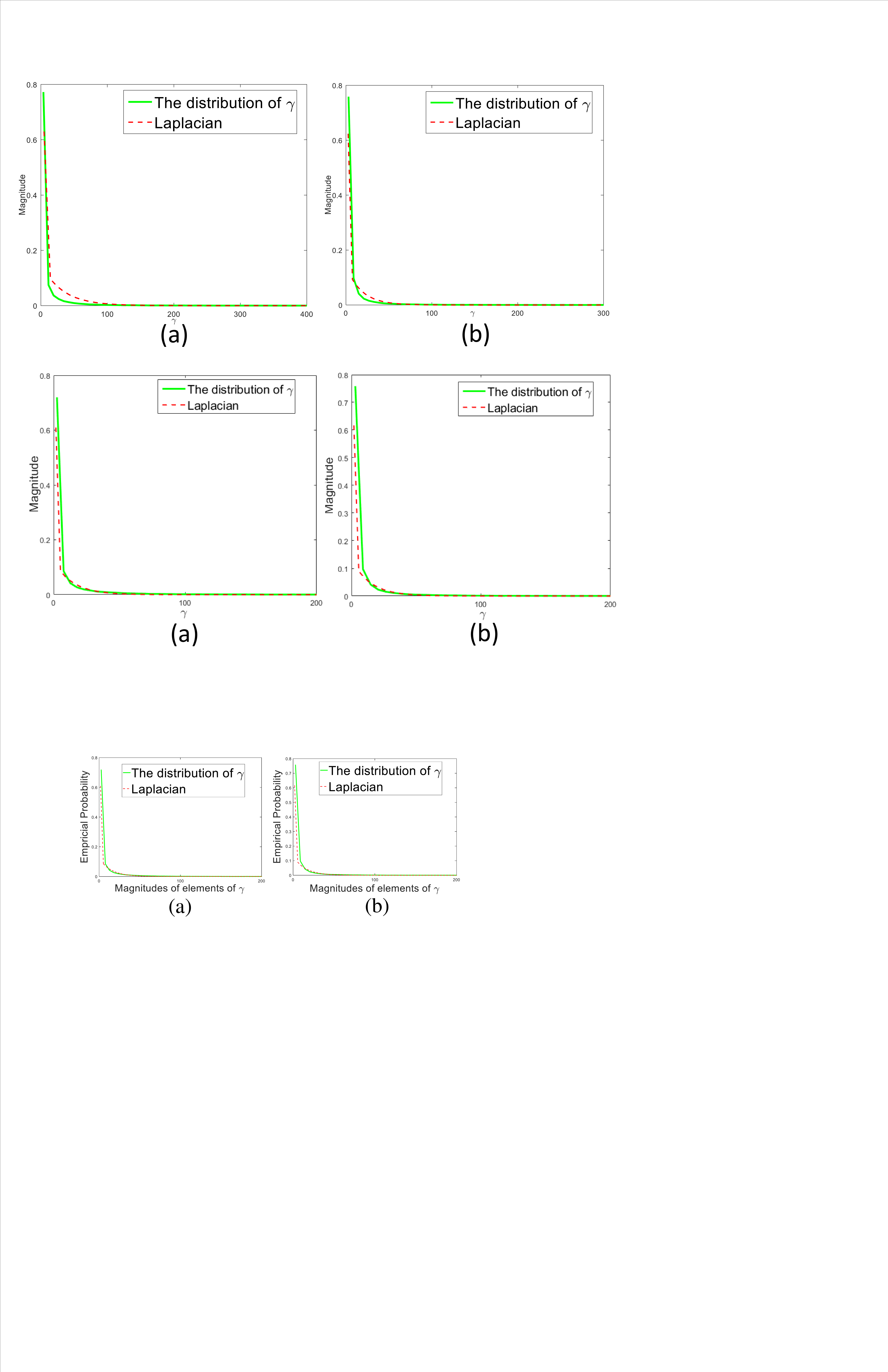}}
	%	\end{minipage}
	\vspace{-2mm}
	\caption{ The distributions of the rank residual $\boldsymbol\gamma$. (a) image $\emph{Lena}$  is compressed by JPEG with QF = 10; (b) image $\emph{House}$ is compressed by JPEG with QF = 20. }
	\label{fig:4}
	\vspace{-4mm}
\end{figure}

According to Eq.~\eqref{eq:15}, one can observe that $\textbf{\emph{Z}}_k'$ and $S_p$ are unknown in our proposed RRC model for image compression artifacts reduction, and thus we need to estimate them. Similar to subsection~\ref{3.2}, we exploit the image NSS prior to estimate  $\textbf{\emph{Z}}_k'$ from the recovered image $\hat{\textbf{\emph{z}}}$ in each iteration, \ie,

\begin{equation}
{\textbf{\emph{z}}}_{k,j}' =\sum\nolimits_{i=1}^{m-j+1} \textbf{\emph{w}}_{k,i}\hat{\textbf{\emph{z}}}_{k,i},
\label{eq:15.1}
\end{equation} %equal 15.1
where ${\textbf{\emph{z}}}_{k,j}'$ and $\hat{\textbf{\emph{z}}}_{k,i}$ represent the $j$-th and $i$-th patch of $\textbf{\emph{Z}}_k'$ and $\hat{\textbf{\emph{Z}}}_k$, respectively. Note that the initialization of $\hat{\textbf{\emph{z}}}$ is the JPEG compressed image.

For $S_p$, we perform some experiments to investigate the statistical property of $\boldsymbol\gamma$, where  $\boldsymbol\gamma$ denotes the set of $\boldsymbol\gamma_k=\boldsymbol\psi_k-\boldsymbol\delta_k$ and we use the original image to construct ${\textbf{\emph{Z}}}'$. Two typical images $\emph{Lena}$ and $\emph{House}$ are compressed by JPEG with quality factor (QF) = 10 and QF = 20 respectively, to generate the degraded images. Fig.~\ref{fig:4} shows the fitting results of empirical distributions of the rank residual $\boldsymbol\gamma$ on these two images. One can also observe that both empirical distributions can be well approximated by a Laplacian distribution, which is usually modeled by an $\ell_1$-norm. Therefore, Eq.~\eqref{eq:15} can be rewritten as
\begin{equation}
{\hat{\textbf{\emph{Z}}}}_k=\arg\min_{\textbf{\emph{Z}}_k}\left(\frac{1}{{2\sigma^2_{\textbf{\emph{e}}}}}\left\|\textbf{\emph{X}}_k-\textbf{\emph{Z}}_k\right\|_F^2 +\lambda\left\|\gammav_k\right\|_{1}\right).
\label{eq:16}
\end{equation} %equal 16

\subsection {Quantization Constraint for Image Compression Artifacts Reduction}
\label{4.1}

This subsection is presented by following \cite{45}, to make the work self-contained and accessible to readers as much as possible. We now formulate image compression artifacts reduction as an image inverse problem. Specifically, given the JPEG compressed image $\textbf{\emph{y}}$, the original image $\textbf{\emph{x}}$ can be obtained by solving $\hat{\textbf{\emph{x}}} =\arg\max_{\textbf{\emph{x}}} p(\textbf{\emph{x}}|\textbf{\emph{y}})$. Taking account of the prior $p(\xv)$ and exploiting the Bayesian rule, we have
\begin{equation}
\hat{\textbf{\emph{x}}} =\arg\max_{\textbf{\emph{x}}}[{\rm log}( p(\textbf{\emph{y}}|\textbf{\emph{x}})) +  {\rm log} (p(\textbf{\emph{x}}))],
\label{eq:17}
\end{equation}%equal 17
where the first term is the data fidelity and the second term depends on the employed image priors. In the following, we will introduce how to design each term of Eq.~\eqref{eq:17}.

\subsubsection {Quantization Noise Model}
\label{4.1.1}
In our task, the observed JPEG-based compression image is usually modeled as the corrupted image with quantization noise,
\begin{equation}
{\textbf{\emph{y}}} ={\textbf{\emph{x}}} +  {\textbf{\emph{s}}},
\label{eq:18}
\end{equation}%equal 18
where ${\textbf{\emph{y}}}$ is the JPEG-based compression image with blocking artifacts and  ${\textbf{\emph{x}}}$, ${\textbf{\emph{s}}}$ represent the original image and the quantization noise, respectively. Many sophisticated models of the quantization noise have been developed \cite{42,43,44,45,67,68,74}, and in particular a Gaussian model has been widely used owing to its simplicity and effectiveness, which has achieved state-of-the-art results \cite{42,43,44,45}. Therefore, the Gaussian model is adopted in our work and specifically we employ the approach proposed in \cite{44} to estimate the noise variance $\sigma^2_{\textbf{\emph{s}}}$,
\begin{equation}
\sigma^2_{\textbf{\emph{s}}} = 1.195 ~(\tilde{{\textbf{\emph{e}}}})^{0.6394} + 0.9693, \  \tilde{{\textbf{\emph{e}}}} =\frac{1}{9} \sum\nolimits_{i,j=1}^3 {\textbf{\emph{M}}}_{[i,j]}^q,
\label{eq:19}
\end{equation}%equal 19
where ${\textbf{\emph{M}}}^q$ is the $8 \times 8$ quantization matrix with QF of $q$,  $\tilde{{\textbf{\emph{e}}}}$ is the mean value of the nine upper-left entries in ${\textbf{\emph{M}}}^q$, corresponding to lowest-frequency DCT harmonics, and we use ${\textbf{\emph{M}}}_{[i,j]}^q$ to denote the $(i,j)^{th}$ element in $\textbf{\emph{M}}^q$. It is worth noting that the noise variance $\sigma^2_{\textbf{\emph{s}}} $ obtained by Eq.~\eqref{eq:19} is only the variance of the hypothetical Gaussian noise, which determines the level of adaptive smoothing that is able to reduce compression artifacts generated by the quantization step with ${\textbf{\emph{M}}}^q$ \cite{42}.

Following this, the first data-fidelity term in Eq.~\eqref{eq:17}  can be formulated by
\begin{equation}
{\rm log}( p(\textbf{\emph{y}}|\textbf{\emph{x}})) = -\frac{1} {2\sigma^2_{\textbf{\emph{s}}}} \left\|{\textbf{\emph{y}}}- {\textbf{\emph{x}}}\right\|_2^2,
\label{eq:20}
\end{equation}%equal 20
where  ${\sigma^2_{\textbf{\emph{s}}}}$ is adaptively determined by Eq.~\eqref{eq:19} and we have discarded the terms unrelated to $\xv$.

\subsubsection {Quantization Constraint Prior} \label{4.1.2}

%\bihan{The quantization constraint (QC)~\cite{43,44,45} is applied to
%In order to tackle the quantization issue in JPEG compression, the quantization constraint (QC) \cite{43,44,45} prior is adopted in \bihan{the proposed} model.
\bihan{We adopt the the quantization constraint (QC) \cite{43,44,45} prior in the proposed model, to tackle the quantization issue in the JPEG scheme.}
Specifically, to simplify the first two stages of JPEG compression, let us define an $N \times N$ \bihan{block-wise} DCT operator $\mathcal{A}$, which transforms each \bihan{non-overlapped} $8\times8$ block of the input image into \bihan{their} frequency domain \cite{44,45,46}.
\bihan{Correspondingly}, the matrix operator $\mathcal{A}^{-1}$ \bihan{denotes} the inverse DCT process.
\bihan{Therefore, the frequency-domain image of $\textbf{\emph{x}}$  is obtained by the following transform,}
\begin{equation}
\widehat{{\textbf{\emph{x}}}}=\mathcal{A} \textbf{\emph{x}}, %\quad \widehat{\textbf{\emph{y}}}=\mathcal{A} \textbf{\emph{y}}.
\label{eq:21}
\end{equation} %equal 21

\bihan{Recall that $\textbf{\emph{M}}^q$ denotes the $8\times 8$ block-level quantization matrix, which is determined by the QF in the range of [1, 100].
Based on the JPEG compression scheme, we have the following forward process,}
\begin{equation}
\widehat{\textbf{\emph{y}}}_{[(k-1)*N+l]} = {\rm round}\left(\frac{\widehat{\textbf{\emph{x}}}_{[(k-1)*N+l]}}{\textbf{\emph{M}}_{[k,l]}}\right)~{\textbf{\emph{M}}_{[k,l]}},
\label{eq:22}
\end{equation} %equal 22
\bihan{where $1\leq k$, $l\leq N$ are the pixel coordinates in the image, and round($\cdot$) denotes the rounding operator mapping the input to its nearest integer. Here $\textbf{\emph{M}}$ is an $N \times N$ image-level quantization matrix, and the values of its elements ${\textbf{\emph{M}}_{[k,l]}} = {\textbf{\emph{M}}^q_{[\bar{k},\bar{l}]}}$, where the block-level coordinates $\bar{k} = mod (k,8)$; $\bar{l}= mod (l,8)$.}

\bihan{We define the lower and upper bound vectors $\widehat{\textbf{\emph{l}}}$ and $\widehat{\textbf{\emph{u}}}$ as}
%The lower and upper bound vectors $\hat{\textbf{\emph{l}}}$ and $\hat{\textbf{\emph{u}}}$ are defined as
\begin{equation}
\begin{aligned}
\widehat{\textbf{\emph{l}}}_{[(k-1)*N+l]}&= (\widehat{\textbf{\emph{y}}}_{[(k-1)*N+l]} - w)*{\textbf{\emph{M}}_{[k,l]}}\; ,\\
\widehat{\textbf{\emph{u}}}_{[(k-1)*N+l]}&= (\widehat{\textbf{\emph{y}}}_{[(k-1)*N+l]} + w)*{\textbf{\emph{M}}_{[k,l]}} \; .
\label{eq:23}
\end{aligned}
\end{equation} %23
\bihan{where $w$ is a constant which is typically set to be not greater than 0.5~\cite{47}.
The frequency-domain coefficients of the image should satisfy the bound (\ie, QC)~\cite{46}:}
\begin{equation}
\widehat{\textbf{\emph{l}}} \preceq  \widehat{\textbf{\emph{x}}} \preceq \widehat{\textbf{\emph{u}}},
\label{eq:24}
\end{equation} %equal 24
\bihan{where $\preceq$ denotes the element-wise $\leq$ operator.
Based on \eqref{eq:24}, the solution space of $\textbf{\emph{x}}$, which is denoted as $\bm \Omega$, is defined as}
\begin{equation}
\bm \Omega = \{\textbf{\emph{x}} \, |  \, \widehat{\textbf{\emph{l}}} \preceq  \mathcal{A}{\textbf{\emph{x}}} \preceq \widehat{\textbf{\emph{u}}}\} \, ,
\label{eq:25}
\end{equation} %equal 15
\bihan{which holds for any image coded by JPEG. In this work,} we explicitly incorporate the feasible solution space $\bm \Omega$ into Eq.~\eqref{eq:17}.

\subsection{Joint Model and Algorithm for Image Compression Artifacts Reduction}
\label{4.2}
Now, we incorporate the quantization noise model in Eq.~\eqref{eq:20} and two image priors (\ie, the proposed RRC in Eq.~\eqref{eq:16} and QC in Eq.~\eqref{eq:25}) into Eq.~\eqref{eq:17}, and achieve the joint image compression artifacts reduction as follows:
\begin{equation}
\begin{aligned}
(\hat{\textbf{\emph{x}}}, \hat{\textbf{\emph{Z}}}_k)&= \arg\min_{{\textbf{\emph{x}}}, {\textbf{\emph{Z}}}_k} \frac{1} {2\sigma^2_{\textbf{\emph{s}}}} \left\|{\textbf{\emph{y}}}- {\textbf{\emph{x}}}\right\|_2^2\\
&~~~+ \frac{\rho}{2\sigma^2_{\textbf{\emph{e}}}} \Sigma_{k=1}^K \left\|\textbf{\emph{R}}_k\textbf{\emph{x}}- \textbf{\emph{Z}}_k\right\|_F^2
+\lambda \Sigma_{k=1}^K \left\|\gammav_k\right\|_{1}, \\
&~~~{ s. t.} \quad\textbf{\emph{x}}\in {\bm \Omega},
\label{eq:26}
\end{aligned}
\end{equation}
where $\textbf{\emph{R}}_k\textbf{\emph{x}}= \textbf{\emph{X}}_k = \{{\textbf{\emph{x}}}_{k,1}, {\textbf{\emph{x}}}_{k,2},\dots,{\textbf{\emph{x}}}_{k,m}\}$ is the operator that extracts the group $\textbf{\emph{X}}_k$ from $\textbf{\emph{x}}$ and $\{\rho, \lambda\}$ are positive constants. Here, we introduce $\rho$ to make the solution of Eq.~\eqref{eq:26} more feasible.

In the following, we develop an alternating minimizing strategy to solve the  large scale non-convex optimization problem in Eq.~\eqref{eq:26}.
Specifically, the minimization of Eq.~\eqref{eq:26} is divided into two sub-problems, \emph{i.e.}, $\textbf{\emph{x}}$ and $\textbf{\emph{Z}}_k$, and we will show that there is an efficient solution to each of them.

\subsubsection{$\textbf{{Z}}_k$ Sub-problem}
\label{4.2.1}
Given $\textbf{\emph{x}}$, each $\textbf{\emph{Z}}_k$ sub-problem can be expressed as:
\begin{equation}
\hat{\textbf{\emph{Z}}}_k = \arg\min_{{\textbf{\emph{Z}}}_k}  \frac{\rho}{2\sigma^2_{\textbf{\emph{e}}}} \left\|\textbf{\emph{R}}_k\textbf{\emph{x}}- \textbf{\emph{Z}}_k\right\|_F^2 +\lambda \left\|\gammav_k\right\|_{1},
\label{eq:27}
\end{equation} %27
Let $\mu = \frac{\lambda \sigma^2_{\textbf{\emph{e}}}}{\rho}$, we can rewrite Eq.~\eqref{eq:27} as
\begin{equation}
\hat{\textbf{\emph{Z}}}_k = \arg\min_{{\textbf{\emph{Z}}}_k}  \frac{1}{2} \left\|\textbf{\emph{X}}_k- \textbf{\emph{Z}}_k\right\|_F^2 +\mu \left\|\gammav_k\right\|_{1}.
\label{eq:28}
\end{equation} %28

Then, based on Theorem~\ref{theorem:2}, the minimization problem in Eq.~\eqref{eq:28} can be simplified to minimize the following problem,
\begin{equation}
\begin{aligned}
&\min\limits_{\psi_{k,i}\geq0}
\left(\frac{1}{2}(\sigma_{k,i}-\psi_{k,i})^2+\mu |\psi_{k,i}-\delta_{k,i}| \right),~
\forall i=1,\dots, j,
\end{aligned}
\label{eq:29}
\end{equation}%equal 29
where $\textbf{\emph{X}}_k= \textbf{\emph{U}}_k\boldsymbol\Sigma_k\textbf{\emph{V}}_k^T$ is the SVD of $\textbf{\emph{X}}_k\in{\mathbb R}^{d\times m}$ with $\boldsymbol\Sigma_k=diag(\sigma_{k,1},\dots,\sigma_{k,j})$, and $j= min(d,m)$, $\textbf{\emph{Z}}_k'= \textbf{\emph{P}}_k\boldsymbol\Delta_k\textbf{\emph{Q}}_k^T$ is the SVD of $\textbf{\emph{Z}}_k'\in{\mathbb R}^{d\times m}$ with $\boldsymbol\Delta_k=diag(\delta_{k,1},\dots,\delta_{k,j})$.

For fixed $\sigma_{k,i}$, $\delta_{k,i}$ and $\mu$, based on Lemma~\ref{lemma:1}, the closed-form solution of Eq.~\eqref{eq:29} is
\begin{equation}
\psi_{k,i}= {\rm soft} (\sigma_{k,i}-\delta_{k,i}, \mu) +  \delta_{k,i},
\label{eq:30}
\end{equation}%equal 30
where $\boldsymbol\Lambda_k=diag(\psi_{k,1},\dots,\psi_{k,j})$. With the solution of $\boldsymbol\psi_k$ in Eq.~\eqref{eq:30} achieved, the estimated matrix ${{\textbf{\emph{Z}}}_k}$ can be reconstructed by $\hat{\textbf{\emph{Z}}}_k={{\textbf{\emph{U}}}_k}\boldsymbol\Lambda_k{{\textbf{\emph{V}}}_k^T}$.

\subsubsection{$\textbf{{x}}$ Sub-problem}
\label{4.2.2}
After obtaining each $\textbf{\emph{Z}}_k$, the desired image $\textbf{\emph{x}}$ can be reconstructed by solving the following constrained quadratic minimization problem,
 \begin{equation}
\begin{aligned}
&\hat{\textbf{\emph{x}}}= \arg\min_{{\textbf{\emph{x}}}} \frac{1} {2\sigma^2_{\textbf{\emph{s}}}} \left\|{\textbf{\emph{y}}}- {\textbf{\emph{x}}}\right\|_2^2+ \frac{\rho}{2\sigma^2_{\textbf{\emph{e}}}} \Sigma_{k=1}^K \left\|\textbf{\emph{R}}_k\textbf{\emph{x}}- \textbf{\emph{Z}}_k\right\|_F^2,\\
&\ \ \ s.t. \ \ \ \textbf{\emph{x}}\in {\bm \Omega}.
\label{eq:31}
\end{aligned}
\end{equation}

We first obtain the solution of the unconstrained quadratic minimization of Eq.~\eqref{eq:31}, and later project the solution to ${\bm \Omega}$. Specifically, without considering the constraint of ${\bm \Omega}$, Eq.~\eqref{eq:31} can be rewritten as
 \begin{equation}
\tilde{\textbf{\emph{x}}}= \arg\min_{{\textbf{\emph{x}}}} \frac{1} {2\sigma^2_{\textbf{\emph{s}}}} \left\|{\textbf{\emph{y}}}- {\textbf{\emph{x}}}\right\|_2^2+ \frac{\rho}{2\sigma^2_{\textbf{\emph{e}}}} \Sigma_{k=1}^K \left\|\textbf{\emph{R}}_k\textbf{\emph{x}}- \textbf{\emph{Z}}_k\right\|_F^2,
\label{eq:32}
\end{equation}
which has a closed-form solution, \ie,
\begin{equation}
\tilde{\textbf{\emph{x}}}= \left(\textbf{\emph{I}} +  \frac{\sigma^2_{\textbf{\emph{s}}} \rho}{ \sigma^2_{\textbf{\emph{e}}}} \Sigma_{k=1}^K\textbf{\emph{R}}_k^T \textbf{\emph{R}}_k \right)^{-1} \left(\textbf{\emph{y}}+ \frac{\sigma^2_{\textbf{\emph{s}}} \rho}{ \sigma^2_{\textbf{\emph{e}}}} \Sigma_{k=1}^K\textbf{\emph{R}}_k^T \textbf{\emph{Z}}_k\right),
\label{eq:33}
\end{equation}
where $\textbf{\emph{I}}$ represents an identity matrix, $\textbf{\emph{R}}_k^T \textbf{\emph{Z}}_k= \Sigma_{i=1}^m \textbf{\emph{R}}_{k,i}^T\textbf{\emph{z}}_{k,i}$ and $\textbf{\emph{R}}_k^T \textbf{\emph{R}}_k = \Sigma_{i=1}^m \textbf{\emph{R}}_{k,i}^T\textbf{\emph{R}}_{k,i}$. Since $(\textbf{\emph{I}} +  \frac{\sigma^2_{\textbf{\emph{s}}} \rho}{ \sigma^2_{\textbf{\emph{e}}}} \Sigma_{k=1}^K\textbf{\emph{R}}_k^T \textbf{\emph{R}}_k) $ is actually a diagonal matrix, Eq.~\eqref{eq:33} can be solved efficiently by element-wise division.

Next, we exploit the projection operator to obtain the solution of  Eq.~\eqref{eq:31} on  ${\bm \Omega}$, that is,
\begin{equation}
\hat{\textbf{\emph{x}}}= \mathcal{A}^{-1} \mathcal{P}(\mathcal{A}(\tilde{\textbf{\emph{x}}}),\widehat{\textbf{\emph{l}}}, \widehat{\textbf{\emph{u}}}),
\label{eq:34}
\end{equation}
where $\mathcal{A}$ is the matrix operator defined in Eq.~\eqref{eq:21}, and $\textbf{\emph{v}} = \mathcal{P} ({\textbf{\emph{x}}}, {\textbf{\emph{l}}}, {\textbf{\emph{u}}} )$ is a projection operator defined by
\begin{equation}
\textbf{\emph{v}}_{[k]} =\left\{
\begin{array}{lcc}
\textbf{\emph{l}}_{[k]}, & if & \textbf{\emph{x}}_{[k]} < \textbf{\emph{l}}_{[k]},\\
\textbf{\emph{x}}_{[k]}, & if & \textbf{\emph{l}}_{[k]} \leq \textbf{\emph{x}}_{[k]} \leq  \textbf{\emph{u}}_{[k]}, \\
\textbf{\emph{u}}_{[k]}, & if &  \textbf{\emph{x}}_{[k]} > \textbf{\emph{u}}_{[k]}.\\
\end{array}
\right.  \ \ 1\leq k \leq N^2.
\label{eq:35}
\end{equation}

From Eqs.~\eqref{eq:33} and~\eqref{eq:34}, we can obtain the solution of Eq.~\eqref{eq:31}. Regarding the two noise parameters ${\sigma_{\textbf{\emph{s}}}}$ and ${ \sigma_{\textbf{\emph{e}}}}$, ${ \sigma_{\textbf{\emph{s}}}}$ is adaptively calculated by Eq.~\eqref{eq:19}. Since each ${\sigma_{\textbf{\emph{e}}}}$ in Eq.~\eqref{eq:27} varies as the iteration number increases, its setting can be adaptively adjusted in each iteration. Inspired by \cite{44,48}, ${\sigma_{\textbf{\emph{e}}}}$ in the $t^{th}$ iteration is set to
\begin{equation}
{\sigma_{\textbf{\emph{e}}}}^{(t)}= \eta~\sqrt{({{ \sigma^2_{\textbf{\emph{s}}}}-\|{{\textbf{\emph{z}}}}^{(t-1)}-{{{\textbf{\emph{y}}}}}\|_2^2})},
\label{eq:36}
\end{equation}
where $\eta$ is a constant and  the image ${{{\textbf{\emph{z}}}}}^{(t-1)}$ is reconstructed by all the ${{{\textbf{\emph{Z}}}}}_k^{(t-1)}$ at the $(t-1)^{th}$ iteration. After solving the two sub-problems, we summarize the complete algorithm to solve Eq.~\eqref{eq:26} in  Algorithm~\ref{algo:2}.

\begin{algorithm}[!t]
	\caption{The Proposed RRC for Image Compression Artifacts Reduction.}
	\begin{algorithmic}[1]
		\REQUIRE JPEG Compressed bit-stream.
		\STATE  Get $\textbf{\emph{y}}$,  $\textbf{\emph{M}}^q$ from the compressed bit-stream.
		\STATE  Set parameters  $d$, $m$, $L$, $h$, $c$, $\rho$, $\epsilon$ and $\tau$.
		\STATE  Initialize $\hat{\textbf{\emph{x}}}^{0}=\textbf{\emph{y}}$, $\textbf{\emph{Z}}_k^{0}=\textbf{\emph{X}}_k^{0}$, $\forall k =1,\dots, K$.
		\STATE  Calculate ${\sigma_{\textbf{\emph{s}}}}$ by Eq.~\eqref{eq:19}.
		\STATE  Determine the solution space ${\bm \Omega}$ by  Eq.~\eqref{eq:25}.
		\FOR{$t=1$ \TO MaxIter ($T$) }
		\STATE  Calculate ${\sigma_{\textbf{\emph{e}}}}$ by Eq.~\eqref{eq:36}.
		\FOR{Each patch $\textbf{\emph{x}}_k$ in $\textbf{\emph{x}}$}	
		\STATE Find similar patches to construct matrix ${\textbf{\emph{X}}}_k$.
		\STATE Perform $[\textbf{\emph{U}}_k, \boldsymbol\Sigma_k, \textbf{\emph{V}}_k]= SVD ({\textbf{\emph{X}}}_k)$.
		\STATE Estimate the reference matrix ${\textbf{\emph{Z}}}_k'$ by Eq.~\eqref{eq:15.1}.
		\STATE Perform $[\textbf{\emph{P}}_k, \boldsymbol\Delta_k, \textbf{\emph{Q}}_k]= SVD ({\textbf{\emph{Z}}}_k')$.
		\STATE Update $\lambda$ by Eq.~\eqref{eq:13}.
		\STATE Calculate $\mu = {\lambda \sigma^2_{\textbf{\emph{e}}}}/{\rho}$.
		\STATE Estimate  $\boldsymbol\Lambda_k$ by  Eq.~\eqref{eq:30}.
		\STATE Get the estimation:   $\hat{\textbf{\emph{Z}}}_k=\textbf{\emph{U}}_k\boldsymbol\Lambda_k\textbf{\emph{V}}_k^T$.
		\ENDFOR
		\STATE Update $\tilde{{\textbf{\emph{x}}}}$  by Eq.~\eqref{eq:33}.
		\STATE Update $\hat{{\textbf{\emph{x}}}}$  by Eq.~\eqref{eq:34}.
		\ENDFOR
		\STATE $\textbf{Output:}$ The final reconstructed image $\hat{\textbf{\emph{x}}}$.
	\end{algorithmic}
	\label{algo:2}
\end{algorithm}

\begin{table*}[!t]
\vspace{-4mm}
\caption{{PSNR ($\textnormal{d}$B) comparison of NNM, BM3D \cite{26}, LSSC \cite{27}, EPLL  \cite{56}, Plow \cite{57}, NCSR \cite{49}, GID \cite{58}, PGPD \cite{29}, LINC \cite{59}, aGMM \cite{60}, OGLR \cite{61}, WNNM \cite{5} and RRC for image denoising.}}
%\vspace{-2mm}
\resizebox{1\textwidth}{!}				
{
%\LARGE
%\Huge
%\large
%\footnotesize

\centering  % ±í¾ÓÖÐ
\begin{tabular}{|c|c|c|c|c|c|c|c|c|c|c|c|c|c||c|c|c|c|c|c|c|c|c|c|c|c|c|c|c|c|c|c|c|}
\hline
\multicolumn{1}{|c|}{}&\multicolumn{13}{|c||}{${\bf \sigma}_n = 20$}&\multicolumn{13}{|c|}{${\bf \sigma}_n = 30$}\\
\hline						\multirow{1}{*}{\textbf{{Images}}}&\multirow{1}{*}{\textbf{{NNM}}}&\multirow{1}{*}{\textbf{{BM3D}}}&\multirow{1}{*}{\textbf{{LSSC}}}
&\multirow{1}{*}{\textbf{{EPLL}}}&\multirow{1}{*}{\textbf{{Plow}}}&\multirow{1}{*}{\textbf{{NCSR}}}&\multirow{1}{*}{\textbf{{GID}}}
&\multirow{1}{*}{\textbf{{PGPD}}}&\multirow{1}{*}{\textbf{{LINC}}}&\multirow{1}{*}{\textbf{{aGMM}}}&\multirow{1}{*}{\textbf{{OGLR}}}
&\multirow{1}{*}{\textbf{{WNNM}}}&\multirow{1}{*}{\textbf{{RRC}}}&\multirow{1}{*}{\textbf{{NNM}}}&\multirow{1}{*}{\textbf{{BM3D}}}
&\multirow{1}{*}{\textbf{{LSSC}}}&\multirow{1}{*}{\textbf{{EPLL}}}&\multirow{1}{*}{\textbf{{Plow}}}&\multirow{1}{*}{\textbf{{NCSR}}}
&\multirow{1}{*}{\textbf{{GID}}}&\multirow{1}{*}{\textbf{{PGPD}}}&\multirow{1}{*}{\textbf{{LINC}}}&\multirow{1}{*}{\textbf{{aGMM}}}
&\multirow{1}{*}{\textbf{{OGLR}}}&\multirow{1}{*}{\textbf{{WNNM}}}&\multirow{1}{*}{\textbf{{RRC}}}\\
\hline
\multirow{1}{*}{Airplane}
&	29.01 	&	30.59 	&	30.68 	&	30.60 	&	29.98 	&	30.50 	&	29.62 	&	30.80 	&	30.57 	&	30.54 	&	30.17 	&	\textcolor{red}{\textbf{30.92}} 	&	\textcolor{blue}{\textbf{30.86}} 	&	27.62 	&	28.49 	&	28.48 	&	28.54 	&	28.03 	&	28.34 	&	27.54 	&	28.63 	&	28.53 	&	28.42 	&	28.21 	&	\textcolor{red}{\textbf{28.76}} 	&	\textcolor{blue}{\textbf{28.63}}
\\
\hline
\multirow{1}{*}{Barbara}
&	30.07 	&	31.24 	&	31.06 	&	29.85 	&	30.75 	&	31.10 	&	29.81 	&	31.12 	&	\textcolor{red}{\textbf{31.66}} 	&	30.51 	&	30.89 	&	31.60 	&	\textcolor{blue}{\textbf{31.63}} 	&	28.08 	&	29.08 	&	28.92 	&	27.58 	&	28.99 	&	28.68 	&	27.35 	&	28.93 	&	29.40 	&	27.88 	&	28.84 	&	\textcolor{blue}{\textbf{29.49}} 	&	\textcolor{red}{\textbf{29.51}}
\\
\hline
\multirow{1}{*}{boats}
&	29.96 	&	31.42 	&	31.46 	&	30.87 	&	30.90 	&	31.26 	&	29.78 	&	31.38 	&	\textcolor{blue}{\textbf{31.51}} 	&	31.20 	&	31.20 	&	\textcolor{red}{\textbf{31.63}} 	&	31.47 	&	28.08 	&	29.33 	&	\textcolor{blue}{\textbf{29.34}} 	&	28.85 	&	29.01 	&	29.04 	&	27.72 	&	29.32 	&	29.30 	&	29.05 	&	29.11 	&	\textcolor{red}{\textbf{29.49}} 	&	29.31
\\
\hline
\multirow{1}{*}{Elaine}
&	31.44 	&	32.51 	&	32.49 	&	32.16 	&	32.29 	&	32.39 	&	31.23 	&	32.43 	&	32.60 	&	\textcolor{red}{\textbf{32.62}} 	&	32.41 	&	32.55 	&	\textcolor{blue}{\textbf{32.61}} 	&	28.88 	&	30.52 	&	30.38 	&	30.15 	&	30.36 	&	30.25 	&	28.98 	&	30.47 	&	30.44 	&	30.59 	&	30.43 	&	\textcolor{blue}{\textbf{30.53}} 	&	\textcolor{red}{\textbf{30.61}}
\\
\hline
\multirow{1}{*}{Fence}
&	28.86 	&	29.93 	&	30.07 	&	29.24 	&	29.13 	&	30.05 	&	28.99 	&	29.99 	&	\textcolor{blue}{\textbf{30.18}} 	&	29.46 	&	29.82 	&	\textcolor{red}{\textbf{30.40}} 	&	30.08 	&	27.43 	&	28.19 	&	28.16 	&	27.22 	&	27.59 	&	28.13 	&	26.90 	&	28.13 	&	\textcolor{blue}{\textbf{28.26}} 	&	27.31 	&	28.12 	&	\textcolor{red}{\textbf{28.55}} 	&	28.25
\\
\hline
\multirow{1}{*}{Flower}
&	28.52 	&	30.01 	&	\textcolor{red}{\textbf{30.35}} 	&	30.01 	&	29.36 	&	30.05 	&	29.12 	&	30.27 	&	30.29 	&	29.88 	&	29.99 	&	\textcolor{blue}{\textbf{30.34}} 	&	30.17 	&	27.28 	&	27.97 	&	28.07 	&	27.95 	&	27.74 	&	27.86 	&	27.01 	&	28.11 	&	\textcolor{red}{\textbf{28.21}} 	&	27.90 	&	27.96 	&	\textcolor{blue}{\textbf{28.19}} 	&	28.12
\\
\hline
\multirow{1}{*}{Foreman}
&	33.34 	&	34.54 	&	34.45 	&	33.67 	&	34.21 	&	34.42 	&	33.08 	&	34.44 	&	34.65 	&	34.20 	&	34.50 	&	\textcolor{red}{\textbf{34.72}} 	&	\textcolor{blue}{\textbf{34.72}} 	&	30.24 	&	32.75 	&	32.85 	&	31.70 	&	32.45 	&	32.61 	&	30.92 	&	32.83 	&	32.99 	&	32.31 	&	32.84 	&	\textcolor{blue}{\textbf{32.95}} 	&	\textcolor{red}{\textbf{33.26}}
\\
\hline
\multirow{1}{*}{House}
&	32.30 	&	33.77 	&	\textcolor{red}{\textbf{34.10}} 	&	32.99 	&	33.40 	&	33.81 	&	32.68 	&	33.85 	&	33.79 	&	33.52 	&	33.77 	&	\textcolor{blue}{\textbf{34.05}} 	&	33.71 	&	29.85 	&	32.09 	&	\textcolor{blue}{\textbf{32.40}} 	&	31.24 	&	31.67 	&	32.01 	&	30.50 	&	32.24 	&	32.29 	&	31.79 	&	32.02 	&	\textcolor{red}{\textbf{32.72}} 	&	32.30
\\
\hline
\multirow{1}{*}{J. Bean}
&	32.61 	&	34.18 	&	34.54 	&	33.79 	&	33.80 	&	34.37 	&	33.38 	&	34.28 	&	34.11 	&	\textcolor{red}{\textbf{34.80}} 	&	34.44 	&	\textcolor{blue}{\textbf{34.75}} 	&	34.66 	&	29.77 	&	31.97 	&	32.38 	&	31.55 	&	31.61 	&	31.99 	&	30.74 	&	31.99 	&	31.97 	&	\textcolor{red}{\textbf{32.50}} 	&	32.15 	&	\textcolor{blue}{\textbf{32.42}} 	&	32.33
\\
\hline
\multirow{1}{*}{Leaves}
&	28.93 	&	30.09 	&	30.45 	&	29.40 	&	29.08 	&	30.34 	&	29.91 	&	30.46 	&	30.23 	&	30.05 	&	29.87 	&	\textcolor{red}{\textbf{31.09}} 	&	\textcolor{blue}{\textbf{30.82}} 	&	27.17 	&	27.81 	&	27.65 	&	27.19 	&	27.00 	&	28.04 	&	27.59 	&	27.99 	&	27.94 	&	27.53 	&	27.77 	&	\textcolor{red}{\textbf{28.69}} 	&	\textcolor{blue}{\textbf{28.35}}
\\
\hline
\multirow{1}{*}{Lena}
&	30.03 	&	31.52 	&	31.64 	&	31.25 	&	30.98 	&	31.48 	&	30.33 	&	31.64 	&	\textcolor{red}{\textbf{31.75}} 	&	31.48 	&	31.27 	&	31.72 	&	\textcolor{blue}{\textbf{31.72}} 	&	28.29 	&	29.46 	&	29.50 	&	29.18 	&	29.16 	&	29.32 	&	28.36 	&	29.60 	&	\textcolor{red}{\textbf{29.74}} 	&	29.38 	&	29.36 	&	\textcolor{blue}{\textbf{29.68}} 	&	29.67
\\
\hline
\multirow{1}{*}{Lin}
&	31.17 	&	32.83 	&	32.62 	&	32.62 	&	32.45 	&	32.66 	&	31.74 	&	32.79 	&	\textcolor{blue}{\textbf{32.99}} 	&	32.77 	&	32.77 	&	\textcolor{red}{\textbf{33.00}} 	&	32.85 	&	29.22 	&	30.95 	&	30.80 	&	30.67 	&	30.76 	&	30.65 	&	29.63 	&	30.96 	&	31.06 	&	30.98 	&	30.85 	&	\textcolor{red}{\textbf{31.27}} 	&	\textcolor{blue}{\textbf{30.96}}
\\
\hline
\multirow{1}{*}{Monarch}
&	29.47 	&	30.35 	&	30.58 	&	30.49 	&	29.50 	&	30.52 	&	29.75 	&	30.68 	&	30.59 	&	30.31 	&	30.13 	&	\textcolor{red}{\textbf{31.18}} 	&	\textcolor{blue}{\textbf{31.00}} 	&	27.63 	&	28.36 	&	28.20 	&	28.36 	&	27.77 	&	28.38 	&	27.68 	&	28.49 	&	28.56 	&	28.27 	&	28.33 	&	\textcolor{red}{\textbf{28.94}} 	&	\textcolor{blue}{\textbf{28.79}}
\\
\hline
\multirow{1}{*}{Parrot}
&	30.95 	&	32.32 	&	32.29 	&	32.00 	&	31.85 	&	32.25 	&	31.24 	&	32.31 	&	\textcolor{blue}{\textbf{32.46}} 	&	32.18 	&	32.12 	&	\textcolor{red}{\textbf{32.66}} 	&	32.41 	&	28.97 	&	30.33 	&	30.30 	&	30.00 	&	29.88 	&	30.20 	&	29.33 	&	30.30 	&	\textcolor{blue}{\textbf{30.57}} 	&	30.26 	&	30.24 	&	\textcolor{red}{\textbf{30.65}} 	&	30.50
\\
\hline
\multirow{1}{*}{Plants}
&	31.38 	&	32.68 	&	32.72 	&	32.45 	&	32.28 	&	32.41 	&	31.27 	&	32.76 	&	\textcolor{blue}{\textbf{32.95}} 	&	32.57 	&	32.68 	&	\textcolor{red}{\textbf{33.04}} 	&	32.82 	&	29.09 	&	30.70 	&	30.56 	&	30.43 	&	30.41 	&	30.19 	&	28.99 	&	30.73 	&	30.86 	&	30.50 	&	30.64 	&	\textcolor{blue}{\textbf{30.82}} 	&	\textcolor{red}{\textbf{30.90}}
\\
\hline
\multirow{1}{*}{Starfish}
&	28.63 	&	29.67 	&	29.96 	&	29.58 	&	28.83 	&	29.85 	&	29.05 	&	29.84 	&	29.61 	&	29.74 	&	29.46 	&	\textcolor{red}{\textbf{30.30}} 	&	\textcolor{blue}{\textbf{30.02}} 	&	27.10 	&	27.65 	&	27.70 	&	27.52 	&	27.02 	&	27.69 	&	26.90 	&	27.67 	&	27.61 	&	27.61 	&	27.47 	&	\textcolor{red}{\textbf{28.08}} 	&	\textcolor{blue}{\textbf{27.95}}
\\
\hline
\multirow{1}{*}{\textbf{Average}}
&	30.42 	&	31.73 	&	31.84 	&	31.31 	&	31.17 	&	31.72 	&	30.69 	&	31.82 	&	31.87 	&	31.61 	&	31.59 	&	\textcolor{red}{\textbf{32.12}} 	&	\textcolor{blue}{\textbf{31.97}} 	&	28.42 	&	29.73 	&	29.73 	&	29.26 	&	29.34 	&	29.59 	&	28.51 	&	29.77 	&	29.86 	&	29.52 	&	29.65 	&	\textcolor{red}{\textbf{30.08}} 	&	\textcolor{blue}{\textbf{29.97}}
\\
\hline

\multicolumn{1}{|c|}{}&\multicolumn{13}{|c||}{${\bf \sigma}_n = 40$}&\multicolumn{13}{|c|}{${\bf \sigma}_n = 50$}\\
\hline						\multirow{1}{*}{\textbf{{Images}}}&\multirow{1}{*}{\textbf{{NNM}}}&\multirow{1}{*}{\textbf{{BM3D}}}&\multirow{1}{*}{\textbf{{LSSC}}}
&\multirow{1}{*}{\textbf{{EPLL}}}&\multirow{1}{*}{\textbf{{Plow}}}&\multirow{1}{*}{\textbf{{NCSR}}}&\multirow{1}{*}{\textbf{{GID}}}
&\multirow{1}{*}{\textbf{{PGPD}}}&\multirow{1}{*}{\textbf{{LINC}}}&\multirow{1}{*}{\textbf{{aGMM}}}&\multirow{1}{*}{\textbf{{OGLR}}}
&\multirow{1}{*}{\textbf{{WNNM}}}&\multirow{1}{*}{\textbf{{RRC}}}&\multirow{1}{*}{\textbf{{NNM}}}&\multirow{1}{*}{\textbf{{BM3D}}}
&\multirow{1}{*}{\textbf{{LSSC}}}&\multirow{1}{*}{\textbf{{EPLL}}}&\multirow{1}{*}{\textbf{{Plow}}}&\multirow{1}{*}{\textbf{{NCSR}}}
&\multirow{1}{*}{\textbf{{GID}}}&\multirow{1}{*}{\textbf{{PGPD}}}&\multirow{1}{*}{\textbf{{LINC}}}&\multirow{1}{*}{\textbf{{aGMM}}}
&\multirow{1}{*}{\textbf{{OGLR}}}&\multirow{1}{*}{\textbf{{WNNM}}}&\multirow{1}{*}{\textbf{{RRC}}}\\
\hline
\multirow{1}{*}{Airplane}
&	26.49 	&	26.88 	&	26.97 	&	27.08 	&	26.70 	&	26.78 	&	26.17 	&	27.12 	&	27.06 	&	26.95 	&	26.82 	&	\textcolor{red}{\textbf{27.38}} 	&	\textcolor{blue}{\textbf{27.21}} 	&	25.16 	&	25.76 	&	25.68 	&	25.96 	&	25.64 	&	25.63 	&	25.10 	&	25.98 	&	25.89 	&	25.83 	&	25.67 	&	\textcolor{blue}{\textbf{26.09}} 	&	\textcolor{red}{\textbf{26.13}}
\\
\hline
\multirow{1}{*}{Barbara}
&	26.97 	&	27.26 	&	27.35 	&	25.99 	&	27.59 	&	27.25 	&	25.77 	&	27.43 	&	27.71 	&	26.34 	&	27.42 	&	\textcolor{blue}{\textbf{27.86}} 	&	\textcolor{red}{\textbf{27.98}} 	&	25.66 	&	26.42 	&	26.26 	&	24.86 	&	26.42 	&	26.13 	&	24.52 	&	26.27 	&	26.31 	&	25.37 	&	26.17 	&	\textcolor{blue}{\textbf{26.78}} 	&	\textcolor{red}{\textbf{26.78}}
\\
\hline
\multirow{1}{*}{boats}
&	27.09 	&	27.76 	&	27.84 	&	27.42 	&	27.55 	&	27.52 	&	26.29 	&	\textcolor{blue}{\textbf{27.90}} 	&	27.80 	&	27.60 	&	27.69 	&	\textcolor{red}{\textbf{28.02}} 	&	27.89 	&	25.81 	&	26.74 	&	26.64 	&	26.31 	&	26.38 	&	26.37 	&	25.09 	&	\textcolor{blue}{\textbf{26.82}} 	&	26.63 	&	26.50 	&	26.41 	&	\textcolor{red}{\textbf{26.85}} 	&	26.81
\\
\hline
\multirow{1}{*}{Elaine}
&	28.17 	&	28.95 	&	28.79 	&	28.73 	&	28.84 	&	28.91 	&	27.67 	&	29.08 	&	28.87 	&	\textcolor{red}{\textbf{29.10}} 	&	29.00 	&	29.05 	&	\textcolor{blue}{\textbf{29.09}} 	&	26.86 	&	27.96 	&	27.51 	&	27.63 	&	27.62 	&	27.68 	&	26.54 	&	27.90 	&	27.60 	&	27.93 	&	27.55 	&	\textcolor{red}{\textbf{28.14}} 	&	\textcolor{blue}{\textbf{27.98}}
\\
\hline
\multirow{1}{*}{Fence}
&	26.47 	&	26.84 	&	26.89 	&	25.74 	&	26.42 	&	26.76 	&	25.61 	&	26.91 	&	26.93 	&	25.80 	&	26.72 	&	\textcolor{red}{\textbf{27.27}} 	&	\textcolor{blue}{\textbf{26.97}} 	&	25.22 	&	25.92 	&	25.87 	&	24.57 	&	25.49 	&	25.77 	&	24.41 	&	25.94 	&	25.86 	&	24.57 	&	25.52 	&	\textcolor{red}{\textbf{26.35}} 	&	\textcolor{blue}{\textbf{25.97}}
\\
\hline
\multirow{1}{*}{Flower}
&	26.12 	&	26.48 	&	26.63 	&	26.55 	&	26.44 	&	26.35 	&	25.60 	&	26.68 	&	26.70 	&	26.55 	&	26.58 	&	\textcolor{red}{\textbf{26.85}} 	&	\textcolor{blue}{\textbf{26.76}} 	&	24.90 	&	25.49 	&	25.49 	&	25.51 	&	25.40 	&	25.31 	&	24.42 	&	25.63 	&	25.52 	&	25.51 	&	25.42 	&	\textcolor{blue}{\textbf{25.67}} 	&	\textcolor{red}{\textbf{25.72}}
\\
\hline
\multirow{1}{*}{Foreman}
&	30.04 	&	31.29 	&	31.43 	&	30.28 	&	30.90 	&	31.52 	&	29.61 	&	31.55 	&	31.57 	&	30.95 	&	31.64 	&	\textcolor{blue}{\textbf{31.54}} 	&	\textcolor{red}{\textbf{32.02}} 	&	28.69 	&	30.36 	&	30.29 	&	29.20 	&	29.60 	&	30.41 	&	28.64 	&	30.45 	&	30.30 	&	29.80 	&	30.00 	&	\textcolor{blue}{\textbf{30.51}} 	&	\textcolor{red}{\textbf{30.87}}
\\
\hline
\multirow{1}{*}{House}
&	29.49 	&	30.65 	&	\textcolor{blue}{\textbf{31.10}} 	&	29.89 	&	30.25 	&	30.79 	&	29.02 	&	31.02 	&	31.04 	&	30.40 	&	30.68 	&	\textcolor{red}{\textbf{31.34}} 	&	31.09 	&	28.00 	&	29.69 	&	\textcolor{blue}{\textbf{29.99}} 	&	28.79 	&	28.99 	&	29.61 	&	27.76 	&	29.93 	&	29.96 	&	29.28 	&	29.17 	&	\textcolor{red}{\textbf{30.35}} 	&	29.92
\\
\hline
\multirow{1}{*}{J. Bean}
&	29.23 	&	30.21 	&	\textcolor{blue}{\textbf{30.75}} 	&	29.96 	&	29.97 	&	30.49 	&	29.53 	&	30.39 	&	30.39 	&	\textcolor{red}{\textbf{30.79}} 	&	30.45 	&	30.71 	&	30.68 	&	27.77 	&	29.26 	&	\textcolor{blue}{\textbf{29.42}} 	&	28.73 	&	28.66 	&	29.24 	&	28.43 	&	29.20 	&	29.07 	&	\textcolor{red}{\textbf{29.46}} 	&	28.94 	&	29.34 	&	29.38
\\
\hline
\multirow{1}{*}{Leaves}
&	25.91 	&	25.69 	&	26.04 	&	25.62 	&	25.45 	&	26.20 	&	25.82 	&	26.29 	&	26.35 	&	25.76 	&	26.06 	&	\textcolor{red}{\textbf{26.99}} 	&	\textcolor{blue}{\textbf{26.65}} 	&	24.22 	&	24.68 	&	24.78 	&	24.39 	&	24.28 	&	24.94 	&	24.41 	&	25.03 	&	25.06 	&	24.42 	&	24.63 	&	\textcolor{red}{\textbf{25.54}} 	&	\textcolor{blue}{\textbf{25.30}}
\\
\hline
\multirow{1}{*}{Lena}
&	27.36 	&	27.82 	&	28.07 	&	27.78 	&	27.78 	&	28.00 	&	26.98 	&	28.22 	&	28.21 	&	27.91 	&	28.04 	&	\textcolor{red}{\textbf{28.43}} 	&	\textcolor{blue}{\textbf{28.22}} 	&	26.15 	&	26.90 	&	27.03 	&	26.68 	&	26.70 	&	26.94 	&	25.82 	&	27.15 	&	27.02 	&	26.85 	&	26.78 	&	\textcolor{red}{\textbf{27.22}} 	&	\textcolor{blue}{\textbf{27.17}}
\\
\hline
\multirow{1}{*}{Lin}
&	28.63 	&	29.52 	&	29.44 	&	29.32 	&	29.40 	&	29.27 	&	28.44 	&	\textcolor{blue}{\textbf{29.73}} 	&	29.74 	&	29.69 	&	29.55 	&	\textcolor{red}{\textbf{29.78}} 	&	29.56 	&	27.33 	&	28.71 	&	28.45 	&	28.26 	&	28.31 	&	28.23 	&	27.50 	&	\textcolor{blue}{\textbf{28.79}} 	&	28.70 	&	28.66 	&	28.35 	&	\textcolor{red}{\textbf{28.86}} 	&	28.51
\\
\hline
\multirow{1}{*}{Monarch}
&	26.68 	&	26.72 	&	26.87 	&	26.89 	&	26.43 	&	26.81 	&	26.32 	&	27.02 	&	27.13 	&	26.87 	&	27.00 	&	\textcolor{red}{\textbf{27.50}} 	&	\textcolor{blue}{\textbf{27.34}} 	&	25.30 	&	25.82 	&	25.88 	&	25.78 	&	25.41 	&	25.73 	&	25.28 	&	26.00 	&	25.98 	&	25.82 	&	25.78 	&	\textcolor{red}{\textbf{26.29}} 	&	\textcolor{blue}{\textbf{26.22}}
\\
\hline

\multirow{1}{*}{Parrot}
&	28.30 	&	28.64 	&	28.92 	&	28.60 	&	28.38 	&	28.77 	&	28.01 	&	28.95 	&	\textcolor{blue}{\textbf{29.23}} 	&	28.88 	&	28.93 	&	\textcolor{red}{\textbf{29.33}} 	&	29.16 	&	26.77 	&	27.88 	&	27.77 	&	27.53 	&	27.26 	&	27.67 	&	26.79 	&	27.91 	&	\textcolor{blue}{\textbf{28.16}} 	&	27.80 	&	27.67 	&	\textcolor{red}{\textbf{28.29}} 	&	28.03
\\
\hline
\multirow{1}{*}{Plants}
&	28.39 	&	29.14 	&	29.03 	&	28.96 	&	28.90 	&	28.73 	&	27.75 	&	29.36 	&	29.36 	&	29.12 	&	29.27 	&	\textcolor{blue}{\textbf{29.28}} 	&	\textcolor{red}{\textbf{29.51}} 	&	27.05 	&	28.11 	&	27.80 	&	27.83 	&	27.75 	&	27.65 	&	26.69 	&	28.25 	&	28.08 	&	28.00 	&	27.94 	&	\textcolor{blue}{\textbf{28.18}} 	&	\textcolor{red}{\textbf{28.32}}
\\
\hline
\multirow{1}{*}{Starfish}
&	25.87 	&	26.06 	&	26.22 	&	26.12 	&	25.70 	&	26.17 	&	25.39 	&	26.21 	&	26.14 	&	26.16 	&	26.00 	&	\textcolor{blue}{\textbf{26.41}} 	&	\textcolor{red}{\textbf{26.43}} 	&	24.58 	&	25.04 	&	25.12 	&	25.05 	&	24.71 	&	25.06 	&	24.07 	&	25.11 	&	24.96 	&	25.09 	&	24.84 	&	\textcolor{blue}{\textbf{25.32}} 	&	\textcolor{red}{\textbf{25.34}}
\\
\hline
\multirow{1}{*}{\textbf{Average}}
&	27.58 	&	28.12 	&	28.27 	&	27.81 	&	27.92 	&	28.15 	&	27.12 	&	28.37 	&	28.39 	&	28.06 	&	28.24 	&	\textcolor{red}{\textbf{28.61}} 	&	\textcolor{blue}{\textbf{28.54}} 	&	26.22 	&	27.17 	&	27.12 	&	26.69 	&	26.79 	&	27.02 	&	25.97 	&	27.27 	&	27.19 	&	26.93 	&	26.93 	&	\textcolor{red}{\textbf{27.49}} 	&	\textcolor{blue}{\textbf{27.40}}
\\
\hline

\multicolumn{1}{|c|}{}&\multicolumn{13}{|c||}{${\bf \sigma}_n = 75$}&\multicolumn{13}{|c|}{${\bf \sigma}_n = 100$}\\
\hline						\multirow{1}{*}{\textbf{{Images}}}&\multirow{1}{*}{\textbf{{NNM}}}&\multirow{1}{*}{\textbf{{BM3D}}}&\multirow{1}{*}{\textbf{{LSSC}}}
&\multirow{1}{*}{\textbf{{EPLL}}}&\multirow{1}{*}{\textbf{{Plow}}}&\multirow{1}{*}{\textbf{{NCSR}}}&\multirow{1}{*}{\textbf{{GID}}}
&\multirow{1}{*}{\textbf{{PGPD}}}&\multirow{1}{*}{\textbf{{LINC}}}&\multirow{1}{*}{\textbf{{aGMM}}}&\multirow{1}{*}{\textbf{{OGLR}}}
&\multirow{1}{*}{\textbf{{WNNM}}}&\multirow{1}{*}{\textbf{{RRC}}}&\multirow{1}{*}{\textbf{{NNM}}}&\multirow{1}{*}{\textbf{{BM3D}}}
&\multirow{1}{*}{\textbf{{LSSC}}}&\multirow{1}{*}{\textbf{{EPLL}}}&\multirow{1}{*}{\textbf{{Plow}}}&\multirow{1}{*}{\textbf{{NCSR}}}
&\multirow{1}{*}{\textbf{{GID}}}&\multirow{1}{*}{\textbf{{PGPD}}}&\multirow{1}{*}{\textbf{{LINC}}}&\multirow{1}{*}{\textbf{{aGMM}}}
&\multirow{1}{*}{\textbf{{OGLR}}}&\multirow{1}{*}{\textbf{{WNNM}}}&\multirow{1}{*}{\textbf{{RRC}}}\\
\hline
\multirow{1}{*}{Airplane}
&	23.15 	&	23.99 	&	23.77 	&	24.03 	&	23.67 	&	23.76 	&	22.91 	&	\textcolor{blue}{\textbf{24.15}} 	&	23.78 	&	23.95 	&	23.79 	&	\textcolor{red}{\textbf{24.20}} 	&	24.10 	&	19.09 	&	22.89 	&	22.56 	&	22.78 	&	22.30 	&	22.60 	&	21.82 	&	\textcolor{red}{\textbf{23.02}} 	&	22.43 	&	22.67 	&	22.31 	&	22.93 	&	\textcolor{blue}{\textbf{22.93}}
\\
\hline
\multirow{1}{*}{Barbara}
&	23.58 	&	24.53 	&	24.31 	&	23.00 	&	24.30 	&	24.06 	&	22.43 	&	24.39 	&	23.97 	&	23.09 	&	24.52 	&	\textcolor{red}{\textbf{24.79}} 	&	\textcolor{blue}{\textbf{24.62}} 	&	22.01 	&	23.20 	&	22.95 	&	21.89 	&	22.86 	&	22.70 	&	21.40 	&	23.11 	&	22.57 	&	21.92 	&	22.73 	&	\textcolor{blue}{\textbf{23.26}} 	&	\textcolor{red}{\textbf{23.37}}
\\
\hline
\multirow{1}{*}{boats}
&	23.64 	&	24.82 	&	24.62 	&	24.33 	&	24.23 	&	24.44 	&	23.18 	&	\textcolor{blue}{\textbf{24.83}} 	&	24.45 	&	24.51 	&	24.40 	&	\textcolor{red}{\textbf{25.03}} 	&	24.79 	&	22.07 	&	\textcolor{blue}{\textbf{23.47}} 	&	23.22 	&	23.01 	&	22.69 	&	22.98 	&	21.91 	&	23.47 	&	22.97 	&	23.14 	&	22.74 	&	\textcolor{red}{\textbf{23.76}} 	&	23.38
\\
\hline
\multirow{1}{*}{Elaine}
&	24.52 	&	\textcolor{blue}{\textbf{25.93}} 	&	25.27 	&	25.60 	&	25.30 	&	25.34 	&	24.54 	&	25.78 	&	25.28 	&	25.77 	&	25.48 	&	\textcolor{red}{\textbf{25.94}} 	&	25.87 	&	22.87 	&	\textcolor{blue}{\textbf{24.48}} 	&	23.67 	&	24.16 	&	23.67 	&	23.77 	&	23.21 	&	24.34 	&	23.75 	&	24.30 	&	23.57 	&	\textcolor{red}{\textbf{24.54}} 	&	24.36
\\
\hline

\multirow{1}{*}{Fence}
&	23.22 	&	24.22 	&	24.04 	&	22.46 	&	23.57 	&	23.75 	&	22.17 	&	24.18 	&	23.80 	&	22.70 	&	23.94 	&	\textcolor{red}{\textbf{24.53}} 	&	\textcolor{blue}{\textbf{24.32}} 	&	21.62 	&	22.92 	&	22.71 	&	21.10 	&	22.17 	&	22.23 	&	19.99 	&	22.87 	&	22.49 	&	21.50 	&	22.36 	&	\textcolor{red}{\textbf{23.69}} 	&	\textcolor{blue}{\textbf{23.08}}
\\
\hline

\multirow{1}{*}{Flower}
&	22.93 	&	23.82 	&	23.53 	&	23.59 	&	23.47 	&	23.50 	&	22.72 	&	\textcolor{blue}{\textbf{23.82}} 	&	23.47 	&	23.53 	&	23.66 	&	\textcolor{red}{\textbf{23.88}} 	&	23.77 	&	21.38 	&	\textcolor{blue}{\textbf{22.66}} 	&	22.29 	&	22.39 	&	22.17 	&	22.22 	&	20.69 	&	22.65 	&	22.19 	&	22.31 	&	22.11 	&	\textcolor{red}{\textbf{22.70}} 	&	22.46
\\
\hline

\multirow{1}{*}{Foreman}
&	26.18 	&	28.07 	&	28.20 	&	27.24 	&	27.15 	&	28.18 	&	26.71 	&	28.39 	&	28.08 	&	27.67 	&	27.96 	&	\textcolor{blue}{\textbf{28.48}} 	&	\textcolor{red}{\textbf{28.83}} 	&	24.79 	&	26.51 	&	26.55 	&	25.91 	&	25.55 	&	26.55 	&	25.33 	&	26.81 	&	26.67 	&	26.20 	&	26.11 	&	\textcolor{red}{\textbf{27.39}} 	&	\textcolor{blue}{\textbf{27.27}}
\\
\hline

\multirow{1}{*}{House}
&	25.56 	&	27.51 	&	27.75 	&	26.70 	&	26.52 	&	27.16 	&	25.23 	&	27.81 	&	27.52 	&	27.11 	&	27.10 	&	\textcolor{red}{\textbf{28.47}} 	&	\textcolor{blue}{\textbf{27.98}} 	&	23.66 	&	25.87 	&	25.71 	&	25.21 	&	24.72 	&	25.49 	&	22.38 	&	26.17 	&	25.69 	&	25.55 	&	25.07 	&	\textcolor{red}{\textbf{26.79}} 	&	\textcolor{blue}{\textbf{26.38}}
\\
\hline

\multirow{1}{*}{J. Bean}
&	25.23 	&	\textcolor{red}{\textbf{27.22}} 	&	\textcolor{blue}{\textbf{27.20}} 	&	26.57 	&	26.23 	&	27.15 	&	26.26 	&	27.07 	&	26.64 	&	27.09 	&	26.48 	&	27.20 	&	27.17 	&	23.73 	&	\textcolor{red}{\textbf{25.80}} 	&	25.64 	&	25.16 	&	24.55 	&	25.61 	&	24.37 	&	25.66 	&	24.96 	&	25.58 	&	24.57 	&	25.64 	&	\textcolor{blue}{\textbf{25.71}}
\\
\hline

\multirow{1}{*}{Leaves}
&	21.79 	&	22.49 	&	22.16 	&	22.03 	&	22.02 	&	22.60 	&	21.28 	&	22.61 	&	22.49 	&	21.96 	&	22.20 	&	\textcolor{red}{\textbf{23.10}} 	&	\textcolor{blue}{\textbf{22.92}} 	&	19.57 	&	20.90 	&	20.54 	&	20.26 	&	20.43 	&	20.84 	&	19.13 	&	20.95 	&	20.44 	&	20.29 	&	20.28 	&	\textcolor{red}{\textbf{21.55}} 	&	\textcolor{blue}{\textbf{21.22}}
\\
\hline

\multirow{1}{*}{Lena}
&	24.08 	&	25.17 	&	25.04 	&	24.75 	&	24.64 	&	25.02 	&	23.78 	&	25.30 	&	25.01 	&	25.02 	&	24.90 	&	\textcolor{red}{\textbf{25.38}} 	&	\textcolor{blue}{\textbf{25.33}} 	&	22.30 	&	23.87 	&	23.70 	&	23.46 	&	23.19 	&	23.63 	&	22.43 	&	24.02 	&	23.71 	&	23.73 	&	23.18 	&	\textcolor{blue}{\textbf{24.08}} 	&	\textcolor{red}{\textbf{24.14}}
\\
\hline

\multirow{1}{*}{Lin}
&	25.07 	&	\textcolor{blue}{\textbf{26.96}} 	&	26.53 	&	26.36 	&	26.08 	&	26.22 	&	25.50 	&	\textcolor{red}{\textbf{27.05}} 	&	26.91 	&	26.76 	&	26.36 	&	26.94 	&	26.86 	&	23.36 	&	25.60 	&	25.04 	&	25.05 	&	24.47 	&	24.85 	&	24.14 	&	\textcolor{blue}{\textbf{25.66}} 	&	25.50 	&	25.30 	&	24.63 	&	\textcolor{red}{\textbf{25.67}} 	&	25.50
\\
\hline

\multirow{1}{*}{Monarch}
&	23.06 	&	23.91 	&	23.66 	&	23.73 	&	23.34 	&	23.67 	&	22.77 	&	24.00 	&	23.81 	&	23.85 	&	23.73 	&	\textcolor{blue}{\textbf{24.15}} 	&	\textcolor{red}{\textbf{24.24}} 	&	21.03 	&	22.52 	&	22.24 	&	22.24 	&	21.83 	&	22.10 	&	20.73 	&	22.56 	&	22.10 	&	22.42 	&	21.87 	&	\textcolor{red}{\textbf{22.87}} 	&	\textcolor{blue}{\textbf{22.76}}
\\
\hline

\multirow{1}{*}{Parrot}
&	24.54 	&	25.94 	&	25.63 	&	25.56 	&	25.15 	&	25.45 	&	24.87 	&	25.98 	&	26.01 	&	25.72 	&	25.74 	&	\textcolor{red}{\textbf{26.32}} 	&	\textcolor{blue}{\textbf{26.22}} 	&	22.84 	&	24.60 	&	24.19 	&	24.08 	&	23.65 	&	23.94 	&	23.54 	&	24.52 	&	24.30 	&	24.26 	&	24.03 	&	\textcolor{red}{\textbf{24.85}} 	&	\textcolor{blue}{\textbf{24.83}}
\\
\hline

\multirow{1}{*}{Plants}
&	24.80 	&	26.25 	&	25.84 	&	25.90 	&	25.57 	&	25.75 	&	24.96 	&	26.33 	&	25.94 	&	26.05 	&	25.89 	&	\textcolor{blue}{\textbf{26.26}} 	&	\textcolor{red}{\textbf{26.40}} 	&	22.27 	&	\textcolor{blue}{\textbf{24.98}} 	&	24.41 	&	24.65 	&	24.14 	&	24.46 	&	23.86 	&	\textcolor{red}{\textbf{25.06}} 	&	24.54 	&	24.75 	&	24.30 	&	24.87 	&	24.91
\\
\hline

\multirow{1}{*}{Starfish}
&	22.52 	&	23.27 	&	23.12 	&	23.17 	&	22.82 	&	23.18 	&	21.99 	&	23.23 	&	22.83 	&	23.22 	&	23.00 	&	\textcolor{blue}{\textbf{23.25}} 	&	\textcolor{red}{\textbf{23.32}} 	&	20.97 	&	\textcolor{red}{\textbf{22.10}} 	&	21.77 	&	21.92 	&	21.48 	&	21.91 	&	20.84 	&	\textcolor{blue}{\textbf{22.08}} 	&	21.31 	&	21.95 	&	21.52 	&	22.03 	&	21.98
\\
\hline

\multirow{1}{*}{\textbf{Average}}
&	23.99 	&	25.26 	&	25.04 	&	24.69 	&	24.63 	&	24.95 	&	23.83 	&	25.31 	&	25.00 	&	24.87 	&	24.95 	&	\textcolor{red}{\textbf{25.49}} 	&	\textcolor{blue}{\textbf{25.42}} 	&	22.10 	&	23.90 	&	23.57 	&	23.33 	&	23.12 	&	23.49 	&	22.24 	&	23.94 	&	23.48 	&	23.49 	&	23.21 	&	\textcolor{red}{\textbf{24.16}} 	&	\textcolor{blue}{\textbf{24.02}}
\\
\hline

\end{tabular}
}
\label{Tab:1}
\vspace{-2mm}
\end{table*}

\section {\bihan{Connection to Group Sparse Representation}}
 \label{sec:5}
\bihan{In this section, we provide  an analytical investigation  on the connection between the proposed RRC model and the popular group-based sparse representation (GSR) model \cite{28,29,72,27}.
More specifically, we show the equivalence of the proposed RRC and the GSR using a specific method to construct the group-wise dictionaries, \ie, the group sparsity residual constraint (GSRC) model \cite{49,50,51,52,53}.}

\subsection{Group-based Sparse Representation}
 \label{sec:5.1}
\bihan{Different from the patch-based sparse representation, \eg, K-SVD \cite{54}, the GSR \cite{27,28,29,72} models $n$ groups of similar patches extracted from the image $\textbf{\emph{x}}$, and represent each group using a data matrix ${\textbf{\emph{X}}}^\star_i\in{\mathbb R}^{d\times m}$. GSR models each data matrix ${\textbf{\emph{X}}}^\star_i$ using a group-wise sparse representation as
}
%We first give a brief introduction to the GSR  model \cite{27,28,29,72}. We extract $n$ group matrices from a clean image $\textbf{\emph{x}}$.
%Similar to patch-based sparse representation, \eg, K-SVD \cite{54}, given a dictionary ${\textbf{\emph{D}}}_{i}$, each group ${\textbf{\emph{X}}}^\star_i\in{\mathbb R}^{d\times m}$ can be sparsely represented by solving
\begin{equation}
{\hat{\textbf{\emph{B}}}}_i^\star=\arg\min_{{\textbf{\emph{B}}}_i^\star} \left(\frac{1}{2}\left\|{\textbf{\emph{X}}}_i^\star-{\textbf{\emph{D}}}_i{\textbf{\emph{B}}}_i^\star\right\|_F^2+\lambda\left\|{\textbf{\emph{B}}}_i^\star\right\|_1\right) \;\; \forall i,
\label{eq:37}
\end{equation} %equal 37
\bihan{where each ${\textbf{\emph{B}}}_i^\star$ is the group sparse coefficient for ${\textbf{\emph{X}}}_i^\star$ and ${\textbf{\emph{D}}}_{i}$ represents the dictionary, which is usually learned from each group \cite{28,49,53}. The $\ell_1$-norm is initially imposed on each column of ${\textbf{\emph{B}}}_i^\star$, and here extended to be the $\ell_1$-norm on matrix.}

\bihan{For image restoration tasks,}
%In image restoration, the goal is to use
the GSR model can be applied to recover the group data matrices $\left\{ \textbf{\emph{X}}_i^\star \right \}$ from their  {\em degraded} \bihan{measurements $\left\{ \textbf{\emph{Y}}_i \right \}$ by solving the following problem,}
\begin{equation}
{\hat{\textbf{\emph{A}}}_i}=\arg\min_{{\textbf{\emph{A}}}_i} \left(\frac{1}{2}\left\|{\textbf{\emph{Y}}}_i-{\textbf{\emph{D}}}_i{\textbf{\emph{A}}}_i\right\|_F^2+\lambda\left\|{\textbf{\emph{A}}}_i\right\|_1\right)
\;\; \forall i.
\label{eq:38}
\end{equation} %equal 38
where ${\textbf{\emph{A}}}_i$ \bihan{denotes} the group \bihan{sparse coefficient} of each group ${\textbf{\emph{Y}}}_i$.
Once all group  sparse codes $\{{\textbf{\emph{A}}}_i\}_{i=1}^n$ are obtained, the underlying image $\hat{\textbf{\emph{x}}}$ can be reconstructed as   $\hat{\textbf{\emph{x}}}= \textbf{\emph{D}}\circ \textbf{\emph{A}} $, where $\textbf{\emph{D}}$ is the \bihan{global} dictionary to sparsely represent all the groups of the degraded image $\textbf{\emph{y}}$, $\textbf{\emph{A}}$ denotes a set of $\{{\textbf{\emph{A}}}_i\}_{i=1}^n$ and $\circ$ is an operator \cite{28}.

\bihan{In practice, one would like to approximate the oracle group sparse coefficient $\textbf{\emph{B}}_i^\star$ using the group sparse coefficient $\textbf{\emph{A}}_i$ based on the corrupted measurement ${\textbf{\emph{Y}}}_i$.
The quality of the approximate heavily depends on the the difference between $\textbf{\emph{B}}_i^\star$ and $\textbf{\emph{A}}_i$, \ie, the {\em group sparsity residual}, which we define as the following,
}
%However, under the degraded environment, it is challenging to estimate the true group sparse coefficients ${{\textbf{\emph{B}}}}_i^\star$ from ${\textbf{\emph{Y}}}_i$ directly. In other words, the group sparse coefficient ${{\textbf{\emph{A}}}}_i$ obtained from Eq.~\eqref{eq:38} is expected to be close to the true group sparse coefficient ${{\textbf{\emph{B}}}}_i^\star$ in Eq.~\eqref{eq:37}. Therefore, the quality of image restoration largely depends on the {\em group sparsity residual},  which is defined by the difference between ${{\textbf{\emph{A}}}}_i$ and ${{\textbf{\emph{B}}}}_i^\star$,
\begin{equation}
{{\textbf{\emph{R}}}}_i^\star={{\textbf{\emph{A}}}}_i-{{\textbf{\emph{B}}}}^\star_i,
\label{eq:39}
\end{equation} %equal 39

%\bihan{Besides, the approximated sparse code $\textbf{\emph{A}}_i$ based on the corrupted $\textbf{\emph{Y}}_i$, one can also apply another external image estimate ${\textbf{\emph{X}}}'_i$ using some reference restoration algorithm. Similarly, we can apply GSR by solving the following sparse coding problem}

Similar to the proposed RRC model, in real applications, ${\textbf{\emph{X}}}^\star_i$ is inaccessible and we thus employ an estimate of it, denoted by ${\textbf{\emph{X}}}'_i$. Given ${\textbf{\emph{X}}}'_i$ and the dictionary ${\textbf{\emph{D}}}_{i}$, the group sparse coefficient ${\textbf{\emph{B}}}_i$ for each group ${\textbf{\emph{X}}}'_i$ is solved by
\begin{equation}
\hat{\textbf{\emph{B}}}_i=\arg\min_{{\textbf{\emph{B}}}_i} \left(\frac{1}{2}\left\|{\textbf{\emph{X}}}_i'-{\textbf{\emph{D}}}_i{\textbf{\emph{B}}}_i\right\|_F^2+\lambda\left\|{\textbf{\emph{B}}}_i\right\|_1\right),
\label{eq:40}
\end{equation} % %equal 40

\bihan{As both $\textbf{\emph{A}}_i$ and $\textbf{\emph{B}}_i$ are to approximate the underlying oracle $\textbf{\emph{B}}_i^\star$, the difference between them needs to be minimized to achieve accurate approximation \cite{49}.
Therefore, the enhanced sparse representation problem based on the {\em group sparse residual constraint} (GSRC) model \cite{49,50,51,52,53} is the following,
}
%Following this, in order to reduce the group sparsity residual ${{\textbf{\emph{R}}}_i} = {{\textbf{\emph{A}}}}_i-{{\textbf{\emph{B}}}}_i$ and enhance the accuracy of  ${{\textbf{\emph{A}}}}_i$, we define the {\em group sparse residual constraint} (GSRC) model \cite{49,50,51,52,53} as below,
\begin{equation}
\hat{\textbf{\emph{A}}}_i=\arg\min_{{\textbf{\emph{A}}}_i} \left(\frac{1}{2}\left\|{\textbf{\emph{Y}}}_i-{\textbf{\emph{D}}}_i{{\textbf{\emph{A}}}_i}\right\|_F^2+\lambda\left\|{{\textbf{\emph{A}}}_i}-{{\textbf{\emph{B}}}_i}\right\|_1\right).
\label{eq:41}
\end{equation} %equal 41
%We will prove that this GSRC model is equivalent to the proposed RRC model under the following adaptive dictionary.

\begin{table*}[!t]
\vspace{-4mm}
\caption{{SSIM comparison of NNM, BM3D \cite{26}, LSSC \cite{27}, EPLL  \cite{56}, Plow \cite{57}, NCSR \cite{49}, GID \cite{58}, PGPD \cite{29}, LINC \cite{59}, aGMM \cite{60}, OGLR \cite{61}, WNNM \cite{5} and RRC for image denoising.}}
%\vspace{-2mm}
\resizebox{1\textwidth}{!}				
{
%\tiny
%\Huge
\large

\centering  % ±í¾ÓÖÐ
\begin{tabular}{|c|c|c|c|c|c|c|c|c|c|c|c|c|c||c|c|c|c|c|c|c|c|c|c|c|c|c|c|c|c|c|c|c|}
\hline
\multicolumn{1}{|c|}{}&\multicolumn{13}{|c||}{${\bf \sigma}_n = 20$}&\multicolumn{13}{|c|}{${\bf \sigma}_n = 30$}\\
\hline						\multirow{1}{*}{\textbf{{Images}}}&\multirow{1}{*}{\textbf{{NNM}}}&\multirow{1}{*}{\textbf{{BM3D}}}&\multirow{1}{*}{\textbf{{LSSC}}}
&\multirow{1}{*}{\textbf{{EPLL}}}&\multirow{1}{*}{\textbf{{Plow}}}&\multirow{1}{*}{\textbf{{NCSR}}}&\multirow{1}{*}{\textbf{{GID}}}
&\multirow{1}{*}{\textbf{{PGPD}}}&\multirow{1}{*}{\textbf{{LINC}}}&\multirow{1}{*}{\textbf{{aGMM}}}&\multirow{1}{*}{\textbf{{OGLR}}}
&\multirow{1}{*}{\textbf{{WNNM}}}&\multirow{1}{*}{\textbf{{RRC}}}&\multirow{1}{*}{\textbf{{NNM}}}&\multirow{1}{*}{\textbf{{BM3D}}}
&\multirow{1}{*}{\textbf{{LSSC}}}&\multirow{1}{*}{\textbf{{EPLL}}}&\multirow{1}{*}{\textbf{{Plow}}}&\multirow{1}{*}{\textbf{{NCSR}}}
&\multirow{1}{*}{\textbf{{GID}}}&\multirow{1}{*}{\textbf{{PGPD}}}&\multirow{1}{*}{\textbf{{LINC}}}&\multirow{1}{*}{\textbf{{aGMM}}}
&\multirow{1}{*}{\textbf{{OGLR}}}&\multirow{1}{*}{\textbf{{WNNM}}}&\multirow{1}{*}{\textbf{{RRC}}}\\
\hline
\multirow{1}{*}{Airplane}
&	0.8486 	&	0.9006 	&	0.9025 	&	0.9017 	&	0.8928 	&	0.9016 	&	0.8837 	&	0.8992 	&	0.9010 	&	0.9018 	&	0.8964 	&	\textcolor{blue}{\textbf{0.9051}} 	&	\textcolor{red}{\textbf{0.9053}} 	&	0.7441 	&	0.8631 	&	0.8669 	&	0.8628 	&	0.8532 	&	0.8660 	&	0.8449 	&	0.8646 	&	0.8654 	&	0.8647 	&	0.8588 	&	\textcolor{blue}{\textbf{0.8687}} 	&	\textcolor{red}{\textbf{0.8716}}
\\
\hline
\multirow{1}{*}{Barbara}
&	0.8770 	&	0.9099 	&	0.9017 	&	0.8864 	&	0.9002 	&	0.9073 	&	0.8758 	&	0.9051 	&	\textcolor{red}{\textbf{0.9191}} 	&	0.8920 	&	0.9036 	&	0.9133 	&	\textcolor{blue}{\textbf{0.9149}} 	&	0.7924 	&	0.8618 	&	0.8515 	&	0.8209 	&	0.8597 	&	0.8524 	&	0.8063 	&	0.8565 	&	0.8709 	&	0.8129 	&	0.8573 	&	\textcolor{blue}{\textbf{0.8690}} 	&	\textcolor{red}{\textbf{0.8736}}
\\
\hline
\multirow{1}{*}{boats}
&	0.8502 	&	0.8890 	&	0.8863 	&	0.8805 	&	0.8766 	&	0.8831 	&	0.8452 	&	0.8852 	&	\textcolor{red}{\textbf{0.8946}} 	&	0.8821 	&	0.8857 	&	\textcolor{blue}{\textbf{0.8937}} 	&	0.8877 	&	0.7571 	&	0.8424 	&	0.8403 	&	0.8317 	&	0.8289 	&	0.8346 	&	0.7934 	&	0.8404 	&	\textcolor{blue}{\textbf{0.8440}} 	&	0.8322 	&	0.8357 	&	\textcolor{red}{\textbf{0.8493}} 	&	0.8409
\\
\hline
\multirow{1}{*}{Elaine}
&	0.8673 	&	0.8900 	&	0.8879 	&	0.8813 	&	0.8871 	&	0.8880 	&	0.8645 	&	0.8867 	&	\textcolor{red}{\textbf{0.8915}} 	&	\textcolor{blue}{\textbf{0.8914}} 	&	0.8881 	&	0.8900 	&	0.8911 	&	0.7457 	&	0.8530 	&	0.8484 	&	0.8399 	&	0.8461 	&	0.8498 	&	0.8076 	&	0.8522 	&	0.8533 	&	0.8542 	&	0.8524 	&	\textcolor{blue}{\textbf{0.8549}} 	&	\textcolor{red}{\textbf{0.8580}}
\\
\hline
\multirow{1}{*}{Fence}
&	0.8378 	&	0.8762 	&	\textcolor{red}{\textbf{0.8836}} 	&	0.8698 	&	0.8561 	&	0.8767 	&	0.8501 	&	0.8714 	&	\textcolor{blue}{\textbf{0.8830}} 	&	0.8621 	&	0.8807 	&	0.8798 	&	0.8698 	&	0.7785 	&	0.8326 	&	\textcolor{blue}{\textbf{0.8364}} 	&	0.8150 	&	0.8182 	&	0.8298 	&	0.7947 	&	0.8255 	&	0.8291 	&	0.8021 	&	0.8344 	&	\textcolor{red}{\textbf{0.8366}} 	&	0.8246
\\
\hline
\multirow{1}{*}{Flower}
&	0.8280 	&	0.8751 	&	\textcolor{blue}{\textbf{0.8822}} 	&	0.8780 	&	0.8597 	&	0.8743 	&	0.8475 	&	0.8765 	&	\textcolor{red}{\textbf{0.8840}} 	&	0.8711 	&	0.8764 	&	0.8810 	&	0.8754 	&	0.7448 	&	0.8194 	&	0.8222 	&	0.8210 	&	0.8116 	&	0.8156 	&	0.7766 	&	0.8213 	&	\textcolor{red}{\textbf{0.8295}} 	&	0.8148 	&	0.8203 	&	\textcolor{blue}{\textbf{0.8282}} 	&	0.8240
\\
\hline
\multirow{1}{*}{Foreman}
&	0.8664 	&	0.9076 	&	0.9035 	&	0.8955 	&	0.9023 	&	0.9065 	&	0.8900 	&	0.9023 	&	0.9104 	&	0.9055 	&	0.9048 	&	\textcolor{blue}{\textbf{0.9109}} 	&	\textcolor{red}{\textbf{0.9116}} 	&	0.7216 	&	0.8823 	&	0.8826 	&	0.8617 	&	0.8698 	&	0.8846 	&	0.8551 	&	0.8818 	&	\textcolor{blue}{\textbf{0.8921}} 	&	0.8766 	&	0.8789 	&	0.8851 	&	\textcolor{red}{\textbf{0.8952}}
\\
\hline
\multirow{1}{*}{House}
&	0.8325 	&	0.8726 	&	\textcolor{red}{\textbf{0.8844}} 	&	0.8609 	&	0.8710 	&	0.8735 	&	0.8563 	&	0.8693 	&	0.8676 	&	0.8646 	&	\textcolor{blue}{\textbf{0.8775}} 	&	0.8727 	&	0.8663 	&	0.7118 	&	0.8480 	&	\textcolor{red}{\textbf{0.8566}} 	&	0.8338 	&	0.8383 	&	0.8479 	&	0.8243 	&	0.8471 	&	0.8501 	&	0.8435 	&	0.8448 	&	\textcolor{blue}{\textbf{0.8535}} 	&	0.8527
\\
\hline
\multirow{1}{*}{J. Bean}
&	0.8904 	&	0.9582 	&	0.9594 	&	0.9523 	&	0.9554 	&	0.9632 	&	0.9577 	&	0.9508 	&	0.9617 	&	\textcolor{blue}{\textbf{0.9632}} 	&	0.9592 	&	0.9617 	&	\textcolor{red}{\textbf{0.9644}} 	&	0.7572 	&	0.9357 	&	\textcolor{blue}{\textbf{0.9459}} 	&	0.9240 	&	0.9204 	&	0.9435 	&	0.9338 	&	0.9317 	&	0.9443 	&	0.9413 	&	0.9361 	&	0.9406 	&	\textcolor{red}{\textbf{0.9482}}
\\
\hline
\multirow{1}{*}{Leaves}
&	0.9360 	&	0.9534 	&	0.9566 	&	0.9480 	&	0.9376 	&	0.9555 	&	0.9493 	&	0.9562 	&	0.9555 	&	0.9559 	&	0.9521 	&	\textcolor{red}{\textbf{0.9635}} 	&	\textcolor{blue}{\textbf{0.9599}} 	&	0.8780 	&	0.9278 	&	0.9209 	&	0.9197 	&	0.9057 	&	0.9311 	&	0.9248 	&	0.9300 	&	0.9311 	&	0.9273 	&	0.9266 	&	\textcolor{red}{\textbf{0.9421}} 	&	\textcolor{blue}{\textbf{0.9366}}
\\
\hline
\multirow{1}{*}{Lena}
&	0.8597 	&	0.8985 	&	0.8998 	&	0.8913 	&	0.8891 	&	0.8979 	&	0.8679 	&	0.8981 	&	\textcolor{red}{\textbf{0.9058}} 	&	0.8960 	&	0.8944 	&	0.8998 	&	\textcolor{blue}{\textbf{0.9020}} 	&	0.7543 	&	0.8584 	&	0.8593 	&	0.8477 	&	0.8493 	&	0.8580 	&	0.8185 	&	0.8622 	&	\textcolor{red}{\textbf{0.8703}} 	&	0.8548 	&	0.8560 	&	0.8643 	&	\textcolor{blue}{\textbf{0.8672}}
\\
\hline
\multirow{1}{*}{Lin}
&	0.8404 	&	\textcolor{blue}{\textbf{0.9017}} 	&	0.8931 	&	0.8942 	&	0.8982 	&	0.8983 	&	0.8773 	&	0.8910 	&	\textcolor{red}{\textbf{0.9018}} 	&	0.8957 	&	0.8990 	&	0.9001 	&	0.8988 	&	0.7055 	&	0.8672 	&	0.8611 	&	0.8546 	&	0.8588 	&	0.8632 	&	0.8287 	&	0.8606 	&	0.8669 	&	0.8634 	&	0.8592 	&	\textcolor{red}{\textbf{0.8719}} 	&	\textcolor{blue}{\textbf{0.8702}}
\\
\hline
\multirow{1}{*}{Monarch}
&	0.8921 	&	0.9179 	&	0.9186 	&	0.9166 	&	0.9097 	&	0.9192 	&	0.9027 	&	0.9187 	&	0.9230 	&	0.9202 	&	0.9171 	&	\textcolor{blue}{\textbf{0.9258}} 	&	\textcolor{red}{\textbf{0.9263}} 	&	0.7980 	&	0.8822 	&	0.8803 	&	0.8789 	&	0.8714 	&	0.8829 	&	0.8628 	&	0.8853 	&	0.8914 	&	0.8831 	&	0.8831 	&	\textcolor{blue}{\textbf{0.8950}} 	&	\textcolor{red}{\textbf{0.8954}}
\\
\hline
\multirow{1}{*}{Parrot}
&	0.8568 	&	\textcolor{blue}{\textbf{0.9002}} 	&	0.8951 	&	0.8924 	&	0.8952 	&	0.8995 	&	0.8824 	&	0.8945 	&	\textcolor{red}{\textbf{0.9020}} 	&	0.8951 	&	0.8941 	&	0.8993 	&	0.9001 	&	0.7337 	&	0.8705 	&	0.8669 	&	0.8569 	&	0.8617 	&	0.8705 	&	0.8435 	&	0.8681 	&	\textcolor{blue}{\textbf{0.8754}} 	&	0.8671 	&	0.8609 	&	0.8745 	&	\textcolor{red}{\textbf{0.8765}}
\\
\hline
\multirow{1}{*}{Plants}
&	0.8416 	&	0.8811 	&	0.8795 	&	0.8744 	&	0.8743 	&	0.8753 	&	0.8506 	&	0.8790 	&	\textcolor{blue}{\textbf{0.8867}} 	&	0.8773 	&	0.8803 	&	\textcolor{red}{\textbf{0.8875}} 	&	0.8813 	&	0.7141 	&	0.8373 	&	0.8330 	&	0.8278 	&	0.8270 	&	0.8273 	&	0.7947 	&	0.8370 	&	\textcolor{blue}{\textbf{0.8447}} 	&	0.8314 	&	0.8352 	&	0.8437 	&	\textcolor{red}{\textbf{0.8459}}
\\
\hline
\multirow{1}{*}{Starfish}
&	0.8509 	&	0.8748 	&	\textcolor{blue}{\textbf{0.8774}} 	&	0.8756 	&	0.8561 	&	0.8748 	&	0.8546 	&	0.8756 	&	0.8694 	&	0.8756 	&	0.8676 	&	\textcolor{red}{\textbf{0.8831}} 	&	0.8720 	&	0.7725 	&	0.8289 	&	0.8238 	&	0.8248 	&	0.8075 	&	0.8283 	&	0.8028 	&	0.8277 	&	0.8234 	&	0.8263 	&	0.8195 	&	\textcolor{red}{\textbf{0.8382}} 	&	\textcolor{blue}{\textbf{0.8304}}
\\
\hline
\multirow{1}{*}{\textbf{Average}}
&	0.8610 	&	0.9004 	&	0.9007 	&	0.8937 	&	0.8913 	&	0.8997 	&	0.8785 	&	0.8975 	&	\textcolor{red}{\textbf{0.9036}} 	&	0.8969 	&	0.8986 	&	\textcolor{blue}{\textbf{0.9042}} 	&	0.9017 	&	0.7568 	&	0.8632 	&	0.8623 	&	0.8513 	&	0.8517 	&	0.8616 	&	0.8320 	&	0.8620 	&	0.8676 	&	0.8560 	&	0.8600 	&	\textcolor{red}{\textbf{0.8697}} 	&	\textcolor{blue}{\textbf{0.8694}}
\\
\hline

\multicolumn{1}{|c|}{}&\multicolumn{13}{|c||}{${\bf \sigma}_n = 40$}&\multicolumn{13}{|c|}{${\bf \sigma}_n = 50$}\\
\hline						\multirow{1}{*}{\textbf{{Images}}}&\multirow{1}{*}{\textbf{{NNM}}}&\multirow{1}{*}{\textbf{{BM3D}}}&\multirow{1}{*}{\textbf{{LSSC}}}
&\multirow{1}{*}{\textbf{{EPLL}}}&\multirow{1}{*}{\textbf{{Plow}}}&\multirow{1}{*}{\textbf{{NCSR}}}&\multirow{1}{*}{\textbf{{GID}}}
&\multirow{1}{*}{\textbf{{PGPD}}}&\multirow{1}{*}{\textbf{{LINC}}}&\multirow{1}{*}{\textbf{{aGMM}}}&\multirow{1}{*}{\textbf{{OGLR}}}
&\multirow{1}{*}{\textbf{{WNNM}}}&\multirow{1}{*}{\textbf{{RRC}}}&\multirow{1}{*}{\textbf{{NNM}}}&\multirow{1}{*}{\textbf{{BM3D}}}
&\multirow{1}{*}{\textbf{{LSSC}}}&\multirow{1}{*}{\textbf{{EPLL}}}&\multirow{1}{*}{\textbf{{Plow}}}&\multirow{1}{*}{\textbf{{NCSR}}}
&\multirow{1}{*}{\textbf{{GID}}}&\multirow{1}{*}{\textbf{{PGPD}}}&\multirow{1}{*}{\textbf{{LINC}}}&\multirow{1}{*}{\textbf{{aGMM}}}
&\multirow{1}{*}{\textbf{{OGLR}}}&\multirow{1}{*}{\textbf{{WNNM}}}&\multirow{1}{*}{\textbf{{RRC}}}\\
\hline
\multirow{1}{*}{Airplane}
&	0.7439 	&	0.8277 	&	0.8372 	&	0.8264 	&	0.8122 	&	0.8330 	&	0.8120 	&	0.8345 	&	0.8333 	&	0.8305 	&	0.8289 	&	\textcolor{blue}{\textbf{0.8378}} 	&	\textcolor{red}{\textbf{0.8429}} 	&	0.6839 	&	0.8044 	&	0.8074 	&	0.7922 	&	0.7698 	&	0.8066 	&	0.7795 	&	0.8059 	&	0.8054 	&	0.7990 	&	0.7848 	&	\textcolor{blue}{\textbf{0.8104}} 	&	\textcolor{red}{\textbf{0.8172}}
\\
\hline
\multirow{1}{*}{Barbara}
&	0.7639 	&	0.8070 	&	0.8018 	&	0.7533 	&	0.8141 	&	0.8006 	&	0.7381 	&	0.8077 	&	0.8209 	&	0.7453 	&	0.8172 	&	\textcolor{blue}{\textbf{0.8240}} 	&	\textcolor{red}{\textbf{0.8308}} 	&	0.7004 	&	0.7698 	&	0.7596 	&	0.6943 	&	0.7663 	&	0.7572 	&	0.6769 	&	0.7613 	&	0.7655 	&	0.7021 	&	0.7630 	&	\textcolor{red}{\textbf{0.7883}} 	&	\textcolor{blue}{\textbf{0.7872}}
\\
\hline
\multirow{1}{*}{boats}
&	0.7412 	&	0.7997 	&	0.7977 	&	0.7888 	&	0.7832 	&	0.7906 	&	0.7441 	&	0.8021 	&	0.8021 	&	0.7909 	&	0.7971 	&	\textcolor{red}{\textbf{0.8087}} 	&	\textcolor{blue}{\textbf{0.8047}} 	&	0.6830 	&	0.7667 	&	0.7569 	&	0.7504 	&	0.7396 	&	0.7541 	&	0.7054 	&	0.7683 	&	0.7659 	&	0.7544 	&	0.7477 	&	\textcolor{blue}{\textbf{0.7702}} 	&	\textcolor{red}{\textbf{0.7738}}
\\
\hline
\multirow{1}{*}{Elaine}
&	0.7503 	&	0.8180 	&	0.8123 	&	0.8047 	&	0.8070 	&	0.8258 	&	0.7838 	&	0.8223 	&	0.8206 	&	0.8209 	&	0.8198 	&	\textcolor{blue}{\textbf{0.8259}} 	&	\textcolor{red}{\textbf{0.8284}} 	&	0.6982 	&	0.7971 	&	0.7799 	&	0.7741 	&	0.7699 	&	0.7976 	&	0.7494 	&	0.7926 	&	0.7902 	&	0.7890 	&	0.7638 	&	\textcolor{blue}{\textbf{0.7995}} 	&	\textcolor{red}{\textbf{0.8022}}
\\
\hline
\multirow{1}{*}{Fence}
&	0.7536 	&	0.7961 	&	\textcolor{blue}{\textbf{0.7978}} 	&	0.7640 	&	0.7828 	&	0.7805 	&	0.7514 	&	0.7908 	&	0.7868 	&	0.7496 	&	0.7975 	&	\textcolor{red}{\textbf{0.8011}} 	&	0.7879 	&	0.6988 	&	\textcolor{blue}{\textbf{0.7621}} 	&	0.7588 	&	0.7162 	&	0.7496 	&	0.7476 	&	0.7051 	&	0.7573 	&	0.7492 	&	0.7010 	&	0.7565 	&	\textcolor{red}{\textbf{0.7716}} 	&	0.7561
\\
\hline
\multirow{1}{*}{Flower}
&	0.7155 	&	0.7696 	&	0.7729 	&	0.7710 	&	0.7605 	&	0.7621 	&	0.7140 	&	0.7738 	&	0.7773 	&	0.7672 	&	0.7779 	&	\textcolor{red}{\textbf{0.7810}} 	&	\textcolor{blue}{\textbf{0.7797}} 	&	0.6526 	&	0.7283 	&	0.7292 	&	0.7273 	&	0.7122 	&	0.7217 	&	0.6430 	&	0.7324 	&	0.7293 	&	0.7250 	&	0.7224 	&	\textcolor{blue}{\textbf{0.7400}} 	&	\textcolor{red}{\textbf{0.7413}}
\\
\hline
\multirow{1}{*}{Foreman}
&	0.7532 	&	0.8565 	&	0.8641 	&	0.8315 	&	0.8354 	&	0.8723 	&	0.8262 	&	0.8621 	&	\textcolor{blue}{\textbf{0.8736}} 	&	0.8515 	&	0.8610 	&	0.8611 	&	\textcolor{red}{\textbf{0.8780}} 	&	0.6983 	&	0.8445 	&	0.8438 	&	0.8051 	&	0.7976 	&	0.8559 	&	0.8080 	&	0.8410 	&	0.8542 	&	0.8270 	&	0.8198 	&	\textcolor{blue}{\textbf{0.8523}} 	&	\textcolor{red}{\textbf{0.8611}}
\\
\hline
\multirow{1}{*}{House}
&	0.7342 	&	0.8256 	&	0.8326 	&	0.8089 	&	0.8058 	&	0.8323 	&	0.7979 	&	0.8302 	&	\textcolor{blue}{\textbf{0.8352}} 	&	0.8221 	&	0.8218 	&	0.8339 	&	\textcolor{red}{\textbf{0.8393}} 	&	0.6780 	&	0.8122 	&	0.8175 	&	0.7845 	&	0.7699 	&	0.8160 	&	0.7718 	&	0.8125 	&	0.8221 	&	0.8002 	&	0.7824 	&	\textcolor{red}{\textbf{0.8273}} 	&	\textcolor{blue}{\textbf{0.8247}}
\\
\hline
\multirow{1}{*}{J. Bean}
&	0.7916 	&	0.9122 	&	\textcolor{blue}{\textbf{0.9299}} 	&	0.8956 	&	0.8847 	&	0.9296 	&	0.9139 	&	0.9133 	&	0.9269 	&	0.9170 	&	0.9137 	&	0.9221 	&	\textcolor{red}{\textbf{0.9308}} 	&	0.7293 	&	0.9006 	&	0.9125 	&	0.8677 	&	0.8430 	&	\textcolor{red}{\textbf{0.9134}} 	&	0.8956 	&	0.8934 	&	0.9086 	&	0.8911 	&	0.8737 	&	0.9064 	&	\textcolor{blue}{\textbf{0.9125}}
\\
\hline
\multirow{1}{*}{Leaves}
&	0.8694 	&	0.8961 	&	0.8939 	&	0.8916 	&	0.8701 	&	0.9028 	&	0.8994 	&	0.9039 	&	0.9084 	&	0.8979 	&	0.8902 	&	\textcolor{red}{\textbf{0.9200}} 	&	\textcolor{blue}{\textbf{0.9139}} 	&	0.8250 	&	0.8680 	&	0.8679 	&	0.8638 	&	0.8354 	&	0.8787 	&	0.8693 	&	0.8794 	&	0.8849 	&	0.8673 	&	0.8484 	&	\textcolor{red}{\textbf{0.8976}} 	&	\textcolor{blue}{\textbf{0.8910}}
\\
\hline
\multirow{1}{*}{Lena}
&	0.7505 	&	0.8178 	&	0.8234 	&	0.8092 	&	0.8081 	&	0.8280 	&	0.7737 	&	0.8297 	&	0.8337 	&	0.8165 	&	0.8250 	&	\textcolor{blue}{\textbf{0.8340}} 	&	\textcolor{red}{\textbf{0.8353}} 	&	0.6966 	&	0.7920 	&	0.7937 	&	0.7732 	&	0.7691 	&	0.8009 	&	0.7289 	&	0.7990 	&	0.7997 	&	0.7820 	&	0.7764 	&	\textcolor{blue}{\textbf{0.8049}} 	&	\textcolor{red}{\textbf{0.8073}}
\\
\hline
\multirow{1}{*}{Lin}
&	0.7263 	&	0.8369 	&	0.8315 	&	0.8210 	&	0.8197 	&	0.8385 	&	0.8145 	&	0.8351 	&	\textcolor{blue}{\textbf{0.8395}} 	&	0.8342 	&	0.8301 	&	0.8373 	&	\textcolor{red}{\textbf{0.8422}} 	&	0.6649 	&	0.8170 	&	0.8098 	&	0.7908 	&	0.7806 	&	\textcolor{blue}{\textbf{0.8171}} 	&	0.7896 	&	0.8118 	&	0.8167 	&	0.8073 	&	0.7871 	&	\textcolor{red}{\textbf{0.8268}} 	&	0.8140
\\
\hline
\multirow{1}{*}{Monarch}
&	0.7944 	&	0.8446 	&	0.8499 	&	0.8441 	&	0.8316 	&	0.8522 	&	0.8266 	&	0.8549 	&	0.8602 	&	0.8478 	&	0.8512 	&	\textcolor{blue}{\textbf{0.8634}} 	&	\textcolor{red}{\textbf{0.8650}} 	&	0.7428 	&	0.8200 	&	0.8250 	&	0.8124 	&	0.7910 	&	0.8252 	&	0.8022 	&	0.8269 	&	0.8294 	&	0.8164 	&	0.8038 	&	\textcolor{red}{\textbf{0.8380}} 	&	\textcolor{blue}{\textbf{0.8361}}
\\
\hline

\multirow{1}{*}{Parrot}
&	0.7532 	&	0.8428 	&	0.8430 	&	0.8265 	&	0.8251 	&	0.8491 	&	0.8003 	&	0.8464 	&	\textcolor{blue}{\textbf{0.8538}} 	&	0.8421 	&	0.8369 	&	0.8471 	&	\textcolor{red}{\textbf{0.8555}} 	&	0.6952 	&	0.8273 	&	0.8224 	&	0.7998 	&	0.7872 	&	0.8310 	&	0.7485 	&	0.8246 	&	\textcolor{blue}{\textbf{0.8349}} 	&	0.8174 	&	0.7949 	&	0.8335 	&	\textcolor{red}{\textbf{0.8371}}
\\
\hline
\multirow{1}{*}{Plants}
&	0.7159 	&	0.7961 	&	0.7914 	&	0.7856 	&	0.7792 	&	0.7895 	&	0.7557 	&	0.8016 	&	\textcolor{blue}{\textbf{0.8059}} 	&	0.7928 	&	0.7996 	&	0.8014 	&	\textcolor{red}{\textbf{0.8151}} 	&	0.6545 	&	0.7669 	&	0.7553 	&	0.7479 	&	0.7327 	&	0.7589 	&	0.7226 	&	0.7669 	&	0.7667 	&	0.7561 	&	0.7452 	&	\textcolor{blue}{\textbf{0.7717}} 	&	\textcolor{red}{\textbf{0.7789}}
\\
\hline
\multirow{1}{*}{Starfish}
&	0.7448 	&	0.7828 	&	0.7803 	&	0.7802 	&	0.7608 	&	0.7812 	&	0.7565 	&	0.7855 	&	0.7795 	&	0.7824 	&	0.7773 	&	\textcolor{blue}{\textbf{0.7909}} 	&	\textcolor{red}{\textbf{0.7925}} 	&	0.6887 	&	0.7433 	&	0.7421 	&	0.7392 	&	0.7175 	&	0.7440 	&	0.7011 	&	0.7457 	&	0.7358 	&	0.7419 	&	0.7258 	&	\textcolor{blue}{\textbf{0.7542}} 	&	\textcolor{red}{\textbf{0.7589}}
\\
\hline
\multirow{1}{*}{\textbf{Average}}
&	0.7564 	&	0.8268 	&	0.8287 	&	0.8126 	&	0.8113 	&	0.8293 	&	0.7943 	&	0.8309 	&	0.8349 	&	0.8193 	&	0.8278 	&	\textcolor{blue}{\textbf{0.8369}} 	&	\textcolor{red}{\textbf{0.8401}} 	&	0.6994 	&	0.8013 	&	0.7989 	&	0.7774 	&	0.7707 	&	0.8016 	&	0.7561 	&	0.8012 	&	0.8037 	&	0.7861 	&	0.7810 	&	\textcolor{blue}{\textbf{0.8120}} 	&	\textcolor{red}{\textbf{0.8125}}
\\
\hline

\multicolumn{1}{|c|}{}&\multicolumn{13}{|c||}{${\bf \sigma}_n = 75$}&\multicolumn{13}{|c|}{${\bf \sigma}_n = 100$}\\
\hline						\multirow{1}{*}{\textbf{{Images}}}&\multirow{1}{*}{\textbf{{NNM}}}&\multirow{1}{*}{\textbf{{BM3D}}}&\multirow{1}{*}{\textbf{{LSSC}}}
&\multirow{1}{*}{\textbf{{EPLL}}}&\multirow{1}{*}{\textbf{{Plow}}}&\multirow{1}{*}{\textbf{{NCSR}}}&\multirow{1}{*}{\textbf{{GID}}}
&\multirow{1}{*}{\textbf{{PGPD}}}&\multirow{1}{*}{\textbf{{LINC}}}&\multirow{1}{*}{\textbf{{aGMM}}}&\multirow{1}{*}{\textbf{{OGLR}}}
&\multirow{1}{*}{\textbf{{WNNM}}}&\multirow{1}{*}{\textbf{{RRC}}}&\multirow{1}{*}{\textbf{{NNM}}}&\multirow{1}{*}{\textbf{{BM3D}}}
&\multirow{1}{*}{\textbf{{LSSC}}}&\multirow{1}{*}{\textbf{{EPLL}}}&\multirow{1}{*}{\textbf{{Plow}}}&\multirow{1}{*}{\textbf{{NCSR}}}
&\multirow{1}{*}{\textbf{{GID}}}&\multirow{1}{*}{\textbf{{PGPD}}}&\multirow{1}{*}{\textbf{{LINC}}}&\multirow{1}{*}{\textbf{{aGMM}}}
&\multirow{1}{*}{\textbf{{OGLR}}}&\multirow{1}{*}{\textbf{{WNNM}}}&\multirow{1}{*}{\textbf{{RRC}}}\\
\hline
\multirow{1}{*}{Airplane}
&	0.5493 	&	0.7488 	&	0.7455 	&	0.7168 	&	0.6589 	&	0.7547 	&	0.6754 	&	0.7492 	&	0.7436 	&	0.7248 	&	0.7174 	&	\textcolor{blue}{\textbf{0.7576}} 	&	\textcolor{red}{\textbf{0.7637}} 	&	0.5005 	&	0.7036 	&	0.7036 	&	0.6523 	&	0.5698 	&	\textcolor{blue}{\textbf{0.7107}} 	&	0.6393 	&	0.6947 	&	0.6960 	&	0.6571 	&	0.6400 	&	0.7080 	&	\textcolor{red}{\textbf{0.7209}}
\\
\hline
\multirow{1}{*}{Barbara}
&	0.5691 	&	0.6798 	&	0.6654 	&	0.5848 	&	0.6548 	&	0.6616 	&	0.5410 	&	0.6729 	&	0.6606 	&	0.5882 	&	0.6791 	&	\textcolor{red}{\textbf{0.6967}} 	&	\textcolor{blue}{\textbf{0.6825}} 	&	0.5026 	&	0.6092 	&	0.5937 	&	0.5135 	&	0.5647 	&	0.5960 	&	0.4960 	&	0.6039 	&	0.5917 	&	0.5163 	&	0.5755 	&	\textcolor{blue}{\textbf{0.6173}} 	&	\textcolor{red}{\textbf{0.6243}}
\\
\hline
\multirow{1}{*}{boats}
&	0.5524 	&	0.6939 	&	0.6831 	&	0.6674 	&	0.6386 	&	0.6876 	&	0.6247 	&	0.6963 	&	0.6864 	&	0.6727 	&	0.6637 	&	\textcolor{red}{\textbf{0.7081}} 	&	\textcolor{blue}{\textbf{0.6997}} 	&	0.4880 	&	0.6375 	&	0.6292 	&	0.5988 	&	0.5548 	&	0.6294 	&	0.5527 	&	0.6355 	&	0.6250 	&	0.6022 	&	0.5764 	&	\textcolor{red}{\textbf{0.6582}} 	&	\textcolor{blue}{\textbf{0.6429}}
\\
\hline
\multirow{1}{*}{Elaine}
&	0.5618 	&	0.7359 	&	0.7116 	&	0.7062 	&	0.6772 	&	0.7342 	&	0.6762 	&	0.7303 	&	0.7235 	&	0.7172 	&	0.6960 	&	\textcolor{blue}{\textbf{0.7434}} 	&	\textcolor{red}{\textbf{0.7438}} 	&	0.5128 	&	0.6817 	&	0.6585 	&	0.6459 	&	0.6003 	&	0.6824 	&	0.6139 	&	0.6721 	&	0.6719 	&	0.6548 	&	0.6048 	&	\textcolor{blue}{\textbf{0.6869}} 	&	\textcolor{red}{\textbf{0.6955}}
\\
\hline

\multirow{1}{*}{Fence}
&	0.5890 	&	0.6962 	&	0.6821 	&	0.6076 	&	0.6586 	&	0.6742 	&	0.5826 	&	0.6872 	&	0.6712 	&	0.6098 	&	0.6848 	&	\textcolor{red}{\textbf{0.7110}} 	&	\textcolor{blue}{\textbf{0.6924}} 	&	0.5044 	&	0.6362 	&	0.6216 	&	0.5252 	&	0.5727 	&	0.6009 	&	0.4671 	&	0.6226 	&	0.6131 	&	0.5386 	&	0.6119 	&	\textcolor{red}{\textbf{0.6755}} 	&	\textcolor{blue}{\textbf{0.6407}}
\\
\hline

\multirow{1}{*}{Flower}
&	0.5168 	&	0.6482 	&	0.6387 	&	0.6296 	&	0.6024 	&	0.6409 	&	0.5797 	&	0.6468 	&	0.6278 	&	0.6256 	&	0.6352 	&	\textcolor{red}{\textbf{0.6574}} 	&	\textcolor{blue}{\textbf{0.6499}} 	&	0.4402 	&	\textcolor{blue}{\textbf{0.5862}} 	&	0.5734 	&	0.5533 	&	0.5181 	&	0.5753 	&	0.3815 	&	0.5797 	&	0.5568 	&	0.5457 	&	0.5452 	&	\textcolor{red}{\textbf{0.5934}} 	&	0.5846
\\
\hline

\multirow{1}{*}{Foreman}
&	0.5524 	&	0.7933 	&	0.8015 	&	0.7467 	&	0.7067 	&	\textcolor{blue}{\textbf{0.8171}} 	&	0.7556 	&	0.7965 	&	0.8129 	&	0.7676 	&	0.7673 	&	0.8105 	&	\textcolor{red}{\textbf{0.8259}} 	&	0.5160 	&	0.7489 	&	0.7722 	&	0.6949 	&	0.6329 	&	0.7833 	&	0.7050 	&	0.7452 	&	\textcolor{blue}{\textbf{0.7852}} 	&	0.7129 	&	0.6983 	&	0.7820 	&	\textcolor{red}{\textbf{0.7969}}
\\
\hline

\multirow{1}{*}{House}
&	0.5439 	&	0.7645 	&	0.7792 	&	0.7251 	&	0.6733 	&	0.7749 	&	0.7052 	&	0.7709 	&	0.7842 	&	0.7419 	&	0.7230 	&	\textcolor{blue}{\textbf{0.7930}} 	&	\textcolor{red}{\textbf{0.7950}} 	&	0.4918 	&	0.7203 	&	0.7313 	&	0.6695 	&	0.5874 	&	0.7397 	&	0.4735 	&	0.7195 	&	0.7478 	&	0.6854 	&	0.6373 	&	\textcolor{blue}{\textbf{0.7531}} 	&	\textcolor{red}{\textbf{0.7655}}
\\
\hline

\multirow{1}{*}{J. Bean}
&	0.5796 	&	0.8573 	&	0.8720 	&	0.8019 	&	0.7422 	&	\textcolor{red}{\textbf{0.8792}} 	&	0.8574 	&	0.8503 	&	0.8662 	&	0.8243 	&	0.8088 	&	0.8645 	&	\textcolor{blue}{\textbf{0.8749}} 	&	0.5341 	&	0.8181 	&	0.8376 	&	0.7429 	&	0.6574 	&	\textcolor{red}{\textbf{0.8472}} 	&	0.7862 	&	0.7999 	&	0.8315 	&	0.7628 	&	0.7331 	&	0.8195 	&	\textcolor{blue}{\textbf{0.8443}}
\\
\hline

\multirow{1}{*}{Leaves}
&	0.7265 	&	0.8072 	&	0.7869 	&	0.7921 	&	0.7512 	&	0.8234 	&	0.7751 	&	0.8121 	&	0.8218 	&	0.7867 	&	0.7763 	&	\textcolor{red}{\textbf{0.8435}} 	&	\textcolor{blue}{\textbf{0.8377}} 	&	0.6345 	&	0.7482 	&	0.7242 	&	0.7163 	&	0.6814 	&	0.7622 	&	0.6857 	&	0.7469 	&	0.7467 	&	0.7106 	&	0.6827 	&	\textcolor{red}{\textbf{0.7946}} 	&	\textcolor{blue}{\textbf{0.7811}}
\\
\hline

\multirow{1}{*}{Lena}
&	0.5647 	&	0.7288 	&	0.7249 	&	0.6968 	&	0.6723 	&	0.7415 	&	0.6700 	&	0.7356 	&	0.7359 	&	0.7101 	&	0.7061 	&	\textcolor{blue}{\textbf{0.7418}} 	&	\textcolor{red}{\textbf{0.7498}} 	&	0.5093 	&	0.6739 	&	0.6766 	&	0.6345 	&	0.5895 	&	0.6906 	&	0.6024 	&	0.6780 	&	\textcolor{blue}{\textbf{0.6918}} 	&	0.6487 	&	0.6215 	&	0.6917 	&	\textcolor{red}{\textbf{0.7100}}
\\
\hline

\multirow{1}{*}{Lin}
&	0.5175 	&	0.7673 	&	0.7638 	&	0.7238 	&	0.6722 	&	0.7730 	&	0.7224 	&	0.7669 	&	\textcolor{red}{\textbf{0.7784}} 	&	0.7433 	&	0.7189 	&	\textcolor{blue}{\textbf{0.7755}} 	&	0.7729 	&	0.4858 	&	0.7262 	&	0.7300 	&	0.6669 	&	0.5907 	&	0.7393 	&	0.6741 	&	0.7151 	&	\textcolor{red}{\textbf{0.7478}} 	&	0.6831 	&	0.6406 	&	0.7392 	&	\textcolor{blue}{\textbf{0.7416}}
\\
\hline

\multirow{1}{*}{Monarch}
&	0.6206 	&	0.7557 	&	0.7503 	&	0.7395 	&	0.6917 	&	0.7648 	&	0.7071 	&	0.7642 	&	0.7651 	&	0.7454 	&	0.7378 	&	\textcolor{blue}{\textbf{0.7759}} 	&	\textcolor{red}{\textbf{0.7782}} 	&	0.5596 	&	0.7021 	&	0.6999 	&	0.6771 	&	0.6102 	&	0.7109 	&	0.6361 	&	0.7029 	&	0.7037 	&	0.6823 	&	0.6419 	&	\textcolor{blue}{\textbf{0.7284}} 	&	\textcolor{red}{\textbf{0.7312}}
\\
\hline

\multirow{1}{*}{Parrot}
&	0.5567 	&	0.7771 	&	0.7673 	&	0.7399 	&	0.6859 	&	0.7892 	&	0.7498 	&	0.7775 	&	\textcolor{blue}{\textbf{0.7951}} 	&	0.7555 	&	0.7333 	&	0.7933 	&	\textcolor{red}{\textbf{0.8028}} 	&	0.5197 	&	0.7345 	&	0.7338 	&	0.6844 	&	0.6096 	&	0.7518 	&	0.6998 	&	0.7251 	&	\textcolor{blue}{\textbf{0.7598}} 	&	0.6979 	&	0.6531 	&	0.7533 	&	\textcolor{red}{\textbf{0.7729}}
\\
\hline

\multirow{1}{*}{Plants}
&	0.5107 	&	0.7006 	&	0.6855 	&	0.6720 	&	0.6255 	&	0.7007 	&	0.6541 	&	0.7009 	&	0.6999 	&	0.6805 	&	0.6639 	&	\textcolor{blue}{\textbf{0.7107}} 	&	\textcolor{red}{\textbf{0.7172}} 	&	0.4789 	&	0.6525 	&	0.6393 	&	0.6129 	&	0.5531 	&	\textcolor{blue}{\textbf{0.6587}} 	&	0.6010 	&	0.6472 	&	0.6558 	&	0.6210 	&	0.5862 	&	0.6560 	&	\textcolor{red}{\textbf{0.6680}}
\\
\hline

\multirow{1}{*}{Starfish}
&	0.5617 	&	0.6670 	&	0.6528 	&	0.6502 	&	0.6192 	&	0.6685 	&	0.6111 	&	0.6638 	&	0.6435 	&	0.6525 	&	0.6446 	&	\textcolor{blue}{\textbf{0.6658}} 	&	\textcolor{red}{\textbf{0.6741}} 	&	0.4979 	&	0.6053 	&	0.5816 	&	0.5799 	&	0.5403 	&	0.6062 	&	0.5415 	&	0.6018 	&	0.5680 	&	0.5813 	&	0.5528 	&	\textcolor{red}{\textbf{0.6173}} 	&	\textcolor{blue}{\textbf{0.6081}}
\\
\hline

\multirow{1}{*}{\textbf{Average}}
&	0.5671 	&	0.7389 	&	0.7319 	&	0.7000 	&	0.6707 	&	0.7428 	&	0.6805 	&	0.7388 	&	0.7385 	&	0.7091 	&	0.7098 	&	\textcolor{blue}{\textbf{0.7531}} 	&	\textcolor{red}{\textbf{0.7538}} 	&	0.5110 	&	0.6865 	&	0.6817 	&	0.6355 	&	0.5896 	&	0.6928 	&	0.5972 	&	0.6806 	&	0.6871 	&	0.6438 	&	0.6251 	&	\textcolor{blue}{\textbf{0.7047}} 	&	\textcolor{red}{\textbf{0.7080}}
\\
\hline

\end{tabular}
}
\label{Tab:2}
\vspace{-2mm}
\end{table*}

\subsection {Dictionary Learning based on SVD}
\label{5.2}
\bihan{To achieve the equivalence of RRC and GSRC, we introduce a specific approach to learn the group-wise dictionaries.
For each group ${\textbf{\emph{X}}}_i\in\mathbb{R}^{d \times m}$, the corresponding dictionary is constructed using the SVD of the corrupted measurement ${\textbf{\emph{Y}}}_i\in\mathbb{R}^{d \times m}$, which is}
%Hereby, an adaptive dictionary learning method is devised, that is, for each group ${\textbf{\emph{X}}}_i\in\mathbb{R}^{d \times m}$, its adaptive dictionary can be learned from its noisy observation ${\textbf{\emph{Y}}}_i\in\mathbb{R}^{d \times m}$.
%Specifically, we apply the SVD to ${\textbf{\emph{Y}}}_i$,
\begin{equation}
{\textbf{\emph{Y}}}_i= {\textbf{\emph{U}}}_i{\boldsymbol\Delta}_i{\textbf{\emph{V}}}_i^T=\sum\nolimits_{k=1}^{j} \boldsymbol\delta_{i,k}{\textbf{\emph{u}}}_{i,k}{\textbf{\emph{v}}}_{i,k}^T,
\label{eq:42}
\end{equation}%equal 42
where $\boldsymbol\delta_i=[\boldsymbol\delta_{i,1},\dots,\boldsymbol\delta_{i,j}]$ and $j={\rm min}(d,m)$; ${\boldsymbol\Delta}_i={\rm diag}(\boldsymbol\delta_i)$ is a diagonal matrix whose non-zero elements are represented by $\boldsymbol\delta_i$; ${\textbf{\emph{u}}}_{i,k}, {\textbf{\emph{v}}}_{i,k}$ are the columns of ${\textbf{\emph{U}}}_i$ and ${\textbf{\emph{V}}}_i$, respectively.
\bihan{Each dictionary atom $\textbf{\emph{d}}_{i,k}$ of the group-wise dictionary $\textbf{\emph{D}}_i$ is constructed as}
\begin{equation}
\textbf{\emph{d}}_{i,k}={\textbf{\emph{u}}}_{i,k}{\textbf{\emph{v}}}_{i,k}^T, \qquad \forall k=1,\dots,{j}.
\label{eq:43}
\end{equation}%equal 43
\bihan{Therefore, the adaptive dictionary for each data group is formed as $\textbf{\emph{D}}_i=[\textbf{\emph{d}}_{i,1},\textbf{\emph{d}}_{i,2},\dots,\textbf{\emph{d}}_{i,j}]$ using the corresponding corrupted measurement ${\textbf{\emph{Y}}}_i$.}
%Then, an adaptive dictionary $\textbf{\emph{D}}_i=[\textbf{\emph{d}}_{i,1},\textbf{\emph{d}}_{i,2},\dots,\textbf{\emph{d}}_{i,j}]$ has been learned for each group ${\textbf{\emph{Y}}}_i$.

\subsection {The Equivalence of RRC and GSRC} \label{5.3}
\bihan{We show the equivalence of the RRC and GSRC models by showing that the problem Eq.~\eqref{eq:6} and Eq.~\eqref{eq:41} are equivalent, provided that the group-wise dictionaries are constructed using Eq.~\eqref{eq:43}.
We first prove the Lemma~\ref{lemma:2}.
}
%Now, recall the classical $\ell_1$-norm based GSR problem in Eq.~\eqref{eq:38} and the adaptive dictionary defined in Eq.~\eqref{eq:43}. In order to prove that RRC is equivalent to GSRC, we first introduce the following Lemma.
\begin{lemma}
\label{lemma:2}
Let ${\textbf{{Y}}}_i={{\textbf{{D}}}_i{\textbf{{K}}}_i}$, ${\textbf{{X}}}_i={{\textbf{{D}}}_i{\textbf{{A}}}_i}$, and ${\textbf{{D}}}_i$ is constructed by Eq.~\eqref{eq:43}. We have
\begin{equation}
\left\|{\textbf{{Y}}}_i-{\textbf{{X}}}_i\right\|_F^2=\left\|{\textbf{{K}}}_i-{\textbf{{A}}}_i\right\|_F^2.
\label{eq:44}
\end{equation} %equal 44
\end{lemma}
\begin{proof}
See Appendix~\ref{lemma2}.
%See supplementary material.
\end{proof}

Based on Lemma~\ref{lemma:1}, Lemma~\ref{lemma:2} and Theorem~\ref{theorem:2}, \bihan{we have the equivalence final result as the Theorem~\ref{theorem:3}.}

\begin{theorem}
\label{theorem:3}
Under the condition of the adaptive dictionary ${\textbf{{D}}}_i$ shown in Eq.~\eqref{eq:43}, the proposed RRC model in Eq.~\eqref{eq:6}  is equivalent to the GSRC model in Eq.~\eqref{eq:41}.
\end{theorem}
\begin{proof}
See Appendix~\ref{theorem3}.
%See supplementary material.
\end{proof}

\bihan{The equivalence analysis helps us to bridge the gap between the proposed RRC model and the popular GSR model, thus to provide another angle to interpret the RRC model.
Note that there are many other approaches to construct the group-wise dictionaries $\left \{ \textbf{\emph{D}}_i \right \}$~\cite{29,49,54,61}, and the equivalence of RRC and GSR does not hold in general: The sparse representation models the data by a union of low-dimensional subspaces, while low-rank modeling projects the data onto a unique subspace~\cite{78}.
%Furthermore, with the proposed dictionary construction, the sparse representation problem can be translated into an equivalent rank minimization problem
Thus, the nature of the proposed RRC model is in general different from the sparse residual models such as NCSR~\cite{49}.
%There are extensive researches on the sparsity residual model for image processing and we have witnessed great successes of these models \cite{49,50,51,52,53}.

}

%According to the above analysis, we bridge the gap between the proposed RRC model and GSR model. It is worth noting that the dictionary can be learned in miscellaneous manners and the devised adaptive dictionary learning approach is just one example.  Although the devised adaptive dictionary learning appears to translate the sparse representation into the rank minimization problem, the main difference between sparse representation and the rank minimization models is that sparse representation has a dictionary learning process while the rank minimization problem does not, to the best of our knowledge. This is also the key difference between our RRC model and the NCSR model \cite{49}. There are extensive researches on the sparsity residual model for image processing and we have witnessed great successes of these models \cite{49,50,51,52,53}. Therefore, encouraged by this and since we have proved the equivalence between the proposed RRC model and the GSRC model based on the devised dictionary, we are confident on the feasibility of the RRC model for image processing, which will be further validated by extensive experiments on image restoration tasks, including image denoising and image compression artifacts reduction in the following section.

\begin{figure}[!t]
%\vspace{-4mm}
\centering
\begin{minipage}[b]{1\linewidth}
{\includegraphics[width= 1\textwidth]{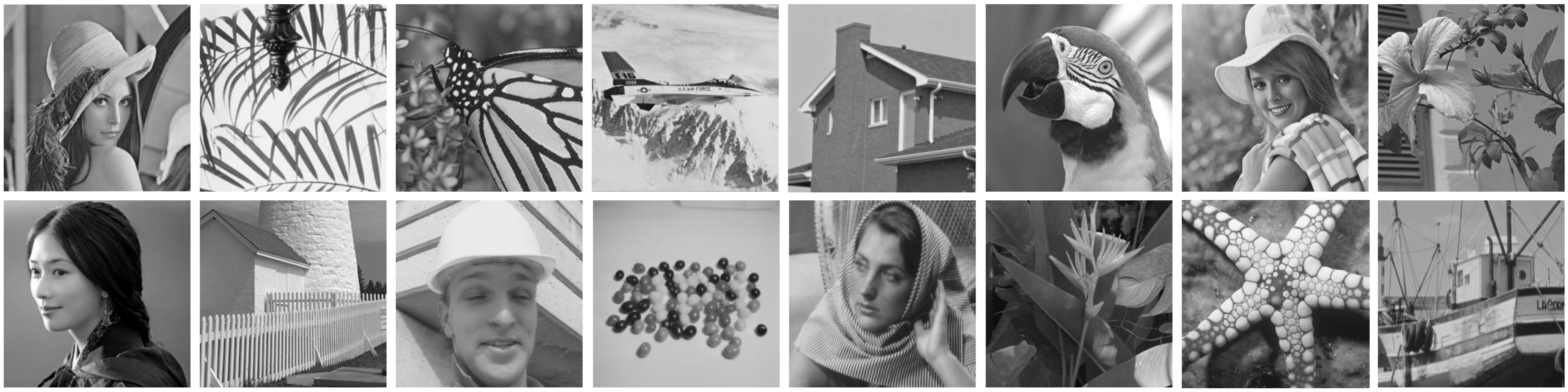}}
\end{minipage}
\vspace{-4mm}
\caption{16 widely used images for image denoising.  \textbf{T}op row, from left to right: Lena, Leaves, Monarch, Airplane, House, Parrot, Elaine, Flower. \textbf{B}ottom row, from left to right: Lin, Fence, Foreman, J. Bean, Barbara, Plants, Starfish, boats.}
\label{fig:5}
\vspace{-4mm}
\end{figure}

\section{Experimental Results}

 \label{sec:6}

In this section, we conduct extensive experiments on image denoising and image compression artifacts reduction to verify the effectiveness of the proposed RRC model. To evaluate the quality of the recovered images, both PSNR and structural similarity (SSIM) \cite{55} metrics are used. The source codes of all competing methods are obtained from the original authors and we used the default parameter settings.  {For color images, we only focus on the restoration of luminance channel (in YCrCb space).} Due to limit space, please enlarge the tables and figures on the screen for better comparison.  We choose the following stop criterion for the proposed RRC based image restoration algorithms,
\begin{equation}
\frac{\|\hat{\textbf{\emph{x}}}^t-\hat{\textbf{\emph{x}}}^{t-1}\|_2^2}{\|\hat{\textbf{\emph{x}}}^{t-1}\|_2^2} <\tau,
\label{eq:45}
\end{equation}%equal 45
where $\tau$ is a small constant. The source code of the proposed RRC method is available at: \url{https://drive.google.com/open?id=18JSSnuF_3x0AvVYq0E1ewErmEynTY8cR}.

\subsection{Image Denoising}

In image denoising, to validate the denoising performance of the proposed RRC model, we compare it with leading denoising methods, including NNM, BM3D \cite{26}, LSSC \cite{27}, EPLL  \cite{56}, Plow \cite{57}, NCSR \cite{49}, GID \cite{58}, PGPD \cite{29}, LINC \cite{59}, aGMM \cite{60} and OGLR \cite{61}. Note that NNM is the baseline algorithm, and the image nonlocal redundancies are used in all competing methods. The parameter settings of the proposed RRC model for image denoising are as follows. The size of each patch $ \sqrt{d} \times \sqrt{d}$ is set to $6 \times 6$, $7 \times 7$, $8 \times 8$ and $9 \times 9$ for $\boldsymbol\sigma_n \leq20$, $20<\boldsymbol\sigma_n \leq50$, $50<\boldsymbol\sigma_n \leq75$ and $75<\boldsymbol\sigma_n \leq100$, respectively. The parameters ($\mu, \rho, c, m, \tau$) are set to (0.1, 0.9, 0.9, 60, 0.001), (0.1, 0.8, 0.9, 60, 0.001), (0.1, 0.8, 0.9, 70, 0.0006), (0.1, 0.8, 1, 80, 0.0006), (0.1, 0.8, 1, 90, 0.0005) and (0.1, 0.8, 1, 100, 0.002) for $\boldsymbol\sigma_n \leq20$, $20<\boldsymbol\sigma_n \leq30$, $30<\boldsymbol\sigma_n \leq40$, $40<\boldsymbol\sigma_n \leq50$, $50<\boldsymbol\sigma_n \leq75$ and $75<\boldsymbol\sigma_n \leq100$, respectively. The searching window for similar patches is set to $L=25$; $\epsilon =0.2, h=40$.

We first evaluate the competing methods on 16 widely used test images, whose scenes are displayed in Fig.~\ref{fig:5}. Zero mean additive white Gaussian noise is added to these test images to generate the noisy observations. We present the denoising results on six noise levels, \ie,  $\boldsymbol\sigma_n$ = \{ 20, 30, 40, 50, 75 and  100\}. The PSNR and SSIM results under these noise levels for all methods are shown in Table~\ref{Tab:1} and Table~\ref{Tab:2} (\textcolor{red}{\textbf{red}} indicates the best and \textcolor{blue}{\textbf{blue}} is the second best performance), respectively.  It is obvious that the proposed RRC significantly outperforms the baseline rank minimization method, \ie, NNM.  Meanwhile, one can observe that the proposed RRC outperforms other competing methods in most cases in terms of PSNR. The average gains of the proposed RRC over NNM, BM3D, LSSC, EPLL, Plow, NCSR, GID, PGPD, LINC, aGMM and OGLR methods are as much as 1.43dB, 0.24dB, 0.29dB, 0.70dB, 0.72dB, 0.40dB, 1.49dB, 0.14dB, 0.25dB, 0.47dB and 0.46dB, respectively. In terms of SSIM, it can be seen that the proposed RRC also achieves higher results than other competing methods. The only exception is when $\boldsymbol\sigma_n$ = 20 for which LINC is slightly higher than the proposed RRC method. Nonetheless, under high noise level $\boldsymbol\sigma_n$ = 100,  the proposed RRC consistently outperforms  other competing methods for all test images ({the only exception is the images $\emph{Flower}$, $\emph{J. Bean}$ and $\emph{Lin}$  for which BM3D, NCSR and LINC respectively, are slightly better than the proposed RRC method}).

\begin{figure}[!t]
%\vspace{-4mm}
\begin{minipage}[b]{1\linewidth}
{\includegraphics[width= 1\textwidth]{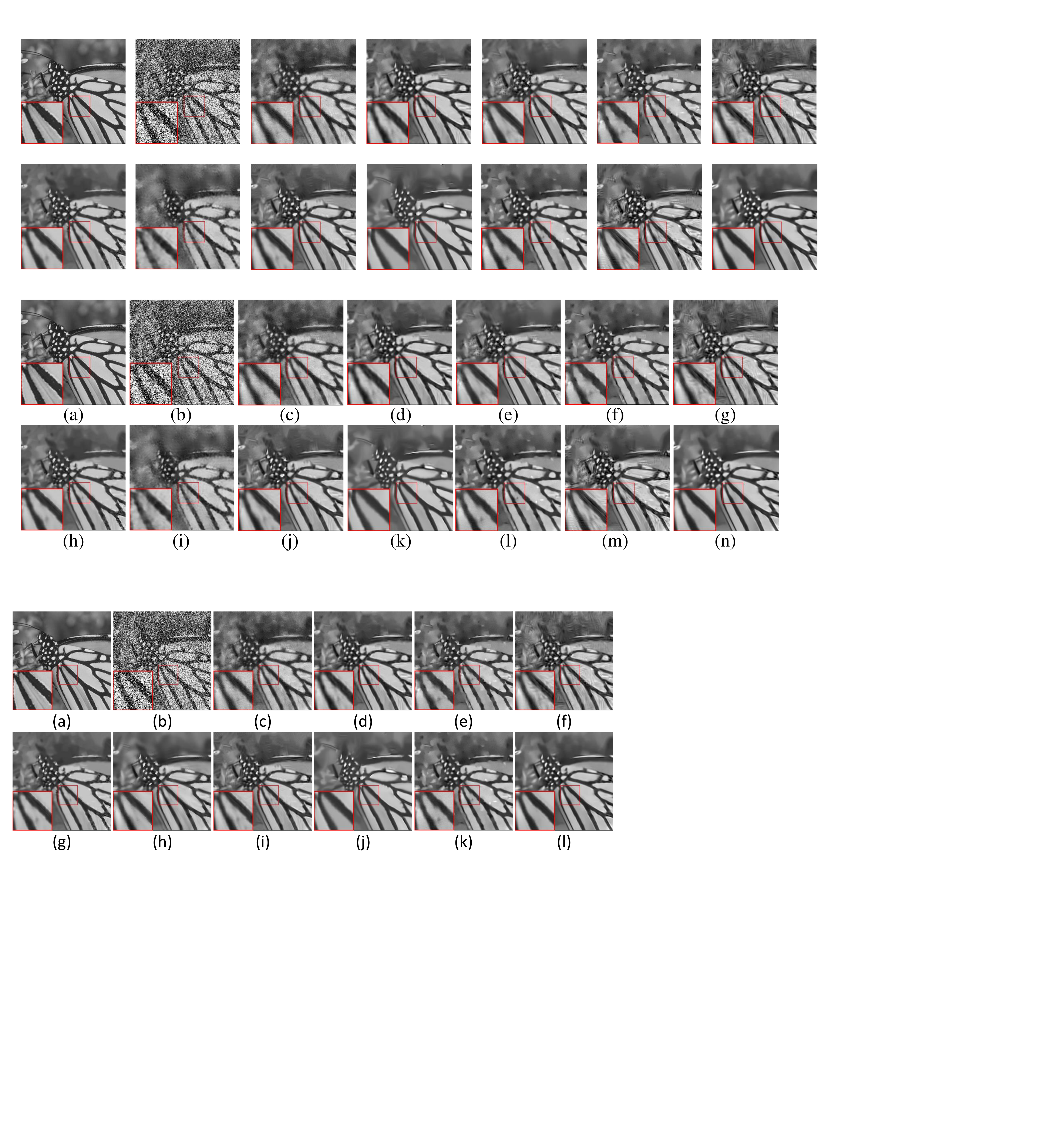}}
\end{minipage}
\vspace{-6mm}
\caption{Denoising results of $\emph{Monarch}$ with $\boldsymbol\sigma_n$ = 100. (a) Original image; (b) Noisy image; (c) NNM (PSNR = 21.03dB, SSIM = 0.5996); (d) BM3D \cite{26} (PSNR = 22.52dB, SSIM = 0.7021); (e) LSSC \cite{27} (PSNR = 22.24dB, SSIM = 0.6999); (f) EPLL \cite{56} (PSNR = 22.24dB, SSIM = 0.6771); (g) Plow \cite{57} (PSNR = 21.83dB, SSIM = 0.6102); (h) NCSR \cite{49} (PSNR = 22.10dB, SSIM = 0.7109); (i) GID \cite{58} (PSNR = 20.73dB, SSIM = 0.6361); (j) PGPD \cite{29} (PSNR = 22.56dB, SSIM = 0.7029); (k) LINC \cite{59} (PSNR = 22.10dB, SSIM = 0.7037); (l) aGMM \cite{60} (PSNR = 22.42dB, SSIM = 0.6823); (m) OGLR \cite{61} (PSNR = 21.87dB, SSIM = 0.6419); (n) RRC (PSNR = \textbf{22.76dB}, SSIM = \textbf{0.7312}).}
\label{fig:6}
\vspace{-2mm}
\end{figure}

\begin{figure}[!t]
%\vspace{-2mm}
\begin{minipage}[b]{1\linewidth}
{\includegraphics[width= 1\textwidth]{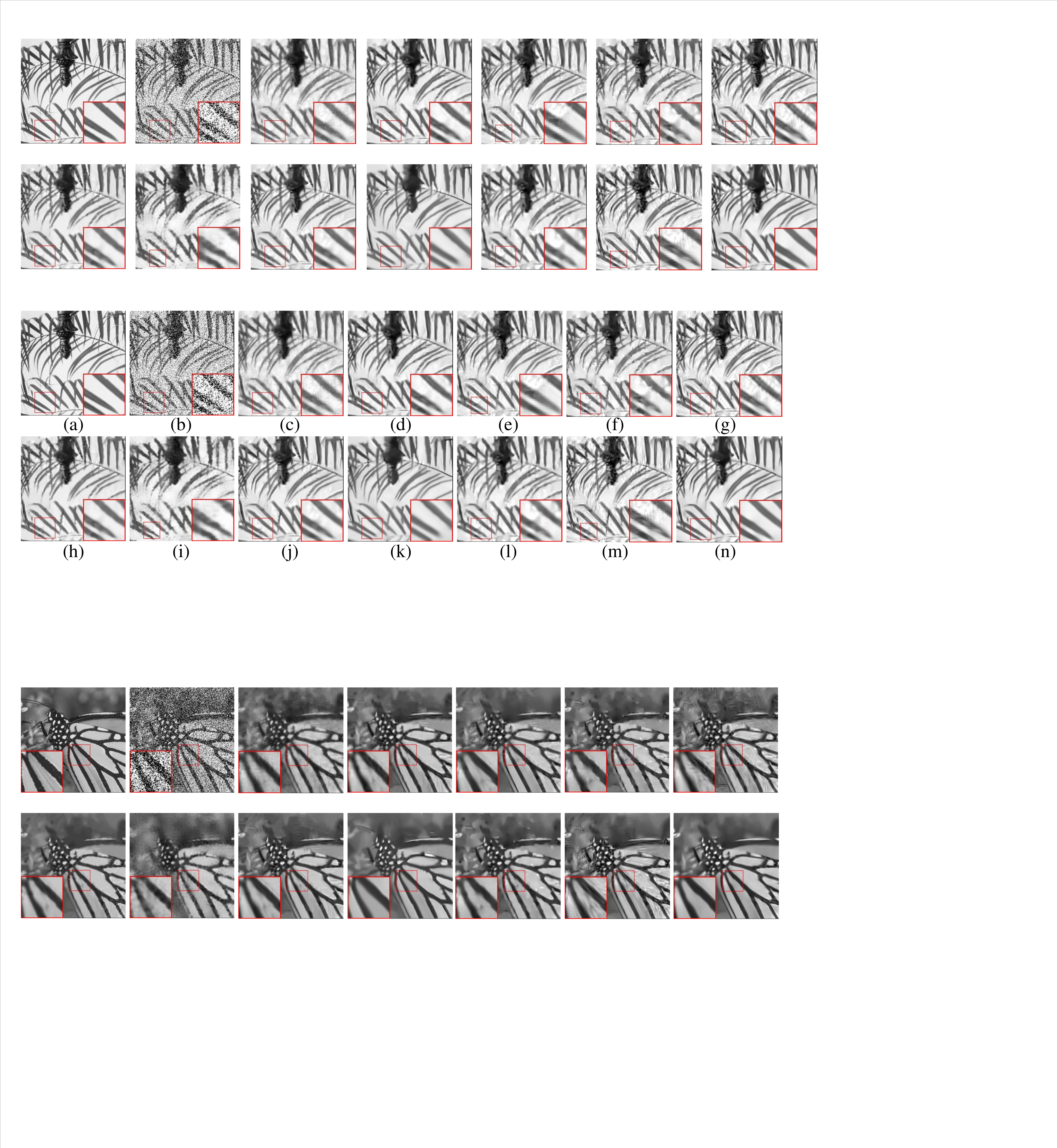}}
\end{minipage}
\vspace{-6mm}
\caption{Denoising results of $\emph{Leaves}$ with $\boldsymbol\sigma_n$ = 100. (a) Original image; (b) Noisy image; (c) NNM (PSNR = 19.57dB, SSIM = 0.6345); (d) BM3D \cite{26} (PSNR = 20.90dB, SSIM = 0.7482); (e) LSSC \cite{27} (PSNR = 20.54dB, SSIM = 0.7242); (f) EPLL \cite{56} (PSNR = 20.26dB, SSIM = 0.7163); (g) Plow \cite{57} (PSNR = 20.43dB, SSIM = 0.6814); (h) NCSR \cite{49} (PSNR = 20.84dB, SSIM = 0.7622); (i) GID \cite{58} (PSNR = 19.13dB, SSIM = 0.6857); (j) PGPD \cite{29} (PSNR = 20.95dB, SSIM = 0.7469); (k) LINC \cite{59} (PSNR = 20.44dB, SSIM = 0.7467); (l) aGMM \cite{60} (PSNR = 20.29dB, SSIM = 0.7106); (m) OGLR \cite{61} (PSNR = 20.28dB, SSIM = 0.6827); (n) RRC (PSNR = \textbf{21.22dB}, SSIM = \textbf{0.7811}).}
\label{fig:7}
\vspace{-4mm}
\end{figure}

The visual comparisons in the case of $\boldsymbol\sigma_n$ = 100 for images $\emph{Monarch}$ and $\emph{Leaves}$ are shown in Fig.~\ref{fig:6} and Fig.~\ref{fig:7}, respectively.  It can be found that NNM, EPLL, Plow, NCSR, GID, PGPD, aGMM and OGLR still suffer from some undesirable visual artifacts, and BM3D, LSSC and LINC tend to over-smooth the image. The proposed RRC approach not only removes most of the visual artifacts, but also preserves large scale sharp edges and small-scale image details.

\begin{table}[!t]
%\vspace{-4mm}
\caption{ {Average PSNR ($\textnormal{d}$B) and SSIM  results of different  denoising algorithms on BSD 200 dataset \cite{62}.}}
%\vspace{-2mm}
\centering
\LARGE
\resizebox{0.48\textwidth}{!}{
\begin{tabular}{|c|c|c|c|c|c|c|c|c|c|c|c|c|c|c|}
\hline
\multicolumn{1}{|c|}{}&\multicolumn{13}{|c|}{\textbf{PSNR Comparison}}\\
\hline
{\textbf{{Noise}}}&{\textbf{{NNM}}}&{\textbf{BM3D}}& {\textbf{LSSC}}&{\textbf{EPLL}}&{\textbf{Plow}}&{\textbf{NCSR}}
&{\textbf{GID}}&{\textbf{PGPD}}&{\textbf{LINC}}&{\textbf{aGMM}}&{\textbf{OGLR}}&{\textbf{WNNM}}&{\textbf{RRC}}\\
\hline
\multirow{1}{*}{20}
&	28.45 	&	29.86 	&	\textcolor{blue}{\textbf{30.02}} 	&	29.96 	&	29.31 	&	29.89 	&	28.87 	&	29.89 	&	29.95 	&	29.49 	&	29.67 &	\textcolor{red}{\textbf{30.11}} 	&	29.98
\\
\hline
\multirow{1}{*}{30}
&	27.24 	&	27.93 	&	\textcolor{blue}{\textbf{28.05}} 	&	28.00 	&	27.58 	&	27.92 	&	27.00 	&	27.96 	&	27.98 	&	27.56 	&	27.84 &	\textcolor{red}{\textbf{28.17}} 	&	28.02
\\
\hline

\multirow{1}{*}{40}
&	26.26 	&	26.58 	&	\textcolor{blue}{\textbf{26.75}} 	&	26.71 	&	26.37 	&	26.58 	&	25.87 	&	26.73 	&	26.68 	&	26.36 	&	26.65 &	\textcolor{red}{\textbf{26.88}} 	&	26.73
\\
\hline

\multirow{1}{*}{50}
&	25.16 	&	25.71 	&	25.80 	&	25.77 	&	25.46 	&	25.65 	&	24.97 	&	\textcolor{blue}{\textbf{25.82}} 	&	25.73 	&	25.31 	&	25.69 &	\textcolor{red}{\textbf{25.96}} 	&	25.81
\\
\hline

\multirow{1}{*}{75}
&	23.36 	&	24.22 	&	24.18 	&	24.18 	&	23.80 	&	24.04 	&	23.37 	&	\textcolor{blue}{\textbf{24.30}} 	&	24.11 	&	23.50 	&	24.16 &	\textcolor{red}{\textbf{24.42}} 	&	24.28
\\
\hline

\multirow{1}{*}{100}
&	21.70 	&	23.21 	&	23.12 	&	23.15 	&	22.66 	&	23.00 	&	22.20 	&	\textcolor{blue}{\textbf{23.29}} 	&	23.02 	&	22.19 	&	22.85 &	\textcolor{red}{\textbf{23.37}} 	&	23.27
\\
\hline

\multirow{1}{*}{\textbf{Average}}
&	25.36 	&	26.25 	&	26.32 	&	26.30 	&	25.86 	&	26.18 	&	25.38 	&	26.33 	&	26.25 	&	25.74 	&26.14 &\textcolor{red}{\textbf{26.49}} 	&	\textcolor{blue}{\textbf{26.35}}
\\
\hline

\multicolumn{1}{|c|}{}&\multicolumn{13}{|c|}{\textbf{SSIM Comparison}}\\
\hline
{\textbf{{Noise}}}&{\textbf{{NNM}}}&{\textbf{BM3D}}& {\textbf{LSSC}}&{\textbf{EPLL}}&{\textbf{Plow}}&{\textbf{NCSR}}
&{\textbf{GID}}&{\textbf{PGPD}}&{\textbf{LINC}}&{\textbf{aGMM}}&{\textbf{OGLR}}&{\textbf{WNNM}}&{\textbf{RRC}}\\
\hline
\multirow{1}{*}{20}
&	0.7896 	&	0.8476 	&	\textcolor{blue}{\textbf{0.8520}} 	&	\textcolor{red}{\textbf{0.8528}} 	&	0.8320 	&	0.8449 	&	0.8111 	&	0.8393 	&	0.8485 	&	0.8329 	&	0.8448 &	0.8481 	&	0.8518
\\
\hline
\multirow{1}{*}{30}
&	0.7058 	&	0.7875 	&	\textcolor{red}{\textbf{0.7936}} 	&	0.7902 	&	0.7765 	&	0.7861 	&	0.7486 	&	0.7803 	&	0.7845 	&	0.7671 	&	0.7852 &	0.7905 	&	\textcolor{blue}{\textbf{0.7926}}
\\
\hline

\multirow{1}{*}{40}
&	0.6804 	&	0.7387 	&	\textcolor{blue}{\textbf{0.7459}} 	&	0.7387 	&	0.7272 	&	0.7337 	&	0.7030 	&	0.7359 	&	0.7339 	&	0.7195 	&	0.7444 &	0.7412 	&	\textcolor{red}{\textbf{0.7471}}
\\
\hline

\multirow{1}{*}{50}
&	0.6201 	&	0.7041 	&	0.7068 	&	0.6963 	&	0.6832 	&	0.6976 	&	0.6627 	&	0.6986 	&	0.6946 	&	0.6732 	&	0.7000 &\textcolor{blue}{\textbf{0.7089}} 	&	\textcolor{red}{\textbf{0.7108}}
\\
\hline

\multirow{1}{*}{75}
&	0.4952 	&	0.6337 	&	0.6364 	&	0.6160 	&	0.5871 	&	0.6320 	&	0.5882 	&	0.6330 	&	0.6252 	&	0.5842 	&  0.6234 &\textcolor{red}{\textbf{0.6446}} 	&	\textcolor{blue}{\textbf{0.6433}}
\\
\hline

\multirow{1}{*}{100}
&	0.4437 	&	0.5814 	&	0.5873 	&	0.5566 	&	0.5157 	&	0.5889 	&	0.5211 	&	0.5810 	&	0.5791 	&	0.5194 	&	0.5528 &\textcolor{blue}{\textbf{0.5949}} 	&	\textcolor{red}{\textbf{0.5986}}
\\
\hline

\multirow{1}{*}{\textbf{Average}}
&	0.6225 	&	0.7155 	&	0.7203 	&	0.7084 	&	0.6870 	&	0.7139 	&	0.6725 	&	0.7114 	&	0.7110 	&	0.6827 	&	0.7084 &\textcolor{blue}{\textbf{0.7214}} 	&	\textcolor{red}{\textbf{0.7240}}
\\
\hline

\end{tabular}}
\label{Tab:3}
\vspace{-2mm}
\end{table}

Now, we compare the proposed RRC with WNNM method \cite{5}, which is a well-known rank minimization method that delivers state-of-the-art denoising results. {The PSNR and SSIM results of WNNM method on 16 widely used test images are shown in the penultimate column of  Table~\ref{Tab:1} and Table~\ref{Tab:2}, respectively.} It can be seen that though the {average} PSNR results of RRC is slightly lower ($<$ 0.11dB) than WNNM, the SSIM results of the proposed RRC is higher than WNNM when the noise level $\boldsymbol\sigma_n> 30$. It is well-known that SSIM \cite{55} often considers the human visual system and leads to more accurate results. The visual comparison of RRC and WNNM with two exemplar images are shown in Fig.~\ref{fig:8} and Fig.~\ref{fig:9}, where we can observe that although PSNR results of the proposed RRC are lower than WNNM, more details are recovered by RRC than WNNM. This phenomenon has been explained in \cite{69}.

{We also compare the proposed RRC with a representative deep learning method: DN-CNN \cite{77}. On average, the proposed RRC cannot achieve better performance than DN-CNN. Nonetheless, if the images have many similar structures and features, such as images $\emph{Barabara}$, $\emph{Fence}$, $\emph{Foreman}$ and  $\emph{House}$, our proposed RRC can outperform DN-CNN, because the nonlocal redundancies are used \cite{25}. An additional advantage of the proposed RRC is that it enjoys the advantage of training free. Due to the page limits, we don't show the detailed comparison results here.}

\begin{figure}[!t]
\vspace{-4mm}
\begin{minipage}[b]{1\linewidth}
{\includegraphics[width= 1\textwidth]{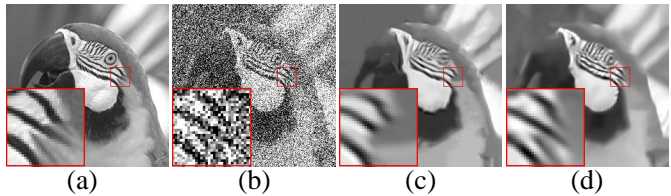}}
\end{minipage}
\vspace{-6mm}
\caption{ {Denoising results of $\emph{Parrot}$ with $\boldsymbol\sigma_n$ = 100. (a) Original {\em Parrot} image; (b) Noisy image;  (c) WNNM \cite{5} (PSNR = \textbf{24.85dB}, SSIM = 0.7533); (d) RRC (PSNR = {24.83dB}, SSIM = \textbf{0.7729}).}}
\label{fig:8}
%\vspace{-2mm}
\end{figure}

\begin{figure}[!t]
\vspace{-2mm}
\begin{minipage}[b]{1\linewidth}
{\includegraphics[width= 1\textwidth]{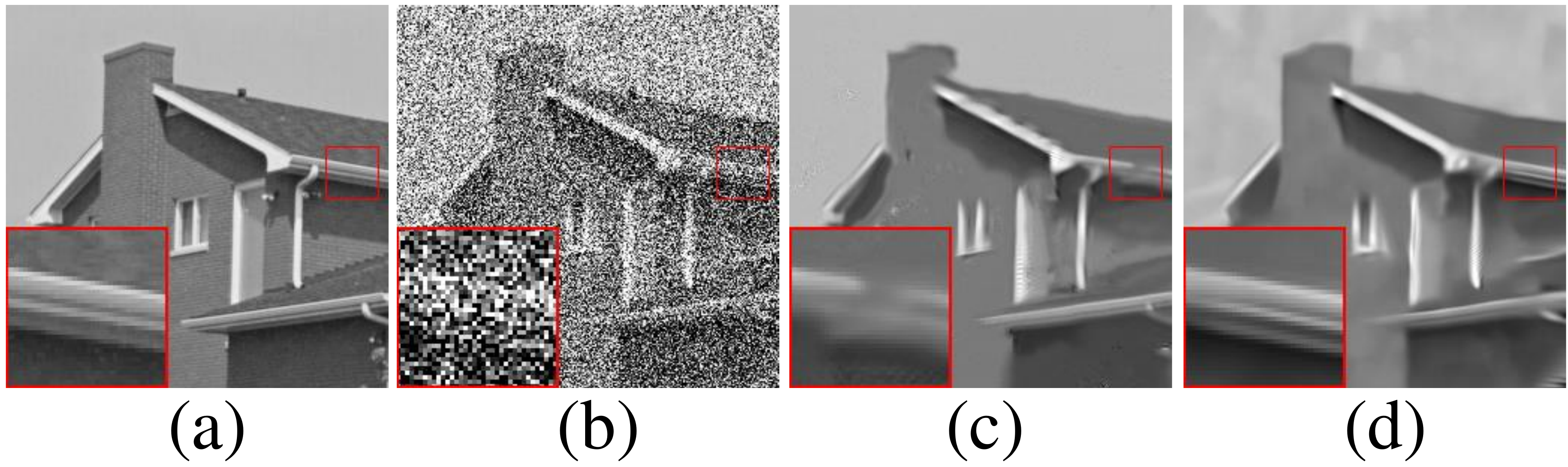}}
\end{minipage}
\vspace{-6mm}
\caption{ {Denoising results of $\emph{House}$ with $\boldsymbol\sigma_n$ = 100. (a) Original {\em House} image; (b) Noisy image;  (c) WNNM \cite{5} (PSNR = \textbf{26.79dB}, SSIM = 0.7531); (d) RRC (PSNR = {26.38dB}, SSIM = \textbf{0.7655}).}}
\label{fig:9}
\vspace{-4mm}
\end{figure}

{Furthermore, we comprehensively evaluate our proposed RRC method on 200 test images from the  Berkeley Segmentation dataset (BSD) \cite{62}\footnote{The denoising results of our proposed RRC method for BSD 200 dataset \cite{62} is available at: \url{https://drive.google.com/open?id=1qiuqPEAza1mF--9nR5nhXFtS8Z2kHUcb}.}. Table~\ref{Tab:3} lists the average PSNR and SSIM comparison results among thirteen competing methods at six noise levels ($\boldsymbol\sigma_n$ = 20, 30, 40, 50, 75 and 100). One can observe that our proposed RRC approach achieves the second best PSNR result, which is only falling behind WNNM by less than 0.14dB.  Nonetheless, our proposed RRC method obtains the best SSIM results on average.}

{Finally, we consider the proposed RRC model to real image denoising. Due to the fact that the noise level of real noisy images is unknown, and thus the noise level in the image is required to estimate through some noise estimation methods. In this paper, we adopt the scheme proposed in \cite{73} to estimate the noise level. Fig.~\ref{fig:17} shows the denoising results of two real images with more complex noise by our proposed RRC model. It can be seen that the proposed RRC method not only recovers visual pleasant results, but also preserves fine image details. Therefore, these results demonstrate the feasibility of our proposed RRC method for practical image denoising applications.}

\begin{figure}[!t]
\vspace{-4mm}
\begin{minipage}[b]{1\linewidth}
{\includegraphics[width= 1\textwidth]{fig17_1.pdf}}
\end{minipage}
\vspace{-6mm}
\caption{{Denoising results on two real images by our proposed RRC method.}}
\label{fig:17}
\vspace{-4mm}
\end{figure}

It is well-known that image denoising is an ideal test bench to measure different statistical image models.  Obviously, these experiments have verified that the proposed RRC is a promising image model.

The proposed RRC is a traditional model-based optimization algorithm. Here we evaluate the average running time of the proposed RRC model for image denoising by using 16 widely used images with different noise levels. The proposed RRC requires about 5$\sim$6 minutes for an image on an Intel (R) Core (TM) i3-4150 with 3.56Hz CPU and 4GB memory PC under the Matlab 2015b environment. The running time of the proposed RRC for image denoising is faster than LSSC, NCSR and GID methods.%, and about twice long as the baseline NNM method.

\begin{table*}[!htbp]
\vspace{-2mm}
\caption{{Average PSNR ($\textnormal{d}$B) (top entry in each cell) and SSIM (bottom entry) by JPEG, SA-DCT \cite{42}, PC-LRM \cite{63}, ANCE \cite{64}, DicTV \cite{46}, BM3D \cite{26}, WNNM \cite{5}, CONCOLOR \cite{44}, SSR-QC \cite{45}, LERaG \cite{65} and RRC  on the BSD 100 \cite{62}, Classic5 and LIVE1 \cite{75} Dataset.}}
%\vspace{-2mm}
\tiny
\centering
\resizebox{1\textwidth}{!}
{
\begin{tabular}{|c|c|c|c|c|c|c|c|c|c|c|c|}
\hline
\multicolumn{12}{|c|}{\textbf{BSD 100 Dataset } \cite{62}}\\
\hline
\multirow{2}{*}{\textbf{{QF}}}&\multirow{2}{*}{\textbf{{JPEG}}}
&{\textbf{{SA-}}}&{\textbf{{PC-}}}&\multirow{2}{*}{\textbf{{ANCE}}}&\multirow{2}{*}{\textbf{{DicTV}}}
&\multirow{2}{*}{\textbf{{BM3D}}}&\multirow{2}{*}{\textbf{{WNNM}}}&{\textbf{{CON}}}&{\textbf{{SSR-}}}
&\multirow{2}{*}{\textbf{{LERaG}}} &\multirow{2}{*}{\textbf{{RRC}}}\\
& & {\textbf{DCT}} &{\textbf{LRM}} & & & & &{\textbf{COLOR}} &{\textbf{QC}} &&  \\
\hline
\multirow{2}{*}{\textbf{10}}
&	27.59	&	28.48	&	28.49	&	28.51	&	28.17	&	28.46	&	28.48	&	28.54	&	28.45	&	28.57	&	\textbf{28.59}
\\
\cline{2-12}
&	0.7688	&	0.7896	&	0.7847	&	0.7916	&	0.7825	&	0.7924	&	0.7828	&	0.7907	&	0.7874	&	0.7939	&	\textbf{0.8015}
\\
\hline
\multirow{2}{*}{\textbf{20}}
&	29.97	&	30.73	&	30.78	&	30.80	&	30.63	&	30.75	&	30.79	&	30.86	&	30.78	&	30.86	&	\textbf{30.98}
\\
\cline{2-12}
&	0.8530	&	0.8650	&	0.8620	&	0.8663	&	0.8574	&	0.8681	&	0.8611	&	0.8647	&	0.8633	&	0.8661	&	\textbf{0.8745}
\\
\hline
\multirow{2}{*}{\textbf{30}}
&	31.37	&	32.07	&	32.16	&	32.18	&	31.82	&	32.09	&	32.17	&	32.29	&	32.18	&	32.27	&	\textbf{32.41}
\\
\cline{2-12}
&	0.8886	&	0.8984	&	0.8967	&	0.8990	&	0.8849	&	0.9004	&	0.8963	&	0.8983	&	0.8975	&	0.8985	&	\textbf{0.9057}
\\
\hline
\multirow{2}{*}{\textbf{40}}
&	32.36	&	33.01	&	33.12	&	33.16	&	32.74	&	33.04	&	33.13	&	33.31	&	33.15	&	33.26	&	\textbf{33.44}
\\
\cline{2-12}
&	0.9084	&	0.9168	&	0.9159	&	0.9174	&	0.8995	&	0.9182	&	0.9158	&	0.9172	&	0.9149	&	0.9168	&	\textbf{0.9231}
\\
\hline
\multirow{2}{*}{\textbf{50}}
&	33.17	&	33.79	&	33.91	&	33.97	&	33.39	&	33.82	&	33.93	&	34.16	&	33.98	&	34.08	&	\textbf{34.28}
\\
\cline{2-12}
&	0.9223	&	0.9296	&	0.9292	&	0.9303	&	0.9098	&	0.9307	&	0.9291	&	0.9304	&	0.9287	&	0.9296	&	\textbf{0.9352}
\\
\hline
\multirow{2}{*}{\textbf{60}}
&	34.02	&	34.60	&	34.73	&	34.83	&	34.12	&	34.63	&	34.75	&	35.02	&	34.83	&	34.92	&	\textbf{35.15}
\\
\cline{2-12}
&	0.9342	&	0.9406	&	0.9404	&	0.9414	&	0.9202	&	0.9414	&	0.9405	&	0.9418	&	0.9404	&	0.9407	&	\textbf{0.9455}
\\
\hline
\multirow{2}{*}{\textbf{70}}
&	35.19	&	35.73	&	35.87	&	36.02	&	35.05	&	35.76	&	35.88	&	36.23	&	36.00	&	36.06	&	\textbf{36.35}
\\
\cline{2-12}
&	0.9477	&	0.9529	&	0.9530	&	0.9541	&	0.9319	&	0.9535	&	0.9531	&	0.9546	&	0.9535	&	0.9532	&	\textbf{0.9572}
\\
\hline
\multirow{2}{*}{\textbf{80}}
&	36.97	&	37.46	&	37.60	&	37.81	&	36.29	&	37.50	&	37.62	&	38.05	&	37.75	&	37.76	&	\textbf{38.16}
\\
\cline{2-12}
&	0.9627	&	0.9667	&	0.9670	&	0.9680	&	0.9446	&	0.9671	&	0.9671	&	0.9695	&	0.9676	&	0.9669	&	\textbf{0.9702}
\\
\hline
\multirow{2}{*}{\textbf{90}}
&	40.61	&	41.00	&	41.15	&	41.45	&	38.11	&	41.07	&	41.16	&	41.69	&	41.27	&	41.11	&	\textbf{41.80}
\\
\cline{2-12}
&	0.9816	&	0.9836	&	0.9838	&	0.9848	&	0.9598	&	0.9838	&	0.9838	&	0.9850	&	0.9843	&	0.9836	&	\textbf{0.9857}
\\
\hline
\multirow{2}{*}{\textbf{Average}}
&	33.47 	&	34.10 	&	34.20 	&	34.30 	&	33.37 	&	34.12 	&	34.21 	&	34.46 	&	34.27 	&	34.32 	&	\textbf{34.57}
\\
\cline{2-12}
&	0.9075 	&	0.9159 	&	0.9147 	&	0.9170 	&	0.8990 	&	0.9173 	&	0.9144 	&	0.9169 	&	0.9153 	&	0.9166 	&	\textbf{0.9221}
\\
\hline

\multicolumn{12}{|c|}{\textbf{Classic5 Dataset} (image size: 256$\times$ 256)}\\
\hline
\multirow{2}{*}{\textbf{{QF}}}&\multirow{2}{*}{\textbf{{JPEG}}}
&{\textbf{{SA-}}}&{\textbf{{PC-}}}&\multirow{2}{*}{\textbf{{ANCE}}}&\multirow{2}{*}{\textbf{{DicTV}}}
&\multirow{2}{*}{\textbf{{BM3D}}}&\multirow{2}{*}{\textbf{{WNNM}}}&{\textbf{{CON}}}&{\textbf{{SSR-}}}
&\multirow{2}{*}{\textbf{{LERaG}}} &\multirow{2}{*}{\textbf{{RRC}}}\\
& & {\textbf{DCT}} &{\textbf{LRM}} & & & & &{\textbf{COLOR}} &{\textbf{QC}} &&  \\
\hline
\multirow{2}{*}{\textbf{10}}
&	27.57 	&	28.72 	&	28.79 	&	28.77 	&	28.45 	&	28.69 	&	28.78 	&	28.89 	&	28.83 	&	28.73 	&	\textbf{28.95}
\\
\cline{2-12}
&	0.7715 	&	0.8060 	&	0.8043 	&	0.8081 	&	0.8053 	&	0.8087 	&	0.8033 	&	0.8123 	&	0.8094 	&	0.8143 	&	\textbf{0.8202}
\\
\hline
\multirow{2}{*}{\textbf{20}}
&	29.90 	&	30.89 	&	30.98 	&	30.96 	&	30.73 	&	30.87 	&	30.98 	&	31.11 	&	31.07 	&	30.91 	&	\textbf{31.17}
\\
\cline{2-12}
&	0.8519 	&	0.8728 	&	0.8723 	&	0.8730 	&	0.8665 	&	0.8753 	&	0.8714 	&	0.8751 	&	0.8740 	&	0.8734 	&	\textbf{0.8818}
\\
\hline
\multirow{2}{*}{\textbf{30}}
&	31.21 	&	32.09 	&	32.21 	&	32.22 	&	31.92 	&	32.07 	&	32.21 	&	32.42 	&	32.34 	&	32.22 	&	\textbf{32.49}
\\
\cline{2-12}
&	0.8844 	&	0.9002 	&	0.9003 	&	0.9002 	&	0.8891 	&	0.9018 	&	0.8998 	&	0.9019 	&	0.9017 	&	0.9002 	&	\textbf{0.9071}
\\
\hline
\multirow{2}{*}{\textbf{40}}
&	32.14 	&	32.96 	&	33.09 	&	33.16 	&	32.77 	&	32.94 	&	33.10 	&	33.41 	&	33.30 	&	33.21 	&	\textbf{33.46}
\\
\cline{2-12}
&	0.9036 	&	0.9168 	&	0.9170 	&	0.9172 	&	0.9039 	&	0.9178 	&	0.9167 	&	0.9191 	&	0.9180 	&	0.9176 	&	\textbf{0.9231}
\\
\hline
\multirow{2}{*}{\textbf{50}}
&	32.93 	&	33.71 	&	33.86 	&	34.00 	&	33.50 	&	33.69 	&	33.86 	&	34.27 	&	34.13 	&	34.02 	&	\textbf{34.31}
\\
\cline{2-12}
&	0.9181 	&	0.9291 	&	0.9295 	&	0.9301 	&	0.9154 	&	0.9298 	&	0.9293 	&	0.9321 	&	0.9313 	&	0.9302 	&	\textbf{0.9353}
\\
\hline
\multirow{2}{*}{\textbf{60}}
&	33.77 	&	34.50 	&	34.66 	&	34.86 	&	34.25 	&	34.50 	&	34.67 	&	35.17 	&	34.97 	&	34.88 	&	\textbf{35.20}
\\
\cline{2-12}
&	0.9304 	&	0.9398 	&	0.9402 	&	0.9410 	&	0.9255 	&	0.9403 	&	0.9401 	&	0.9434 	&	0.9425 	&	0.9412 	&	\textbf{0.9459}
\\
\hline
\multirow{2}{*}{\textbf{70}}
&	34.95 	&	35.61 	&	35.77 	&	35.99 	&	35.17 	&	35.61 	&	35.78 	&	36.38 	&	36.12 	&	36.04 	&	\textbf{36.42}
\\
\cline{2-12}
&	0.9447 	&	0.9517 	&	0.9521 	&	0.9531 	&	0.9365 	&	0.9521 	&	0.9520 	&	0.9557 	&	0.9547 	&	0.9531 	&	\textbf{0.9575}
\\
\hline
\multirow{2}{*}{\textbf{80}}
&	36.70 	&	37.30 	&	37.47 	&	37.75 	&	36.43 	&	37.31 	&	37.47 	&	38.19 	&	37.81 	&	37.76 	&	\textbf{38.21}
\\
\cline{2-12}
&	0.9602 	&	0.9651 	&	0.9655 	&	0.9668 	&	0.9488 	&	0.9654 	&	0.9655 	&	0.9690 	&	0.9678 	&	0.9666 	&	\textbf{0.9700}
\\
\hline
\multirow{2}{*}{\textbf{90}}
&	40.33 	&	40.79 	&	40.95 	&	41.25 	&	38.23 	&	40.82 	&	40.96 	&	41.71 	&	41.19 	&	41.01 	&	\textbf{41.76}
\\
\cline{2-12}
&	0.9798 	&	0.9820 	&	0.9822 	&	0.9833 	&	0.9630 	&	0.9821 	&	0.9823 	&	0.9842 	&	0.9833 	&	0.9824 	&	\textbf{0.9848}
\\
\hline
\multirow{2}{*}{\textbf{Average}}
&	33.28 	&	34.06 	&	34.20 	&	34.33 	&	33.49 	&	34.06 	&	34.20 	&	34.62 	&	34.42 	&	34.31 	&	\textbf{34.66}
\\
\cline{2-12}
&	0.9050 	&	0.9182 	&	0.9182 	&	0.9192 	&	0.9060 	&	0.9193 	&	0.9178 	&	0.9214 	&	0.9203 	&	0.9199 	&	\textbf{0.9251}
\\
\hline

\multicolumn{12}{|c|}{\textbf{LIVE1 Dataset } \cite{75} (image size: 256$\times$ 256)}\\
\hline
\multirow{2}{*}{\textbf{{QF}}}&\multirow{2}{*}{\textbf{{JPEG}}}
&{\textbf{{SA-}}}&{\textbf{{PC-}}}&\multirow{2}{*}{\textbf{{ANCE}}}&\multirow{2}{*}{\textbf{{DicTV}}}
&\multirow{2}{*}{\textbf{{BM3D}}}&\multirow{2}{*}{\textbf{{WNNM}}}&{\textbf{{CON}}}&{\textbf{{SSR-}}}
&\multirow{2}{*}{\textbf{{LERaG}}} &\multirow{2}{*}{\textbf{{RRC}}}\\
& & {\textbf{DCT}} &{\textbf{LRM}} & & & & &{\textbf{COLOR}} &{\textbf{QC}} &&  \\
\hline
\multirow{2}{*}{\textbf{10}}
&	26.37 	&	27.23 	&	27.24 	&	27.24 	&	27.02 	&	27.16 	&	27.25 	&	27.33 	&	27.26 	&	\textbf{27.39} 	&	27.36
\\
\cline{2-12}
&	0.7611 	&	0.7869 	&	0.7835 	&	0.7879 	&	0.7872 	&	0.7877 	&	0.7824 	&	0.7888 	&	0.7859 	&	0.7916 	&	\textbf{0.7988}
\\
\hline
\multirow{2}{*}{\textbf{20}}
&	28.55 	&	29.24 	&	29.28 	&	29.29 	&	29.11 	&	29.21 	&	29.29 	&	29.41 	&	29.33 	&	\textbf{29.46} 	&	\textbf{29.46}
\\
\cline{2-12}
&	0.8423 	&	0.8571 	&	0.8550 	&	0.8585 	&	0.8500 	&	0.8591 	&	0.8542 	&	0.8588 	&	0.8576 	&	0.8603 	&	\textbf{0.8665}
\\
\hline
\multirow{2}{*}{\textbf{30}}
&	29.86 	&	30.48 	&	30.54 	&	30.57 	&	30.36 	&	30.45 	&	30.55 	&	30.74 	&	30.60 	&	30.77 	&	\textbf{30.78}
\\
\cline{2-12}
&	0.8791 	&	0.8903 	&	0.8892 	&	0.8913 	&	0.8793 	&	0.8917 	&	0.8888 	&	0.8922 	&	0.8913 	&	0.8924 	&	\textbf{0.8981}
\\
\hline
\multirow{2}{*}{\textbf{40}}
&	30.80 	&	31.37 	&	31.45 	&	31.51 	&	31.23 	&	31.35 	&	31.46 	&	31.71 	&	31.57 	&	31.73 	&	\textbf{31.76}
\\
\cline{2-12}
&	0.8998 	&	0.9093 	&	0.9089 	&	0.9102 	&	0.8952 	&	0.9103 	&	0.9087 	&	0.9115 	&	0.9099 	&	0.9109 	&	\textbf{0.9160}
\\
\hline
\multirow{2}{*}{\textbf{50}}
&	31.60 	&	32.14 	&	32.23 	&	32.31 	&	31.97 	&	32.12 	&	32.24 	&	32.54 	&	32.36 	&	32.49 	&	\textbf{32.58}
\\
\cline{2-12}
&	0.9144 	&	0.9227 	&	0.9227 	&	0.9235 	&	0.9068 	&	0.9234 	&	0.9226 	&	0.9251 	&	0.9237 	&	0.9234 	&	\textbf{0.9285}
\\
\hline
\multirow{2}{*}{\textbf{60}}
&	32.44 	&	32.96 	&	33.05 	&	33.15 	&	32.75 	&	32.94 	&	33.06 	&	33.41 	&	33.20 	&	33.34 	&	\textbf{33.45}
\\
\cline{2-12}
&	0.9270 	&	0.9345 	&	0.9347 	&	0.9352 	&	0.9174 	&	0.9349 	&	0.9347 	&	0.9370 	&	0.9356 	&	0.9350 	&	\textbf{0.9396}
\\
\hline
\multirow{2}{*}{\textbf{70}}
&	33.64 	&	34.13 	&	34.23 	&	34.36 	&	33.80 	&	34.11 	&	34.24 	&	34.65 	&	34.37 	&	34.53 	&	\textbf{34.69}
\\
\cline{2-12}
&	0.9417 	&	0.9480 	&	0.9484 	&	0.9489 	&	0.9301 	&	0.9482 	&	0.9484 	&	0.9506 	&	0.9493 	&	0.9754 	&	\textbf{0.9524}
\\
\hline
\multirow{2}{*}{\textbf{80}}
&	35.51 	&	35.96 	&	36.07 	&	36.22 	&	35.30 	&	35.95 	&	36.07 	&	36.55 	&	36.21 	&	36.33 	&	\textbf{36.60}
\\
\cline{2-12}
&	0.9585 	&	0.9632 	&	0.9637 	&	0.9643 	&	0.9442 	&	0.9633 	&	0.9637 	&	0.9657 	&	0.9645 	&	0.9634 	&	\textbf{0.9667}
\\
\hline
\multirow{2}{*}{\textbf{90}}
&	39.43 	&	39.79 	&	39.92 	&	40.02 	&	37.47 	&	39.81 	&	39.91 	&	40.47 	&	40.04 	&	39.86 	&	\textbf{40.53}
\\
\cline{2-12}
&	0.9797 	&	0.9820 	&	0.9823 	&	0.9829 	&	0.9607 	&	0.9820 	&	0.9823 	&	0.9837 	&	0.9829 	&	0.9816 	&	\textbf{0.9841}
\\
\hline
\multirow{2}{*}{\textbf{Average}}
&	32.02 	&	32.59 	&	32.67 	&	32.74 	&	32.11 	&	32.57 	&	32.67 	&	32.98 	&	32.77 	&	32.88 	&	\textbf{33.02}
\\
\cline{2-12}
&	0.9004 	&	0.9104 	&	0.9098 	&	0.9114 	&	0.8968 	&	0.9112 	&	0.9095 	&	0.9126 	&	0.9112 	&	0.9149 	&	\textbf{0.9167}
\\
\hline

\end{tabular}}
\label{Tab:4}
\vspace{-2mm}
\end{table*}

%\vspace{-2mm}
\subsection{Image Compression Artifacts Reduction}

{In this subsection, we verify the proposed algorithm to restore JPEG-compressed images on three widely used dataset, including BSD 100 \cite{62}\footnote{The compression artifacts reduction results of our proposed RRC method for BSD 100 dataset \cite{62} is available at: \url{https://drive.google.com/open?id=1EphUBFVnEHO0Xx35y9-Cc90poHTHB0r2}.}, Classic5 and LIVE1 \cite{75}.} We compare our algorithm with advanced image deblocking methods, including SA-DCT \cite{42}, PC-LRM \cite{63}, ANCE \cite{64}, DicTV \cite{46}, CONCOLOR \cite{44}, SSR-QC \cite{45}, LERaG \cite{65} and two representative image denoising methods, \ie, BM3D \cite{26} and WNNM \cite{5}. Note that PC-LRM, WNNM, COCOLOR, SSR-QC and LERaG also exploited low-rank priors and achieved the state-of-the-art image deblocking or denoising results.  The parameters of our proposed algorithm for image compression artifacts reduction are as follows. The size of each patch $\sqrt{d} \times \sqrt{d}$ is  {set to} $7 \times 7$. The searching window for similar patches is {set to} $L=25$. The maximum iteration number is set to $T=20$; $h=40$, $\epsilon=0.2$, $m=60$, $\rho=5$ and $w=0.2$. The parameters ($\eta$, $c$, $\tau$) are set to (0.3, 0.9, 0.0007), (0.2, 1.3, 0.0005), (0.2, 1.3, 0.0003) and (0.2, 1.5, 0.0003) for $QF \leq 10$, $10<QF \leq 20$, $20<QF \leq 30$ and $QF > 30$, respectively.

\begin{figure}[!t]
\centering
\vspace{-2mm}
\begin{minipage}[b]{1\linewidth}
{\includegraphics[width= 1\textwidth]{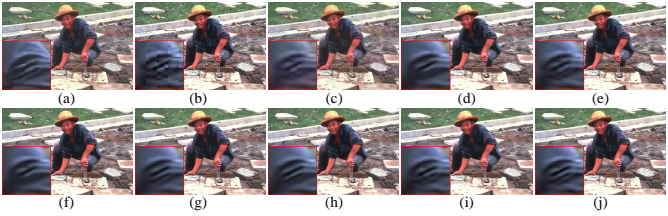}}
\end{minipage}
\vspace{-6mm}
\caption{{Visual comparison results of image $\emph{85048}$ at  QF = 10. (a) Original image; (b) JPEG compressed image (PSNR = 28.00dB, SSIM = 0.7551); (c) SA-DCT \cite{42} (PSNR = 28.97dB, SSIM = 0.7789); (d) PC-LRM \cite{63} (PSNR = 28.97dB, SSIM = 0.7739); (e) ANCE \cite{64} (PSNR = 29.05dB, SSIM = 0.7852);  (f) WNNM \cite{5} (PSNR = 28.90dB, SSIM = 0.7684); (g) CONCOLOR \cite{44} (PSNR = 28.97dB, SSIM = 0.7833); (h) SSR-QC \cite{45} (PSNR = 28.92dB, SSIM = 0.7810);  (i) LERaG \cite{65} (PSNR = 28.97dB, SSIM = 0.7865); (j) RRC (PSNR = \textbf{29.14dB}, SSIM = \textbf{0.7997}).}}
\label{fig:10}
\vspace{-2mm}
\end{figure}

\begin{figure}[!t]	
%\vspace{-2mm}
	\centering
%	\vspace{-2mm}
\begin{minipage}[b]{1\linewidth}
{\includegraphics[width= 1\textwidth]{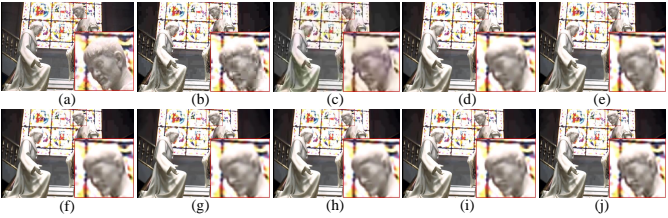}}
\end{minipage}
	\vspace{-6mm}
	\caption{{Visual comparison results of image $\emph{24077}$ at  QF = 10. (a) Original image; (b) JPEG compressed image (PSNR = 27.83dB, SSIM = 0.8897); (c) SA-DCT \cite{42} (PSNR = 28.99dB, SSIM = 0.9029); (d) PC-LRM \cite{63} (PSNR = 29.26dB, SSIM = 0.9064); (e) ANCE \cite{64} (PSNR = 28.98dB, SSIM = 0.9022);  (f) WNNM \cite{5} (PSNR = 29.30dB, SSIM = 0.9068); (g) CONCOLOR \cite{44} (PSNR = 29.23dB, SSIM = 0.9100); (h) SSR-QC \cite{45} (PSNR = 29.17dB, SSIM = 0.9070);  (i) LERaG \cite{65} (PSNR = 29.15dB, SSIM = 0.9060); (j) RRC (PSNR = \textbf{29.43dB}, SSIM = \textbf{0.9119}).}}
	\label{fig:11}
\vspace{-4mm}
\end{figure}

We comprehensively evaluate all competing methods on {these three dataset} at each QF.  Nine JPEG qualities are evaluated, \ie, QF = \{10, 20, 30, 40, 50, 60, 70, 80 and 90\}. The average PSNR and SSIM comparisons are presented in Table~\ref{Tab:4}, with the best results highlighted in bold. It can be seen  that our proposed RRC  consistently outperforms all competing methods on different JPEG QFs.  {The only exception is the LIVE1 dataset at QF = 10 and 20 for which  LERaG  is slightly higher than our proposed method.}   {Based on these three dataset, the average gains of our RRC  over JPEG, SA-DCT, PC-LRM, ANCE, DicTV, BM3D, WNNM, CONCOLOR, SSR-QC and LERaG methods are  \{1.16dB, 0.50dB, 0.40dB, 0.30dB, 1.09dB, 0.50dB, 0.39dB, 0.07dB, 0.27dB and 0.25dB\}  in PSNR and \{0.0170, 0.0065, 0.0071, 0.0054, 0.0207, 0.0054, 0.0074, 0.0043, 0.0057 and 0.0042\} in SSIM, respectively.}

The visual comparisons of images $\emph{85048}$ and $\emph{24077}$  {on BSD 100 dataset} with QF = 10 are shown in Fig.~\ref{fig:10} and Fig.~\ref{fig:11}, respectively, where we compare five typical low-rank based methods (\ie, PC-LRM, WNNM, COCOLOR, SSR-QC and LERaG) and two well-known image compression reduction methods (\ie, SA-DCT and ANCE) with the proposed RRC method. One can observe that the blocking artifacts are obvious in the image decoded directly by the standard JPEG. For image $\emph{85048}$, it can be observed that some blocking artifacts are still visible in SA-DCT, ANCE and LERaG methods, while PC-LRM, WNNM, COCOLOR and SSR-QC methods generate over-smooth effect. For image $\emph{24077}$, though SA-DCT, PC-LRM, ANCE, WNNM, CONCOLOR, SSR-QC and LERaG methods can suppress the blocking artifacts effectively, they often over-smooth the image. Our proposed RRC method not only removes blocking or ringing artifacts across the image, but also preserves large-scale sharp edges and small-scale fine image details. Obviously, compared with these typical low-rank based methods, the proposed RRC method can achieve higher performance. Therefore, these results demonstrate the effectiveness and superiority of our proposed RRC model.

\begin{figure}[!t]	
\vspace{-2mm}
	\centering
	%\vspace{-2mm}
\begin{minipage}[b]{1\linewidth}
{\includegraphics[width= 1\textwidth]{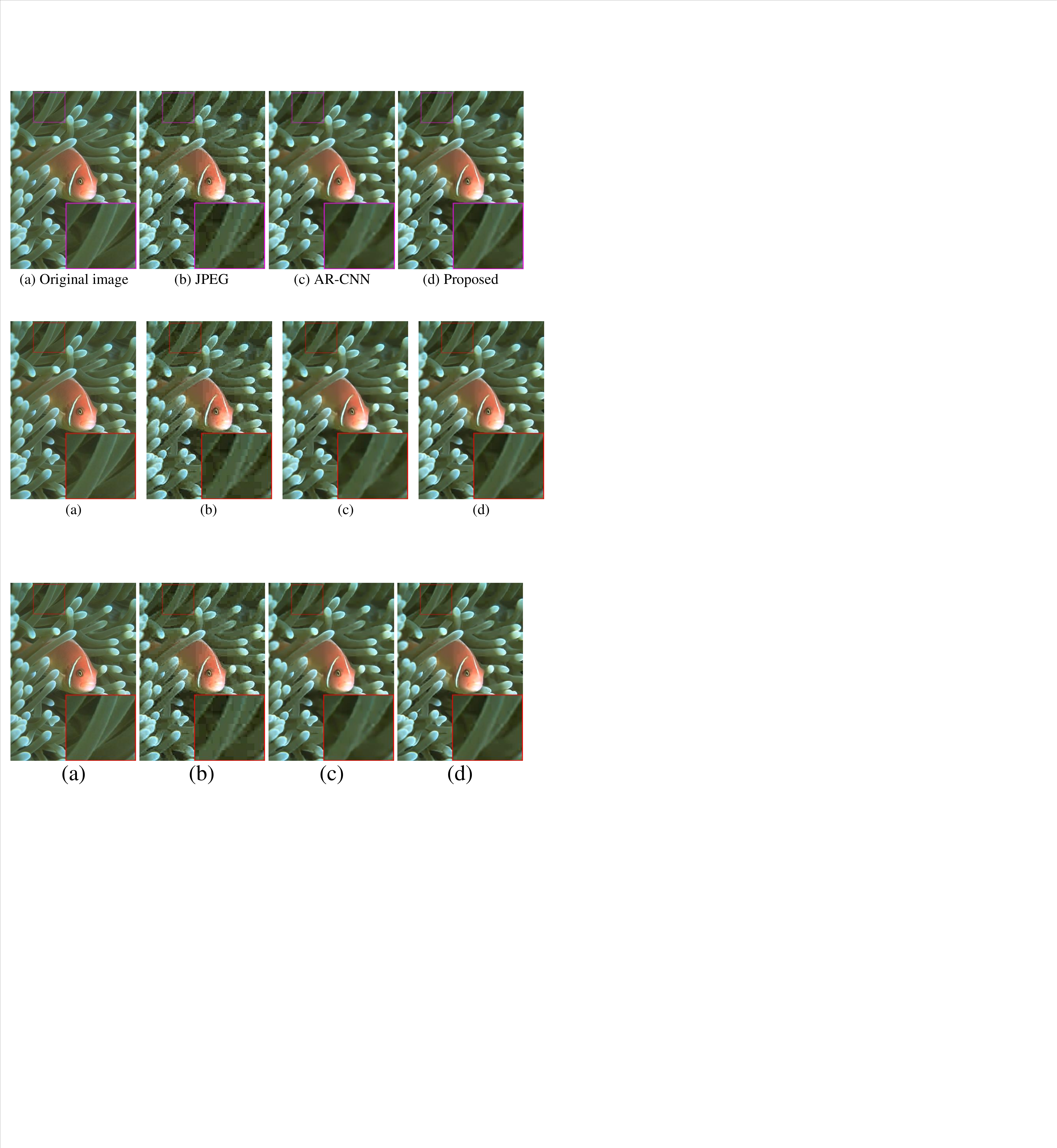}}
\end{minipage}
	\vspace{-6mm}
	\caption{Visual comparison results of image $\emph{210088}$ at  QF = 10. (a) Original image; (b) JPEG compressed image (PSNR = 30.33dB, SSIM = 0.8342); (c) AR-CNN \cite{66} (PSNR = 33.21dB, SSIM = 0.9147); (d) RRC (PSNR = \textbf{33.37dB}, SSIM = \textbf{0.9201}).}
	\label{fig:12}
%\vspace{-2mm}
\end{figure}

\begin{figure}[!t]	
	\centering
\vspace{-2mm}
\begin{minipage}[b]{1\linewidth}
{\includegraphics[width= 1\textwidth]{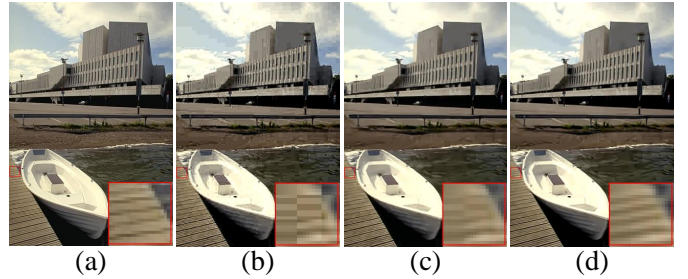}}
\end{minipage}
	\vspace{-6mm}
	\caption{Visual comparison results of image $\emph{78004}$ at  QF = 10. (a) Original image; (b) JPEG compressed image (PSNR = 27.72dB, SSIM = 0.7854); (c) AR-CNN \cite{66} (PSNR = 28.86dB, SSIM = 0.8206); (d) RRC (PSNR = \textbf{29.57dB}, SSIM = \textbf{0.8361}).}
	\label{fig:13}
\vspace{-2mm}
\end{figure}

\begin{table}[!t]
%\vspace{-2mm}
\caption{ {The average PSNR ($\textnormal{d}$B) and SSIM  of RRC and AR-CNN \cite{66} for compression artifacts reduction on the BSD 100 dataset \cite{62}.}}
%\vspace{-2mm}
\centering
%\LARGE
\resizebox{0.48\textwidth}{!}{
\begin{tabular}{|c|c|c|c|c|c|c|c|}
\hline
\multicolumn{1}{|c|}{}&\multicolumn{5}{|c|}{\textbf{PSNR}}\\
\hline
{\textbf{{Methods}}}&{\textbf{{$QF=10$}}}&{\textbf{{$QF=20$}}}
& {\textbf{{$QF=30$}}}&{\textbf{{$QF=40$}}}&{\textbf{{Average}}}\\
\hline
\multirow{1}{*}{\textbf{AR-CNN}}
& 28.55 & 30.75 & 32.14  & 33.08 & 31.13 \\
\hline
\multirow{1}{*}{\textbf{RRC}}
& \textbf{28.59} & \textbf{30.98} & \textbf{32.41}  & \textbf{33.44} & \textbf{31.36} \\
\hline
\multicolumn{1}{|c|}{}&\multicolumn{5}{|c|}{\textbf{SSIM}}\\
\hline
{\textbf{{Methods}}}&{\textbf{{$QF=10$}}}&{\textbf{{$QF=20$}}}
& {\textbf{{$QF=30$}}}&{\textbf{{$QF=40$}}}&{\textbf{{Average}}}\\
\hline
\multirow{1}{*}{\textbf{AR-CNN}}
& 0.7952 & 0.8658 & 0.8997  & 0.9166 & 0.8693 \\
\hline
\multirow{1}{*}{\textbf{RRC}}
& \textbf{0.8015} & \textbf{0.8745} & \textbf{0.9057}  & \textbf{0.9231} & \textbf{0.8762} \\
\hline
\end{tabular}}
\label{Tab:5}
\vspace{-4mm}
\end{table}

Recently, deep learning based techniques for image compression artifacts reduction have attracted significant attentions due to their impressive performance. We also compare the proposed RRC with the AR-CNN \cite{66} method, which is deemed as the benchmark of CNN-based compression artifacts reduction algorithms. As shown in Table~\ref{Tab:5}, the proposed algorithm outperforms the AR-CNN method across all cases on the BSD 100 dataset. The average PSNR and SSIM gain is up to 0.23dB and 0.0069, respectively. The visual comparisons of image $\emph{210088}$ and $\emph{78004}$ {on BSD 100 dataset} with QF = 10 are shown in Fig.~\ref{fig:12} and Fig.~\ref{fig:13}, respectively. It can be observed that AR-CNN still suffers from undesirable artifacts and over-smooth effect. The proposed RRC not only preserves sharp edges and fine details, but also eliminates the blocking artifacts more effective than AR-CNN  method. These results further verify the superiority of the proposed RRC model.

\begin{table}[!t]
%\vspace{-2mm}
	\caption{The average PSNR ($\textnormal{d}$B) comparisons with different $w$ on the BSD 100 dataset \cite{62}.}
%\vspace{-2mm}
\tiny
	\centering
	\resizebox{.48\textwidth}{!}
	{
		\begin{tabular}{|c|c|c|c|c|c|c|c|c|}
			\hline
			\multirow{1}{*}{\textbf{{QF}}}&\multirow{1}{*}{\textbf{{0.1}}} &\multirow{1}{*}{\textbf{{0.2}}}
             &\multirow{1}{*}{\textbf{{0.3}}} &\multirow{1}{*}{\textbf{{0.4}}}  &\multirow{1}{*}{\textbf{{0.5}}}\\
			\hline
			\multirow{1}{*}{\textbf{10}}
			&	28.37	&	\textbf{28.59}	&	28.52	&	28.25	&	27.89	\\
\hline
			\multirow{1}{*}{\textbf{20}}
			&	30.75	&	\textbf{30.98}	&	30.91	&	30.64	&	30.29	\\
\hline
			\multirow{1}{*}{\textbf{Average}}
			&	29.56	&	\textbf{29.79}	&	29.72	&	29.45	&	29.09	\\
\hline
	\end{tabular}}
	\label{Tab:6}
\vspace{-4mm}
\end{table}

Next, we will discuss how to select the best quantization constraint parameter $w$ for the performance of the proposed RRC algorithm. In this paper, we adopt the narrow quantization constraint (NQC) \cite{47} by setting $w$ = 0.2 rather than  traditional $w$ = 0.5. Specifically, we report the performance comparisons with different $w$ for BSD 100 dataset in Table~\ref{Tab:6}. It is quite clear that, the setting of $w$ = 0.2 achieves better results than that of $w$ = 0.5 and obtains about 0.70dB improvement on average, which verifies the effectiveness of NQC.

We further report computational time of the proposed RRC algorithm for image compression artifacts reduction. Our proposed algorithm is implemented in MATLAB 2015b environment and it requires about 20 minutes (on a PC with Intel i5-4590 3.30GHz CPU and 8GB memory) to deblock an image (size: 320$\times$ 480), which is similar to CONCOLOR method. It is well-known that a common drawback of low-rank methods is high complexity. We are working on using GPU to accelerate the proposed algorithm since the SVD of each group can be performed in parallel.

%\vspace{-4mm}

\section{Conclusion}
 \label{sec:7}
We have proposed a novel rank minimization model, dubbed rank residual constraint (RRC), to reinterpret the rank minimization problem from the perspective of matrix approximation. Different from existing low-rank based methods, which estimated the underlying low-rank matrix directly from the corrupted observations, we progressively approximate the underlying low-rank matrix by minimizing the rank residual. Based on the group-based sparse representation model, an analytical investigation on the feasibility of the RRC model has been provided. We have developed the high performance low-rank matrix estimation based image restoration algorithms via minimizing the rank residual. Specifically, by exploiting the image NSS prior, we have applied the proposed RRC model to image restoration tasks, including image denoising and image compression artifacts reduction. Experimental results have demonstrated that the proposed RRC not only leads to visible quantitative improvements over many state-of-the-art methods, but also preserves the image local structures and suppresses undesirable artifacts.

\section*{Acknowledge}

The authors would like to appreciate the associate editor for coordinating the review of the manuscript, and appreciate the anonymous reviewers for their constructive suggestions to improve the manuscript. The authors would like to appreciate Prof. Jian Zhang  of Peking University, Prof. Tao Yue and Prof. Zhan Ma of Nanjing University  for their helps and appreciate the authors of \cite{5,26,27,56,57,49,58,29,59,60,42,63,64,46,44,45,65,77} for providing their source codes or experimental results.

\appendices

\section{Proof of the Theorem~\ref{theorem:2}}
\label{theorem2}
\begin{proof}

Supposing that the SVD of $\textbf{\emph{X}}_i, \textbf{\emph{Y}}_i, \textbf{\emph{X}}_i'$ are $\textbf{\emph{X}}_i= \textbf{\emph{U}}_i\boldsymbol\Sigma_i\textbf{\emph{V}}_i^T$, $\textbf{\emph{Y}}_i= \textbf{\emph{P}}_i\boldsymbol\Delta_i\textbf{\emph{S}}_i^T$ and $\textbf{\emph{X}}_i'= \textbf{\emph{R}}_i\boldsymbol\Lambda_i\textbf{\emph{Q}}_i^T$, respectively, where $\boldsymbol\Sigma_i$, $\boldsymbol\Delta_i$ and $\boldsymbol\Lambda_i$ are ordered singular value matrices with the same order.
Recalling  Eq.~\eqref{eq:6} and from Theorem~\ref{theorem:1}, we have
\begin{equation}
\begin{aligned}
&\left\|\textbf{\emph{Y}}_i-\textbf{\emph{X}}_i\right\|_F^2 = \|\textbf{\emph{P}}_i\boldsymbol\Delta_i\textbf{\emph{S}}_i^T-\textbf{\emph{U}}_i\boldsymbol\Sigma_i\textbf{\emph{V}}_i^T\|_F^2\\
&= {\rm Tr}(\boldsymbol\Delta_i {\boldsymbol\Delta_i}^T)+{\rm Tr}(\boldsymbol\Sigma_i {\boldsymbol\Sigma_i}^T)
-2{\rm Tr}({\textbf{\emph{X}}}_i^T {\textbf{\emph{Y}}}_i)\\
& \geq {\rm Tr}(\boldsymbol\Delta_i {\boldsymbol\Delta}_i^T)+{\rm Tr}(\boldsymbol\Sigma_i {\boldsymbol\Sigma}_i^T)
-2{\rm Tr}({\boldsymbol\Sigma}_i^T \boldsymbol\Delta_i)\\
&=\left\|\boldsymbol\Delta_i-\boldsymbol\Sigma_i\right\|_F^2,\\
\end{aligned}
\label{eq:46}
\end{equation}%equal 46
where the equality holds only when $\textbf{\emph{P}}_i=\textbf{\emph{U}}_i$ and $\textbf{\emph{S}}_i=\textbf{\emph{V}}_i$. Therefore, Eq.~\eqref{eq:6} is minimized when $\textbf{\emph{P}}_i=\textbf{\emph{U}}_i$ and $\textbf{\emph{S}}_i=\textbf{\emph{V}}_i$, and the optimal solution of ${\boldsymbol\Sigma}_i$ is obtained by solving
\begin{equation}
\begin{aligned}
&\min_{\boldsymbol\Sigma_i\geq0}\frac{1}{2} \left\|\boldsymbol\Delta_i-\boldsymbol\Sigma_i\right\|_F^2 +\lambda \sum_{k=1}^j|\gamma_{i,k}|\\
&= \min\limits_{\sigma_{i,k}\geq0}
\sum_{k=1}^j\left(\frac{1}{2}(\delta_{i,k}-\sigma_{i,k})^2+\lambda |\sigma_{i,k}-\psi_{i,k}| \right),
\label{eq:47}
\end{aligned}
\end{equation}%equal 47
where $\sigma_{i,k}$, $\delta_{i,k}$ and $\psi_{i,k}$ are the $k^{th}$ singular value of  $\textbf{\emph{X}}_i$, $\textbf{\emph{Y}}_i$ and $\textbf{\emph{X}}_i'$, respectively.

\end{proof}

\section{Proof of the Lemma~\ref{lemma:2}}
\label{lemma2}
\begin{proof}
From ${\textbf{\emph{D}}}_i$ in Eq.~\eqref{eq:43} and the unitary property of ${\textbf{\emph{U}}}_i$ and ${\textbf{\emph{V}}}_i$, %we have
\begin{equation}
\begin{aligned}
&\left\|{\textbf{\emph{Y}}}_i-{\textbf{\emph{X}}}_i\|_F^2=\|{\textbf{\emph{D}}}_i({\textbf{\emph{K}}}_i-{\textbf{\emph{A}}}_i)\right\|_F^2
=\left\|{\textbf{\emph{U}}}_i{\rm diag}({\textbf{\emph{K}}}_i-{\textbf{\emph{A}}}_i){\textbf{\emph{V}}}_i^T\right\|_F^2\\
&= {\rm Tr}({\textbf{\emph{U}}}_i{\rm diag}({\textbf{\emph{K}}}_i-{\textbf{\emph{A}}}_i){\textbf{\emph{V}}}_i^T
{\textbf{\emph{V}}}_{i}{\rm diag}({\textbf{\emph{K}}}_i-{\textbf{\emph{A}}}_i){\textbf{\emph{U}}}_{i}^T)\\
&= {\rm Tr}({\textbf{\emph{U}}}_i{\rm diag}({\textbf{\emph{K}}}_i-{\textbf{\emph{A}}}_i)
{\rm diag}({\textbf{\emph{K}}}_i-{\textbf{\emph{A}}}_i){\textbf{\emph{U}}}_{i}^T)\\
&={\rm Tr}({\rm diag}({\textbf{\emph{K}}}_i-{\textbf{\emph{A}}}_i)
{\textbf{\emph{U}}}_i^T{\textbf{\emph{U}}}_{i}{\rm diag}({\textbf{\emph{K}}}_i-{\textbf{\emph{A}}}_i))\\
&={\rm Tr}({\rm diag}({\textbf{\emph{K}}}_i-{\textbf{\emph{A}}}_i)
{\rm diag}({\textbf{\emph{K}}}_i-{\textbf{\emph{A}}}_i))\\
&=\left\|{\textbf{\emph{K}}}_i-{\textbf{\emph{A}}}_i\right\|_F^2.
\label{eq:48}
\end{aligned}
\end{equation}

\end{proof}

\section{Proof of the Theorem~\ref{theorem:3}}
\label{theorem3}
\begin{proof}

On the basis of Lemma~\ref{lemma:2}, we have
\begin{equation}
\begin{aligned}
\hat{\textbf{\emph{A}}}_i& = \arg\min\limits_{{\textbf{\emph{A}}}_i}	\left(\frac{1}{2}\left\|{\textbf{\emph{Y}}}_i-{\textbf{\emph{D}}}_i{\textbf{\emph{A}}}_i\right\|_F^2
+\lambda\left\|{\textbf{\emph{A}}}_i-{\textbf{\emph{B}}}_i\right\|_1\right)\\
&={\arg\min\limits_{{\textbf{\emph{A}}}_i}\left(\frac{1}{2}\left\|{\textbf{\emph{K}}}_i-{\textbf{\emph{A}}}_i\right\|_F^2
+\lambda\left\|{\textbf{\emph{A}}}_i-{\textbf{\emph{B}}}_i\right\|_1\right)}\\
&={\arg\min\limits_{\boldsymbol\alpha_i}\left(\frac{1}{2}\left\|{\boldsymbol\kappa}_i-{\boldsymbol\alpha}_i\right\|_2^2
+\lambda\left\|{\boldsymbol\alpha}_i-{\boldsymbol\beta}_i\right\|_1\right),}
\end{aligned}
\label{eq:49}
\end{equation} %equal 49
where ${\textbf{\emph{X}}}_i={{\textbf{\emph{D}}}_i{\textbf{\emph{A}}}_i}$ and ${\textbf{\emph{Y}}}_i={{\textbf{\emph{D}}}_i{\textbf{\emph{K}}}_i}$. ${{{{\boldsymbol\alpha}}}_i}$, ${{{{\boldsymbol\beta}}}_i}$ and ${{{{\boldsymbol\kappa}}}_i}$ denote the vectorization of the matrix ${{{\textbf{\emph{A}}}}_i}$, ${{{\textbf{\emph{B}}}}_i}$ and ${{{\textbf{\emph{K}}}}_i}$, respectively.

Following this, based on Lemma~\ref{lemma:1}, we have
\begin{equation}
{{\boldsymbol\alpha}}_i ={\rm soft}({{\boldsymbol\kappa}}_i-{{\boldsymbol\beta}}_i,\lambda)+{{\boldsymbol\beta}}_i,
\label{eq:50}
\end{equation}%equal 50

Then the closed-form solution of the $k$-th element ${{\boldsymbol\alpha}}_{i, k}$ of ${{\boldsymbol\alpha}}_i$ in Eq.~\eqref{eq:50} is solved by the following problem,

\begin{equation}
{{\boldsymbol\alpha}}_{i, k} ={\rm soft}({{\boldsymbol\kappa}}_{i, k}-{{\boldsymbol\beta}}_{i, k},\lambda)+{{\boldsymbol\beta}}_{i, k}.
\label{eq:50.1}
\end{equation}%equal 50

Then, based on the adaptive dictionary $\textbf{\emph{D}}_i$ in Eq.~\eqref{eq:43} and Theorem~\ref{theorem:2}, we have proved that Eq.~\eqref{eq:50.1} is equivalent to Eq.~\eqref{eq:11}. Note that we assume the PCA space of ${\textbf{\emph{X}}}_i$ and $\textbf{\emph{X}}_i'$ are equivalent here. We have thus that RRC is equivalent to GSRC, \ie,
\begin{equation}
\begin{aligned}
\hat{\textbf{\emph{X}}}_i&= \arg\min_{\textbf{\emph{X}}_i} \left(\frac{1}{2}\left\|{\textbf{\emph{Y}}}_i-{\textbf{\emph{X}}}_i\right\|_F^2+\lambda\left\|\gammav_i\right\|_1\right) \\
&\ \ \ \ \ \ \ \ \ \ \ \ \ \ \  \ \ \ \ \ \ \ \  \ \ \ \  \ \ \Updownarrow \\
\hat{\textbf{\emph{A}}}_i&=\arg\min\limits_{{\textbf{\emph{A}}}_i}		\left(\frac{1}{2}\left\|{\textbf{\emph{Y}}}_i-{\textbf{\emph{D}}}_i{\textbf{\emph{A}}}_i\right\|_F^2
+\lambda\left\|{\textbf{\emph{A}}}_i-{\textbf{\emph{B}}}_i\right\|_1\right).
\label{eq:52}
\end{aligned}
\end{equation}
\end{proof}

{\footnotesize
\bibliographystyle{IEEEtran}
\bibliography{rrc_ref}
}
\end{document}